%% file: main.tex
\documentclass[twoside,11pt]{article}

\usepackage[preprint]{jmlr2e}
\input{math_commands.tex}

\usepackage{hyperref}
\usepackage{url}
\usepackage{graphicx}
\usepackage{xcolor}
\usepackage{enumerate}
\usepackage{textcomp}
\usepackage{mathrsfs}
\usepackage{amsthm}
\usepackage{bbm}
\usepackage{epigraph}
\usepackage{etoolbox}
\usepackage{algorithm,algpseudocode,pgffor,caption,subcaption,multirow}
\usepackage{thmtools, thm-restate}
\usepackage{proof-at-the-end}
\ifdefined\proof
    \renewenvironment{proof}{\par\noindent{\bf Proof\ }}{\hfill\BlackBox\\[2mm]}
\else
    \newenvironment{proof}{\par\noindent{\bf Proof\ }}{\hfill\BlackBox\\[2mm]}
\fi
\newtheorem{example}{Example} 
\newtheorem{theorem}{Theorem}
\newtheorem{lemma}[theorem]{Lemma}

\newtheorem{corollary}[theorem]{Corollary}
\newtheorem{definition}[theorem]{Definition}

\def\environment{\mathcal{E}}

\def\proxy{\tilde{\environment}}
\def\proxyset{\tilde{\Theta}}

\def\regret{\mathcal{R}}
\def\KL{\mathbf{d}_{\mathrm{KL}}}
\def\normal{\mathcal{N}}

\def\ratedistortion{\H_\epsilon(\environment)}
\def\diffentropy{\bf h}
\def\var{\mathbb{V}}
\def\sign{{\rm sign}}
\def\proxytheta{\tilde{\theta}}
\def\relu{{\rm ReLU}}

\def\subgauss{\nu^2}

\def\E{\mathbb{E}}
\def\F{\mathbb{F}}
\def\H{\mathbb{H}}
\def\diffentropy{\mathbf{h}}
\def\I{\mathbb{I}}
\def\Pr{\mathbb{P}}
\def\R{\mathbb{R}}
\def\1{\mathbf{1}}

\newcommand{\Xc}{\mathcal{X}}
\newcommand{\Yc}{\mathcal{Y}}

\newcommand{\sphere}{\mathbb{S}^{d-1}}
\usepackage{hyperref}

\usepackage[mathscr]{euscript}
\hypersetup{
    colorlinks=true,
    linkcolor=blue,
    filecolor=magenta,      
    urlcolor=purple,
}
\usepackage{xcolor}

\newcount\Comments
\newcommand{\kibitz}[2]{\ifnum\Comments=1{\textcolor{#1}{\textsf{\footnotesize #2}}}\fi}
\usepackage{color}
\definecolor{darkred}{rgb}{0.7,0,0}
\definecolor{darkgreen}{rgb}{0.0,0.5,0.0}
\definecolor{darkblue}{rgb}{0.0,0.0,0.5}
\definecolor{teal}{rgb}{0.0,0.5,0.5}

\jmlrheading{1}{2021}{}{}{}{}{Hong Jun Jeon and Yifan Zhu and Benjamin Van Roy}

\ShortHeadings{An Information-Theoretic Framework for Supervised Learning}{Jeon and Van Roy}
\firstpageno{1}

\begin{document}

\title{An Information-Theoretic Framework for\\ Supervised Learning}

\author{\name Hong Jun Jeon \email hjjeon@stanford.edu \\
       \addr Department of Computer Science\\
       Stanford University\\
       Stanford, CA 94305, USA
       \AND
        \name Yifan Zhu \email zhuyifan@stanford.edu \\
        \addr Department of Electrical Engineering\\
        Stanford University\\
        Stanford, CA 94305, USA
       \AND
       \name Benjamin Van Roy \email bvr@stanford.edu \\
       \addr 
       Stanford University\\
       Stanford, CA 94305, USA}
\editor{}

\maketitle

\begin{abstract}%
Each year, deep learning demonstrates new and improved empirical results with deeper and wider neural networks. Meanwhile, with existing theoretical frameworks, it is difficult to analyze networks deeper than two layers without resorting to counting parameters or encountering sample complexity bounds that are exponential in depth. Perhaps it may be fruitful to try to analyze modern machine learning under a different lens. In this paper, we propose a novel information-theoretic framework with its own notions of regret and sample complexity for analyzing the data requirements of machine learning. With our framework, we first work through some classical examples such as scalar estimation and linear regression to build intuition and introduce general techniques. Then, we use the framework to study the sample complexity of learning from data generated by deep neural networks with ReLU activation units. For a particular prior distribution on weights, we establish sample complexity bounds that are \emph{simultaneously} \emph{width} independent and linear in \emph{depth}.  This prior distribution gives rise to high-dimensional latent representations that, with high probability, admit reasonably accurate low-dimensional approximations.  We conclude by corroborating our theoretical results with experimental analysis of random single-hidden-layer neural networks.
\end{abstract}


\section{Introduction}

The refrain ``success is guaranteed'' espoused by some deep learning researchers suggests that, given a large data set, a sufficiently large neural network trained via stochastic gradient descent will deliver a useful model.  Perhaps this statement is not intended to be taken literally, as it is easy to generate data in a manner for which {\it no algorithm} can accomplish this by learning from any reasonable number of samples.  Yet, neural networks {\it have} successfully addressed many complex data sets.  This begs the question: ``for what data generating processes can neural networks succeed?''

% \ben{not sure whether this is worth mentioning in the paper, but I am saying this here because it might be interesting/useful to you.  It was recently shown that you can do better than GPT3 by using fewer parameters and more data.  this led to Chinchilla, which is an alternative to GPT3.  the lesson from that study is that, while GPT3 continued to improve with a fixed amount of training data as the number of parameters increased, a better investment of computation would have been made if the number of parameters was reduced and more data was used.  note that with these language models, there is so much data available that each data point only gets used once in sgd, and a lot of data never gets used.}
Perhaps this refrain stems from the \emph{empirical} phenomena of this era. In modern machine learning, the apparent capabilities of empirical methods have rapidly outpaced what is soundly understood theoretically. Modern neural network architectures have scaled immensely in both parameter count and depth. GPT3 for example has ~175 billion parameters and 96 decoder blocks, each with many layers within. Yet, contrary to traditional statistical analyses, these deep neural networks with gargantuan parameter counts are able to generalize well and produce useful models. This gap between what has been shown theoretically versus empirically makes it quite enticing to develop a coherent framework that could explain this phenomenon.

% Paragraph on parametric statistics and bounds that depend on parameter count.
In parametric statistics, the number of parameters typically drives sample complexity. In the realm of classical statistics, problems such as linear regression, this analysis based on parameter count can produce sharp results that mirror what is observed in practice. Naturally, researchers have made efforts to extend these techniques to try and understand deep learning.

However, these existing results break down when trying to explain learning under models that are simultaneously very deep (many layers) and wide (many hidden units per layer). Classical results such as those of \citet{HAUSSLER199278} and \cite{bartlett1998almost} can potentially handle the deep but narrow case. These results bound the sample complexity of a learning a neural network function in terms of the number of parameters and the depth. More recently, \cite{pmlr-v65-harvey17a} established a general result that suggests that for neural networks with piecewise-linear activation units, the sample complexity grows linearly in the product of parameter count and depth. However, when we consider neural networks with arbitrary width, these bounds become vacuous. This is unnerving as in practice, wider neural networks have been observed to generalize \emph{better} \citep{neyshabur2014search}.

As an alternative to parameter count methods, researchers have produced sample complexity bounds that depend on the product of norms of realized weight matrices. In these networks, while the layers may be arbitrarily wide, the complexity is instead constrained by bounded weight matrix norms. \cite{bartlett2017spectrally} and \cite{neyshabur2018pac}, for example, establish sample complexity bounds that scale with the product of spectral norms. \cite{neyshabur2015norm} and \cite{pmlr-v75-golowich18a} establish similar bounds that instead scale in the product of Frobenius norms. While this line of work provides sample complexity bounds that are \emph{width}-independent, they pay for it via an exponential dependence on \emph{depth}, which is also inconsistent with empirical results. 

A large line of work has tried to ameliorate this exponential depth dependence via so-called {\it data-dependent} quantities \citep{dziugaite2017computing,arora2018stronger,nagarajan2018deterministic,NEURIPS2019_0e795480}. Among these, the most relevant to our work is \cite{NEURIPS2019_0e795480}, which bounds sample complexity as a function of depth and statistics of trained neural network weights. While difficult to interpret due to dependence on complicated data-dependent statistics, their bound suggests a nonic dependence on depth. \cite{arora2018stronger} also utilize concepts of compression in their analysis, which we generalize and expand upon. While they establish a sample complexity bound that suggests quadratic dependence on depth, further dependence may be hidden in data-dependent terms. These data-dependent bounds are touted for their flexibility as one can derive expected error bounds by simply integrating out randomness of the data. However, such a procedure does not allow one to derive the bounds that we establish in this paper.

The results in computational theory would suggest a similarly bleak picture. They suggests that, even for single-hidden-layer teacher networks, the computation required to achieve this sample complexity is intractable.
For example, \citet{2020-Goel-superpolynomial-lower-bounds, 2020-Diakonikolas-algorithms-and-sq-lower-bounds} establish that, for batched stochastic gradient descent with respect to squared or logistic loss to achieve small generalization error \emph{for all} single-hidden-layer teacher networks, the number of samples or number of gradient steps must be superpolynomial in input dimension or network width.
Furthermore, current theoretical guarantees for all computationally tractable algorithms proposed for fitting single-hidden-layer teacher networks with parameters drawn from natural distributions only bound sample complexity by high-order polynomial \citep{2015-Janzamin-generalization-bounds-for-neural-networks-through-tensor-factorization,2017-Ge-learning-one-hidden-layer-neural-networks-with-landscape-design}
or exponential \citep{2017-Zhong-recovery-guarantees, 2020-Fu-guaranteed-recovery-of-one-hidden-layer-neural-networks-via-cross-entropy}
functions of input dimension or width.

We suspect that the looseness of all aforementioned results in comparison to empirical findings is due to the worst-case analysis framework. In this paper, we study an average-case notions of regret and sample complexity that are motivated by information theory. Our information-theoretic framework generalizes that developed by \cite{haussler1994bounds}, which provided a basis for understanding the relationship between prediction error and information. In a similar vein, \cite{russo2019much} introduced tools that establish general relationships between mutual information and error. Using these results, \cite{xu2017information} established upper bounds on the generalization error of learning algorithms with countably infinite hypothesis spaces. We extend these results in several directions to enable analysis of data generating processes related to deep learning. For example, the results of \cite{haussler1994bounds} do not address noisy observations, and all three aforementioned papers do not accommodate continuous parameter spaces, let alone nonparametric data generating processes. A distinction of our work is that it builds on rate-distortion theory to address these limitations. While \cite{7541669} also use rate-distortion theory to study Bayes risk, these results are again limited to parametric classification and only offer lower bounds. The rate-distortion function that we study is equivalent to one defined by the \emph{information bottleneck} \citep{tishby2000information}. However, instead of using it as a basis for \emph{optimization methods} as do \cite{shwartz2017opening}, we develop tools to study \emph{sample complexity} and arrive at concrete and novel results.

In this paper, we consider contexts in which an agent learns from an iid sequence of data pairs. We consider a suite of data generating processes ranging from classical examples to those for which deep neural networks may be suited. For each data generating process, we quantify the number of samples required to arrive at a useful model.  These analyses rely on general and elegant information-theoretic results that we introduce. We establish tight upper and lower bounds for the average regret and sample complexity that depend on the rate-distortion function. With these information-theoretic tools, we analyze three deep neural network data generating processes and quantify the number of samples required to arrive at a useful model. For a ReLU deep neural network with independent weights, \textbf{we establish novel sample complexity bounds that are roughly linear in the parameter count} (as opposed to linear in the product of parameter count and depth as established in \citep{pmlr-v65-harvey17a}). For a ReLU deep neural network with weights drawn from an appropriately scaled Dirichlet distribution, \textbf{we establish sample complexity bounds that are, within logarithmic factors, linear in depth and independent of width} as opposed to exponential \citep{bartlett2017spectrally} or high-order polynomial \citep{NEURIPS2019_0e795480} in depth. Despite the fact that these bounds are prescribed for an \emph{optimal} learner, \textbf{we provide extensive computational evidence that, in practice, the performance of agents that apply \emph{stochastic gradient descent} (SGD) and automated width selection closely approximates our bounds, even though they pertain to \emph{optimal} learners}.

\textbf{We view our approach to bounding sample complexity of multilayer data generating processes, our foundational information-theoretic tools, and our extensive experimental verification to be the primary contributions of this paper}. Beyond this paper, we expect future results to build on this framework.  In particular, its generality and conceptual simplicity positions it to address problems beyond supervised learning, such as reinforcement learning \citep{sutton2018reinforcement} and learning with side information \citep{jonschkowski2015patterns}.  Indeed, information theory has already influenced thought on these topics \citep{lu2021reinforcement}, and our results should provide tools to develop further understanding.

\section{Prediction and Error}
\label{se:rate-distortion}

Consider an environment which, when presented with an input, responds with an output.  In the standard framing of supervised learning, an agent learns from input-output data pairs to predict the output corresponding to any future input.  Accuracy of the agent's prediction depends on information the agent has acquired about the environment.  In this section, we introduce mathematical formalisms for reasoning about environments and predictions.  We formalize a notion of error that is equivalent to the incremental information that a new data pair provides about the environment. With this notion of error, we establish the notions of regret and sample complexity that we will study in this work. While to many these ideas may seem unconventional, we provide connections between our formalism and existing hallmarks of machine learning.

\subsection{Environment}

We denote input and output spaces by $\Xc$ and $\Yc$.  While our concepts and results extend to more general spaces, for the purpose of this paper, we restrict attention to cases where $\Xc$ is a finite-dimensional real-valued vector space and $\Yc$ is also a finite-dimensional real-valued vector space (regression) or finite (classification).  As illustrated in Figure \ref{fig:environment}, an environment $\environment$ prescribes to each input $X$ and a conditional probability measure $\environment(\cdot | X)$ of the output $Y$.

\begin{figure}[!ht]
\centering
\includegraphics[scale=0.15]{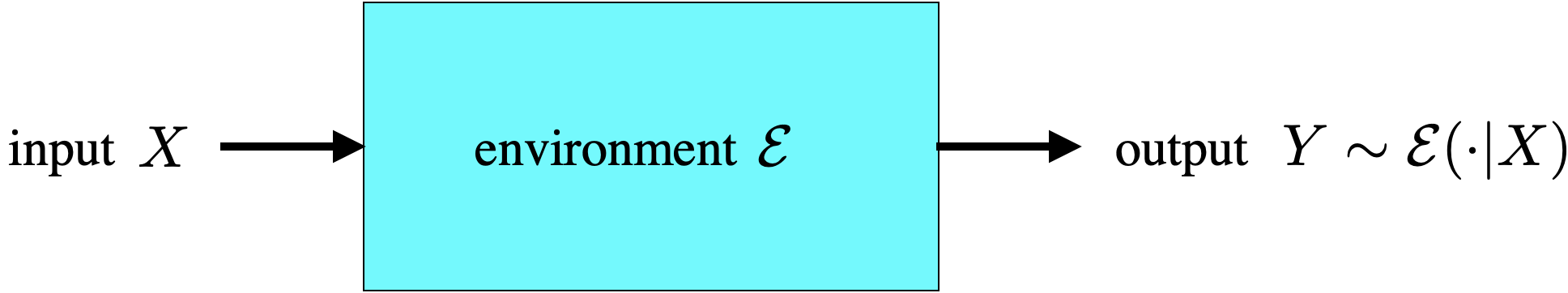}
\caption{Presented with an input, an environment responds with an output.}
\label{fig:environment}
\end{figure}

In order to model the agent's uncertainty about the environment, we treat $\environment$ as a random variable.  Before gathering any data, the agent's beliefs about the environment are represented by the prior distribution $\Pr(\environment \in \cdot)$.  The agent's beliefs evolve as it conditions this distribution on observations.

\subsection{Data Generating Process}
\label{se:data-generating-process}

We consider a stochastic process that generates a sequence $((X_t, Y_{t+1}): t=0,\ldots,T-1)$ of data pairs. We refer to each $X_t$ as an {\it input} and each $Y_{t+1}$ as an {\it output}. We define these and all other random variables we will consider with respect to a probability space $(\Omega, \F, \Pr)$.  

Elements of the sequence $(X_t: t=0,\ldots,T-1)$ are iid. Denote the history of data generated through time $t$ by $H_t = (X_0, Y_1,\ldots,X_{t-1}, Y_t, X_t)$.  Each output is distributed according to $\Pr(Y_{t+1} \in \cdot | \environment, H_t) = \environment(\cdot|X_t)$. Here, $\environment$ (the environment) is a random function that specifies a conditional output distribution $\environment(\cdot|x)$ for each input $x$. As aforementioned, initial uncertainty about $\environment$ is expressed by the prior distribution $\Pr(\environment \in \cdot)$. Note that, conditioned on $\environment$, the sequence $((X_t, Y_{t+1}): t=0,\ldots,T-1)$ is iid.

\subsection{Prediction}
We consider an agent that predicts the next output $Y_{t+1}$ given the history $H_t$.  Rather than a point estimate, the agent provides as its prediction a probability distribution $P_t$ over possible outputs.  We characterize the agent in terms of a function $\pi$ for which $P_t = \pi(H_t)$.

It will be useful to introduce some notation for referring to particular predictions.  We will generally use $P_t$ as a dummy variable -- that is a generic prediction whose definition depends on context.  We denote the prediction conditioned on the environment, which could only be produced by a prescient agent, by
$$P^*_t = \Pr(Y_{t+1} \in \cdot |\environment, X_t) = \environment(\cdot|X_t).$$
We will refer to this as the {\it target distribution} as it represents what an agent aims to learn.
Finally, we denote by $\hat{P}_t$ the posterior-predictive 
$$\hat{P}_t = \Pr(Y_{t+1} \in \cdot |H_t),$$
which will turn out to be optimal for the objective we will define next.

\subsection{Error}
We assess the error of a prediction $P_t$ in terms of the KL-divergence relative to $P^*_t$:
$$\KL(P^*_t \| P_t)
= \int P^*_t(dy)  \ln\frac{d P^*_t}{d P_t}(y).$$
This quantifies mismatch between the prediction $P_t$ and target distribution $P^*_t$.  As the following examples illustrate, this generalizes notions of error, like mean-squared error and cross-entropy loss, that are commonly used in the machine learning literature.

\subsection{Connections to Cross-Entropy Loss}
We establish that in the classification setting, our notion of error is equivalent to cross-entropy loss up to translations.
\begin{example}
{\bf (cross-entropy loss)}
Suppose the set $\mathcal{Y}$ of possible outputs is finite.  Then,
$$\KL(P^* \| P) = \sum_{y \in \mathcal{Y}} P^*(y)  \ln\frac{P^*(y)}{P(y)} = \sum_{y \in \mathcal{Y}} P^*(y)  \ln P^*(y) - \sum_{y \in \mathcal{Y}} P^*(y)  \ln P(y).$$
The first term of this difference does not depend on $P$, so minimizing KL-divergence is equivalent to minimizing the final term,
$$- \sum_{y \in \mathcal{Y}} P^*(y) \ln P(y) = - \E[\ln(P(Y_{t+1})) | \environment, P, X],$$
which is exactly the expected cross-entropy loss of $P$, as is commonly used to assess classifiers.
\end{example}

\subsection{Connections to Mean-Squared Error}
\label{sec:conn-mean-squar-err}
In the regression setting we first establish a direct link between KL-divergence and mean-squared error for the case in which $P^*_t$ and $P_t$ are Gaussian.
\begin{example}
    {\bf (gaussian mean-squared error)}
    \label{ex:guass_mse}
    Fix $\mu \in \Re$ and $\sigma^2 \in \Re_{++}$.  Let $\Pr(Y_{t+1}\in\cdot | \environment, X_t) \sim \mathcal{N}(\mu, \sigma^2)$.  Consider a point prediction $\hat{\mu}_t$ that is determined by $H_t$ and a distributional prediction $P_t \sim \mathcal{N}(\hat{\mu}_t, \sigma^2)$.  Then,
    $$\KL(P^*_t \| P_t) = \frac{\E[(\mu - \hat{\mu}_t)^2 | \environment, H_t]}{2 \sigma^2} 
    = \frac{\E[(Y_{t+1} - \hat{\mu}_t)^2 | \environment, H_t] - \E[(Y_{t+1} - \mu_t)^2 | \environment, H_t]}{2 \sigma^2}.$$
    Hence, KL-divergence grows monotonically with respect to the squared error $\E[(Y_{t+1} - \hat{\mu}_t)^2 | \environment, H_t]$.  However, while the minimal squared error $\E[(Y_{t+1} - \mu)^2 | \environment, H_t] = \sigma^2$ that is attainable with full knowledge of the environment remains positive, the minimal KL-divergence, which is delivered by $P_t \sim \mathcal{N}(\mu, \sigma^2)$, is zero.
    
    Now consider a distributional prediction $P_t \sim \mathcal{N}(\hat{\mu}_t, \hat{\sigma}_t^2)$, based on a variance estimate $\hat{\sigma}_t^2 \neq \sigma^2$.  Then,
    \begin{align*}
        \KL(P^*_t \| P_t)
        =& \frac{\E[(\mu - \hat{\mu}_t)^2 | \environment, H_t]}{2 \hat{\sigma}_t^2} + \frac{1}{2} \left(\frac{\sigma^2}{\hat{\sigma}_t^2} - 1 - \ln \frac{\sigma^2}{\hat{\sigma}^2_t}\right).
    \end{align*}
    Consider optimizing the choice of $\hat{\sigma}^2_t$ given $H_t$:
    $$\min_{\hat{\sigma}^2_t} \E[\KL(P^*_t \| P_t) | H_t].$$
    The minimum is attained by
    $$\hat{\sigma}^2_t = \underbrace{\sigma^2}_{\rm aleatoric}  + \underbrace{\E[(\mu - \E[\mu|H_t])^2 | H_t]}_{\rm epistemic} + \underbrace{\E[(\E[\mu|H_t] - \hat{\mu}_t)^2 | H_t]}_{bias},$$
    which differs from $\sigma^2$.  While $\sigma^2$ characterizes aleatoric uncertainty, the incremental variance $\hat{\sigma}^2_t - \sigma^2$ accounts for epistemic uncertainty and bias.
\end{example}
Now, for $P_t^*$ and $P_t$ that are not Gaussian, we have the following upper bound:
\begin{lemma}
    For all $t \in \mathbb{Z}_{+}$, if $\hat{\mu}_t = \int_{y\in\mathcal{Y}} y\ dP_t(y)$, then
    $$\E\left[\KL(P_t^*\|P_t)\right] \leq \frac{1}{2}\ln\left(1 + \frac{\E\left[\left(Y_{t+1} - \hat{\mu}_t\right)^2\right]}{\sigma^2}\right).$$
\end{lemma}
Therefore, decreasing the mean squared error will always decrease the expected KL-divergence. A corresponding lower bound holds for data generating processes for which $Y_{t+1}$ satisfies a certain subgaussian condition:
\begin{lemma}
    For all $t \in \mathbb{Z}_{+}$, let $\hat{\mu}_t = \int_{y\in\mathcal{Y}} y\ dP_t(y)$. If $P_t(Y_{t+1}\in\cdot)$ is $\delta_t^2$ -subgaussian conditioned on $H_t$ w.p $1$, then
    $$\E\left[\KL(P^*_t\|P_t)\right] \geq \frac{\E\left[(Y_{t+1} - \hat{\mu}_t)^2\right]}{\delta^2_t}$$
\end{lemma}
Therefore, for data generating processes that obey these subgaussian conditions, we have both upper and lower bounds for expected KL divergence in terms of mean-squared error.

\section{Regret and Sample Complexity}
We assess an agent's performance over duration $T$ in terms of the expected cumulative error
$$\regret_\pi(T) = \E\left[\sum_{t=0}^{T-1} \KL(P^*_t \| P_t)\right].$$

The focus of this paper is on understanding how well an {\it optimal} agent can perform, given particular data generating processes.  We will use \emph{regret} to refer to the optimal performance defined below.
\begin{definition} {\bf (optimal regret)}
    For all $T\in\mathbb{Z}_{+}$, the \textbf{optimal regret} is
    $$\regret(T) := \inf_\pi \regret_\pi(T).$$
\end{definition}
With this notation, the error incurred by an optimal uninformed prediction is given by $\regret(1)$.  We will also consider sample complexity, which we take to be the duration required to attain expected average error within some threshold $\epsilon \geq 0$.

\begin{definition} {\bf (sample complexity)}
  \label{def:sample-complexity}
    For all $\epsilon \geq 0$, the \textbf{sample complexity} is
    $$T_\epsilon := \min\left\{T: \frac{\regret(T)}{T}  \leq \epsilon \right\}.$$
\end{definition}
\subsection{Optimal Predictions}

We focus in this paper on how well an {\it optimal} agent performs, rather than on how to design practical agents that economize on memory and computation.  Recall that an agent is characterized by a function $\pi$, which generates predictions $P_t = \pi(H_t, Z_t)$, where $Z_t$ represents algorithmic randomness.  The following result establishes that the conditional distribution $\hat{P}_t = \Pr(Y_{t+1} \in \cdot |H_t)$ offers an optimal prediction.
\begin{theorem}{\bf (optimal prediction)}
\label{th:opt_prediction}
For all $t \geq 0$,
$$\E[\KL(P^*_t \| \hat{P}_t)\ |\ H_t] = \inf_\pi\ \E[\KL(P^*_t \| P_t)\ |\ H_t],$$
where $P_t= \pi(H_t, Z_t)$.
\end{theorem}
\begin{proof}
    Let $\hat{P}_t = \Pr(Y_{t+1}\in\cdot|H_t)$.  By Gibbs' inequality,
    \begin{align*}
        \inf_{P_t} \KL(\hat{P}_t \| P_t) = \KL(\hat{P}_t \| \hat{P}_t) = 0.
    \end{align*}
    Let $P^*_t = \Pr(Y_{t+1} \in \cdot | \environment, X_t)$.  Then, for all $P_t$,
    \begin{align*}
        \KL(P^*_t \| P_t)
        =& \E\left[ \ln\frac{d P^*_t}{d P_t}(Y_{t+1}) \Big|\environment, H_t\right] \\
        =& \E\left[ \ln dP^*_t (Y_{t+1}) |\environment, H_t\right] - \E\left[ \ln d P_t(Y_{t+1}) |\environment, H_t\right] \\
        =& \E\left[ \ln dP^*_t (Y_{t+1}) |\environment, H_t\right] - \E\left[ \ln d \hat{P}_t(Y_{t+1}) |\environment, H_t\right] \\
        &+ \E\left[ \ln d \hat{P}_t(Y_{t+1}) |\environment, H_t\right] - \E\left[ \ln d P_t(Y_{t+1}) |\environment, H_t\right] \\
        =& \E\left[ \ln\frac{d P^*_t}{d \hat{P}_t}(Y_{t+1}) \Big|\environment, H_t\right] + \E\left[ \ln\frac{d \hat{P}_t}{d P_t}(Y_{t+1}) \Big|\environment, H_t\right] \\
        =& \KL(P^*_t\|\hat{P}_t) + \E\left[ \ln\frac{d \hat{P}_t}{d P_t}(Y_{t+1}) \Big|\environment, H_t\right].\\
    \end{align*}
    It follows that
    \begin{align*}
        \inf_\pi \E[\KL(P^*_t \| P_t) | H_t]
        =& \inf_\pi \E\left[\KL(P^*_t\|\hat{P}_t) + \E\left[ \ln\frac{d \hat{P}_t}{d P_t}(Y_{t+1}) \Big|\environment, H_t\right] \Big| H_t\right] \\
        =& \E[\KL(P^*_t \| \hat{P}_t)| H_t] + \inf_\pi \E[\KL(\hat{P}_t \| P_t) |H_t]\\
        =& \E[\KL(P^*_t \| \hat{P}_t)| H_t] + \E[\KL(\hat{P}_t \| \hat{P}_t) |H_t]\\
        =& \E[\KL(P^*_t \| \hat{P}_t)| H_t].
    \end{align*}
\end{proof}
In the remainder of the paper we will study an agent that generates optimal predictions $P_t = \Pr(Y_{t+1} \in \cdot | H_t)$, as illustrated in Figure \ref{fig:agent}.
\begin{figure}[!ht]
\centering
\includegraphics[scale=0.15]{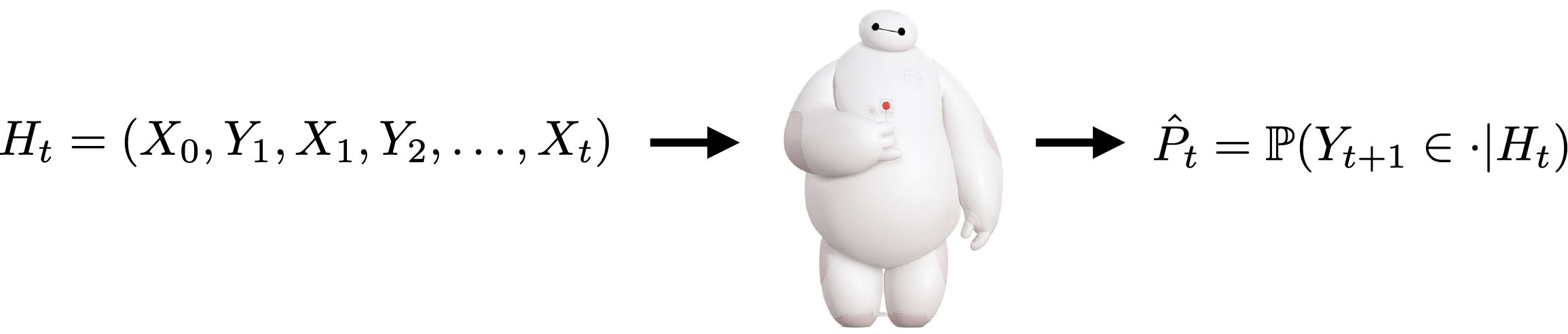}
\caption{We consider an agent that, given a history $H_t$, generates an optimal prediction $P_t = \Pr(Y_{t+1} \in \cdot | H_t)$.}
\label{fig:agent}
\end{figure}

\section{Information}

As tools for analysis of prediction error and sample complexity, we will define concepts for quantifying uncertainty and the information gained from observations.  The entropy $\H(\environment)$ of the environment quantifies the agent's initial degree of uncertainty in terms of the information required to identify $\environment$.  We will measure information in units of ${\it nats}$, each of which is equivalent to $1/\ln 2$ bits.  For example, if $\environment$ occupies a countable range $\Theta$ then $\H(\environment) = - \sum_{\theta \in \Theta} \Pr(\environment=\theta) \ln \Pr(\environment=\theta)$.  Uncertainty at time $t$ can be expressed in terms of the conditional entropy $\H(\environment|H_t)$, which is the number of remaining nats after observing $H_t$.  The mutual information $\I(\environment; H_t) = \H(\environment) - \H(\environment|H_t)$ quantifies the information about $\environment$ gained from $H_t$.

\subsection{Learning from Errors}

Each data pair $(X_t, Y_{t+1})$ provides $\I(\environment; (X_t, Y_{t+1}) | H_{t-1}, Y_t)$ nats of new information about the environment.  By the chain rule of mutual information, this is the sum
$$\I(\environment; (X_t, Y_{t+1}) | H_{t-1}, Y_t) = \I(\environment; X_t | H_{t-1}, Y_t) + \I(\environment; Y_{t+1} | H_t)$$ 
of the information gained from $X_t$ and $Y_{t+1}$.  The former term $\I(\environment; X_t | H_{t-1}, Y_t)$ is equal to zero because $X_t$ is independent from both $\environment$ and $(H_{t-1}, Y_t)$.  The latter term $\I(\environment; Y_{t+1} | H_t)$ can be thought of as the level of surprise experienced by the agent upon observing $Y_{t+1}$.  Surprise is associated with prediction error, and the following result formalizes the equivalence between error and information gain.
\begin{lemma}
    \label{le:error-and-environment-information}
    {\bf (expected prediction error equals information gain)}
    For all $t \in \mathbb{Z}_+$, 
    $$\E[\KL(P^*_t \| \hat{P}_t)] = \I(\environment; Y_{t+1} | H_t),$$
    and $\regret(t) = \I(\environment; H_t)$.
\end{lemma}
\begin{proof}
    It is well known that the mutual information $\I(A;B)$ between random variables $A$ and $B$ can be expressed in terms of the expected KL-divergence $\I(A; B) = \E[\KL(\Pr(A\in \cdot|B) \| \Pr(A \in \cdot))]$.  It follows that
    \begin{align*}
        \I(Y_{t+1}; \environment | H_t)
        & = \E[\KL(\Pr(Y_{t+1}\in\cdot | \environment, H_t) \ \| \ \Pr(Y_{t+1}\in\cdot | H_t))]\\
        & \overset{(a)}{=} \E[\KL(\Pr(Y_{t+1}\in\cdot | \environment, X_t) \ \| \ \Pr(Y_{t+1}\in\cdot | H_t))] \\
        & = \E[\KL(P^*_t \| \hat{P}_t)],
    \end{align*}
    where $(a)$ follows from the fact that $Y_{t+1}\perp H_t|(\environment, X_t)$. We then have
$$\regret(T)
= \E\left[\sum_{t=0}^{T-1} \KL(P^*_t \| \hat{P}_t)\right]
= \sum_{t=0}^{T-1} \I(Y_{t+1}; \environment | H_t)
= \I(\environment; H_T),$$
where the final equality follows from the chain rule of mutual information.
\end{proof}

The agent's ability to predict tends to improve as it learns from experience.  This is formalized by the following result, which establishes that expected prediction errors are monotonically nonincreasing.
\begin{lemma}
    \label{le:monotonic-error}
    {\bf (expected prediction error is monotonically nonincreasing)}
    For all $t \in \mathbb{Z}_+$,
    $$\E[\KL(P^*_t \| \hat{P}_t)] \geq \E[\KL(P^*_{t+1} \| \hat{P}_{t+1})].$$
\end{lemma}
\begin{proof}
    We have
    \begin{align*}
        \E[\KL(P^*_{t+1} \| \hat{P}_{t+1})]
        \overset{(a)}{=}& \I(\environment; Y_{t+2} | H_{t+1})\\
        =& \diffentropy(Y_{t+2}|H_{t+1}) - \diffentropy(Y_{t+2}|\environment, H_{t+1})\\
        \overset{(b)}{=}& \diffentropy(Y_{t+2}|H_{t+1}) - \diffentropy(Y_{t+2}|\environment, H_{t-1}, Y_t, X_{t+1})\\
        \overset{(c)}{\leq}& \diffentropy(Y_{t+2}|H_{t-1}, Y_t, X_{t+1}) - \diffentropy(Y_{t+2}|\environment, H_{t-1}, Y_t, X_{t+1})\\
        =& \I(\environment; Y_{t+2} | H_{t-1}, Y_{t}, X_{t+1}) \\
        \overset{(d)}{=}& \I(\environment; Y_{t+1} | H_{t-1}, Y_t, X_t) \\
        =& \I(\environment; Y_{t+1} | H_t) \\
        \overset{(e)}{=}& \E[\KL(P^*_t \| \hat{P}_t)],
    \end{align*}
    where $(a)$ follows from Lemma \ref{le:error-and-environment-information}, $(b)$ follows since $Y_{t+2} \perp (X_{t}, Y_{t+1})|(\environment, X_{t+1})$, $(c)$ follows from the fact that conditioning reduces differential entropy, $(d)$ follows from the fact that $(X_t, Y_{t+1})$ and $(X_{t+1}, Y_{t+2})$ are independent and identically distributed conditioned on $(H_{t-1}, Y_t)$, and $(e)$ follows from the equivalence between mutual information and expected KL-divergence.
\end{proof}

\section{General Regret and Sample Complexity Bounds}
We will characterize fundamental limits of performance by establishing bounds on the error and sample complexity attained by an optimal agent.  These bounds are very general, applying to {\it any} data generating process.  The results bound error and sample complexity in terms of rate-distortion.  As such, for any particular data generating process, bounds can be produced by characterizing the associated rate-distortion function.  In subsequent sections, we will consider particular data generating processes to which we will specialize the bounds by characterizing associated rate-distortion functions.
\subsection{Bound Regret and Sample Complexity via Entropy}
In this section we will establish the core link between discrete entropy and our notions of regret and sample complexity. We begin with the following core result
\begin{theorem}{\bf (regret and mutual information)}
    For all $T\in\mathbb{Z}_{+}$,
    $$\regret(T) = \I(H_T;\environment)$$
\end{theorem}
\begin{proof}
    \begin{align*}
        \regret(T)
        & = \inf_{\pi} \sum_{t=0}^{T-1}\E\left[\KL(P_t^*\|P_t)\right]\\
        & \overset{(a)}{=} \sum_{t=0}^{T-1}\E\left[\KL(P_t^*\|\hat{P}_t)\right]\\
        & \overset{(b)}{=} \sum_{t=0}^{T-1}\I(Y_{t+1};\environment|H_t)\\
        & \overset{(c)}{=} \I(H_T;\environment),
    \end{align*}
    where $(a)$ follows from Theorem \ref{th:opt_prediction}, $(b)$ follows from Lemma \ref{le:error-and-environment-information} and $(c)$ follows from the chain-rule of mutual information.
\end{proof}
The following upper bounds on regret and sample complexity are an almost direct result of this theorem:
\begin{theorem}
    \label{th:entropy_bounds}
    For all $T\in\mathbb{Z}_+$,
    $$\regret(T) \leq \H(\environment);\quad T_\epsilon\leq \left\lceil\frac{\H(\environment)}{\epsilon}\right\rceil$$
\end{theorem}
\begin{proof}
    We begin by showing the regret bound: 
    \begin{align*}
        \regret(T) = \I(H_T;\environment) \leq \H(\environment).
    \end{align*}
    The sample complexity bound follows as a result:
    \begin{align*}
        \left\lceil\frac{\H(\environment)}{\epsilon}\right\rceil \geq \left\lceil\frac{\regret(T)}{\epsilon}\right\rceil = T_\epsilon.
    \end{align*}
\end{proof}
This establishes that the maximum total error we can incur is $\H(\environment)$. This is intuitive as if we have incurred $\H(\environment)$ error, this means we have learned all $\H(\environment)$ bits of information that exist pertaining to $\environment$ and so all future predictions should produce $0$ additional error.

While this is a nice result for understanding simple problems for which the realizations of $\environment$ are restricted to a countable set, $\environment$ will be a continuous random variable in the majority of interesting learning problems. When $\environment$ is a continuous random variable, $\H(\environment)$ will almost always be $\infty$, resulting in vacuous regret and sample complexity bounds.

However, this vacuousness is often due to the inherent looseness of the bound $\I(\environment;H_T) \leq \H(\environment)$, i.e., it is not that the regret $\regret(T)$ is truly $\infty$ but rather that $\H(\environment)$ is an overly lofty upper bound. In fact, it is often the case that $\I(\environment;H_T)$ is actually finite and tractable. Another shortcoming is that these entropy-based bounds do not provide any insight about lower bounds on $\regret(T)$ and $T_\epsilon$. In the following section, we introduce \emph{rate-distortion theory} a set of information-theoretic tools which will allow us to address both the vacuous upper bounds and the absence of lower bounds.

\subsection{The Rate-Distortion Function}
An {\it environment proxy} is a random variable $\proxy$ that provides information about the environment $\environment$ but no additional information pertaining to inputs or outputs.  In other words, $\proxy \perp (X,Y) | \environment$.  We will denote the set of environment proxies by $\proxyset$.  While an infinite amount of information must be acquired to identify the environment when $\H(\environment) = \infty$, there can be a proxy $\proxy$ with $\H(\proxy) < \infty$ that enables accurate predictions.  The minimal expected error attainable based on the proxy is achieved by a prediction $\tilde{P} = \Pr(Y \in \cdot | \proxy, X_t)$.  This results in expected error $\E[\KL(P^* \| \tilde{P})]$.

We will establish that the expected error $\E[\KL(P^* \| \tilde{P})]$ equals the information gained, beyond that supplied by the proxy $\proxy$, about the environment $\environment$ from observing $Y$.  This is intuitive: more is learned from $Y$ if knowledge of $\environment$ enables a better prediction of $Y$ than does $\proxy$.  We quantify this information gain in terms of the difference $\H(Y | \proxy, X) - \H(Y | \environment, X)$ between the uncertainty conditioned on $\proxy$ and that conditioned on $\environment$.  This is equal to the mutual information $\I(\environment; Y | \proxy, X) = \H(Y | \proxy, X) - \H(Y | \environment, X)$.  The following result equates this with expected error.
\begin{lemma}
\label{le:proxy-error-and-environment-information}
{\bf (proxy error equals information gain)}
For all $\proxy \in \proxyset$,
$$\E[\KL(P^* \| \tilde{P})] = \I(\environment; Y | \proxy, X).$$
\end{lemma}
\begin{proof}
It is well known that the mutual information $\I(A;B)$ between random variables $A$ and $B$ can be expressed in terms of the expected KL-divergence $\I(A; B) = \E[\KL(\Pr(A\in \cdot|B) \| \Pr(A \in \cdot))]$.  We therefore have
\begin{align*}
\I(\environment; Y| \proxy, X) 
=& \E[\KL(\Pr(Y \in\cdot | \environment, \proxy, X) \ \| \ \Pr(Y \in \cdot | \proxy, X))] \\
=& \E[\KL(\Pr(Y \in\cdot | \environment, X) \ \| \ \Pr(Y \in \cdot | \proxy, X))] \\
=& \E[\KL(P^*\|\tilde{P})],
\end{align*}
where the second equation follows from the fact that $(X, Y) \perp \proxy|\environment$.
\end{proof}

Intuitively, $\E[\KL(P^*\|\tilde{P})]$ (or $\I(\environment;Y|\proxy, X)$) is a measure of the \emph{distortion} incurred in our estimate of $Y$ from knowing only $\proxy$ as opposed to the true $\environment$. For example, $\proxy$ may be a quantization or lossy compression of $\environment$ and $\E[\KL(P^*\|\tilde{P})]$ is measuring how inaccurate our prediction of $Y$ is under this compression $\proxy$.

Now we consider the following $\epsilon$-optimal set:
$$\proxyset_\epsilon = \left\{\proxy \in \proxyset : \E[\KL(P^* \| \tilde{P})] \leq \epsilon\right\}.$$
$\proxyset_\epsilon$ denotes the set of proxies that produce predictions that incur a \emph{distortion} of no more than $\epsilon$.

With the \emph{distortion} component of \emph{rate-distortion} covered, it suffices now to discuss the \emph{rate}. The \emph{rate} is the mutual information $\I(\environment; \proxy)$, which quantifies the amount of information about the environment conveyed by proxy $\proxy$.  For example, a finer quantization would result in a higher rate since $\proxy$ would capture $\environment$ up to more bits of precision. However, in turn one may expect that with this higher rate, the \emph{distortion} incurred by $\proxy$ should be \emph{lower}. A higher fidelity compression should produce less distortion. The \emph{rate-distortion function} formalizes this trade-off mathematically:
\begin{definition}{\bf (rate-distortion function)}
    For all $\epsilon\geq0$, The \textbf{rate-distortion function} for environment $\environment$ w.r.t distortion function $\E[\KL(P^*\|\tilde{P})]$ is
    $$\H_\epsilon(\environment) := \inf_{\proxy \in \proxyset_\epsilon} \I(\environment; \proxy).$$
\end{definition}
The rate-distortion function characterizes the minimal amount of information that a proxy must convey in order to be an element of $\proxyset_\epsilon$.  Intuitively, this can be thought of as the amount of information about the environment required to make $\epsilon$-accurate predictions.  Even when $\H(\environment)$ is infinite and $\epsilon$ is small, $\H_\epsilon(\environment)$ can be manageable.  As we will see in the following section, both the regret and sample complexity of learning scales with $\H_\epsilon(\environment)$.
\subsection{Bound Regret and Sample Complexity via Rate-Distortion}
With rate-distortion in place, we will tighten the bounds of Theorem \ref{th:entropy_bounds}.  These bounds are very general, applying to {\it any} data generating process.  The results upper and lower bound error and sample complexity in terms of rate-distortion as opposed to entropy.  As such, for any particular data generating process, bounds can be produced by characterizing the associated rate-distortion function.  In subsequent sections, we will consider particular data generating processes to which we will specialize the bounds by characterizing associated rate-distortion functions.

The following result brackets the cumulative error of optimal predictions.

\begin{theorem}\label{th:general-error-bound}{\bf (rate-distortion regret bounds)}
For all $T \in\mathbb{Z}_+$,
$$\sup_{\epsilon \geq 0} \min\{\H_\epsilon(\environment),\ \epsilon  T\} \leq \regret(T) \leq \inf_{\epsilon \geq 0}(\H_\epsilon(\environment) + \epsilon T).$$
\end{theorem}
\begin{proof} 
We begin by establishing the upper bound.
\begin{align*}
    \regret(T)
    & = \sum_{t=0}^{T-1} \E\left[\KL(P^*_T\|\hat{P}_t)\right]\\
    & \overset{(a)}{=} \sum_{t=0}^{T-1} \I(Y_{t+1};\environment|H_t)\\
    & = \sum_{t=0}^{T-1} \I(Y_{t+1}; \environment, \proxy | H_t) \nonumber\\
    & \overset{(b)}{=} \sum_{t=0}^{T-1} \I(Y_{t+1}; \proxy | H_t) + \I(Y_{t+1}; \environment | \proxy, H_t) \nonumber \\
    & \overset{(c)}{=} \I(H_T;\proxy) + \sum_{t=0}^{T-1} \I(Y_{t+1};\environment|\proxy, H_t)\\
    & \overset{(d)}{\leq} \I(H_T;\proxy) + \sum_{t=0}^{T-1} \I(Y_{t+1};\environment|\proxy, X_t)\\
    & \overset{(e)}{\leq} \I(H_T;\proxy) + \epsilon T\\
    & \overset{(f)}{\leq} \I(\environment;\proxy) + \epsilon T\\
\end{align*}
where (a) follows from Lemma \ref{le:error-and-environment-information}, $(b)$ follows from the chain rule of mutual information, $(c)$ follows from the chain rule of mutual information, $(d)$ follows from the facts that $\diffentropy(Y_{t+1}|\proxy, H_t) \leq \diffentropy(Y_{t+1}|\proxy, X_t)$ and $\diffentropy(Y_{t+1}|\environment, H_t) = \diffentropy(Y_{t+1}|\environment, X_t)$, and $(e)$ holds for any $\proxy\in\proxyset_\epsilon$, and $(f)$ follows from the data processing inequality.  Since the above inequality holds for all $\epsilon\geq 0$ and $\proxy\in\proxyset_\epsilon$, the result follows.
    
Next, we establish the lower bound. Fix $T \in \mathbb{Z}_+$.  Let $\proxy = (\tilde{H}_{T-2}, \tilde{Y}_{T-1})$ be independent from but distributed identically with $(H_{T-2}, Y_{T-1})$, conditioned on $\environment$.  In other words, $\proxy \perp (H_{T-2}, Y_{T-1}) | \environment$ and $\Pr(\proxy \in \cdot | \environment) = \Pr((H_{T-2}, Y_{T-1}) \in \cdot | \environment)$.  This implies that $\Pr((\environment, \tilde{H}_{T-2}, \tilde{Y}_{T-1}, X_{T-1}, Y_T) \in \cdot) = \Pr((\environment, H_{T-2}, Y_{T-1}, X_{T-1}, Y_T) \in \cdot)$, and therefore, $\I(\environment; Y_T | H_{T-1}) = \I(\environment; Y_T | \proxy, X_{T-1}).$

Fix $\epsilon \geq 0$.  If $\regret(T) < \H_\epsilon(\environment)$ then $\proxy \notin \proxyset_\epsilon$ and
\begin{align*}
\regret(T)
\overset{(a)}{=}& \I(\environment; H_T) \\
\overset{(b)}{=}& \sum_{t=0}^{T-1} \I(\environment; Y_{t+1} | H_t) \\
\overset{(c)}{\geq}& \I(\environment; Y_T | H_{T-1}) T \\
=& \I(\environment; Y_T | \proxy, X_{T-1}) T \\
\overset{(d)}{=}& \E[\KL(P^*_{T-1} \| \tilde{P}_{T-1})] T \\
\overset{(e)}{>}& \epsilon T,
\end{align*}
where (a) follows from Lemma \ref{le:error-and-environment-information}, (b) follows from the chain rule of mutual information, (c) follows from Lemma \ref{le:monotonic-error}, (d) follows from Lemma \ref{le:proxy-error-and-environment-information}, and (e) follows from the fact that $\proxy \notin \proxyset_\epsilon$.  Therefore, 
$$\regret(T) \geq \min\{\H_\epsilon(\environment),\ \epsilon T\}.$$
Since this holds for any $\epsilon \geq 0$, the result follows.
\end{proof}

This upper bound is intuitive. Knowledge of a proxy $\proxy \in \proxyset_\epsilon$ enables an agent to limit prediction error to $\epsilon$ per timestep. Getting to that level of prediction error requires $\H_\epsilon(\environment)$ nats, and therefore, that much cumulative error. Hence, $\regret(T) \leq \H_\epsilon(\environment) + \epsilon T$.

To motivate the lower bound, note that an agent requires $\H_\epsilon(\environment)$ nats to attain per timestep error within $\epsilon$. Obtaining those nats requires cumulative error at least $\H_\epsilon(\environment)$. So prior to obtaining that many nats, the agent \emph{must} incur at least $\epsilon$ error per timestep, hence the $\epsilon T$ term in the minimum. Meanwhile, if at time $T$, the agent is able to produce predictions with error less than $\epsilon$ it means that it has already accumulated at least $\H_\epsilon(\environment)$ nats of information about $\environment$ (error).

Sample complexity bounds follow almost immediately from Theorem \ref{th:general-error-bound}.
\begin{theorem}{\bf (rate-distortion sample complexity bounds)}
\label{th:general-sample-bound}
For all $\epsilon \geq 0$,
$$\frac{\H_{\epsilon}(\environment)}{\epsilon} \leq T_\epsilon \leq \inf_{\delta \in [0,\epsilon]} \left\lceil\frac{\H_{\epsilon-\delta}(\environment)}{\delta} \right\rceil \leq \left\lceil\frac{2 \H_{\epsilon/2}(\environment)}{\epsilon} \right\rceil.$$
\end{theorem}
\begin{proof}
We begin by showing the upper bound. Fix $\epsilon \geq 0$ and $\delta \in [0,\epsilon]$.  Let
$$T = \left\lceil\frac{\H_{\epsilon-\delta}(\environment)}{\delta} \right\rceil,$$
so that $\H_{\epsilon-\delta}(\environment) \leq \delta T$.  We have that:
\begin{align*}
    \regret(T)
    & \overset{(a)}{\leq} \H_{\epsilon-\delta}(\environment) + (\epsilon-\delta) T\\
    & \overset{(b)}{\leq} \delta T + (\epsilon-\delta) T\\
    & = \epsilon T,
\end{align*}
where $(a)$ follows from the upper bound of Theorem \ref{th:general-error-bound} and $(b)$ follows from our choice of $T$. Since $T_\epsilon = \min\{T : \regret(T) \leq \epsilon T\}$, it follows that $T \geq T_\epsilon$.  Since the above holds for arbitrary $\delta \in [0, \epsilon]$, the result follows.
    
We now show the lower bound. Fix $\epsilon \geq 0$.  By the definition of $T_\epsilon$, we have
$$\regret(T_\epsilon) \leq \epsilon T_\epsilon.$$
In the proof of the lower bound in Theorem \ref{th:general-error-bound}, we show that for all $\epsilon \geq 0$, $\regret(T) < \H_{\epsilon}(\environment) \implies \regret(T) > \epsilon T$. Therefore, using the contrapositive and the above definition of $T_\epsilon$, we have that $\H_{\epsilon}(\environment) \leq \regret(T_\epsilon)$ and therefore
$$\H_{\epsilon}(\environment) \leq \regret(T_\epsilon) \leq \epsilon T_\epsilon.$$
The result follows.
\end{proof}

\section{Bounds for Classical Examples}

We now demonstrate our machinery on some classical problems: scalar estimation and linear regression. While the results in these settings are not novel, we hope that they provide the reader with some intuition for the techniques that can be used to bound the rate-distortion function.

\subsection{Scalar Estimation}
We begin with scalar estimation, a problem for which for all $t$, the range of $X_t$ is a singleton and the environment is identified by a deterministic scalar $\sigma^2 \in \Re_{++}$ and a random scalar $\theta$, with $\environment(\cdot|X_t) \sim \normal(\theta, \sigma^2)$. Note that for each $t$, the output $Y_{t+1}$ is independent of the input $X_t$ and can be interpreted as a scalar signal $\theta$ perturbed by noise: $Y_{t+1} = \theta + W_{t+1}$ for a random variable $W_{t+1} \sim \mathcal{N}(0,\sigma^2)$ that is independent from $\theta$.

In this section we will use $\diffentropy$ to denote differential entropy. Before proceeding to the main results, we will state a well known result about the maximum differential entropy of a random vector with a given covariance matrix.

\begin{lemma}{\bf (maximum differential entropy)}
\label{le:max_entropy}
For all random vectors $X: \Omega \mapsto \Re^d$ with covariance $K$,
$$\diffentropy(X) \leq \frac{1}{2}\ln\left((2\pi e)^d|K|\right),$$
with equality iff $Pr(X\in\cdot) \sim \mathcal{N}(\mu, K)$ for some $\mu \in\Re^d$.
\end{lemma}
\begin{proof}
    Follows from Theorems 8.6.3 and 8.6.5 of \cite{10.5555/1146355}.
\end{proof}

We will cite this result extensively throughout the paper.\newline

Recall that $\regret(1) = \E[\KL(P^*_0 \| P_0)]$, where $P_0 = \Pr(Y_1 \in \cdot | X_0) = \Pr(Y_{t+1}\in\cdot|X_t) \text{ for all } t$.  For our scalar estimation context, $\regret(1)$ satisfies the following bounds:
\begin{lemma}
    \label{le:scalar-estimation-regret}
    For all $\sigma^2 \in \Re_{++}$ and real-valued random variables $\theta$ with variance $1$, if for all $x\in \Re$, $\environment(\cdot|x)\sim \mathcal{N}(\theta, \sigma^2)$ then
    $$ \frac{1}{2}\ln\left(1 + \frac{e^{2\diffentropy(\theta)}}{\sigma^22\pi e}\right) \leq \regret(1) \leq \frac{1}{2}\ln\left(1 + \frac{1}{\sigma^2}\right).$$
\end{lemma}
\begin{proof}
    Note that
    $$\regret(1) = \I(Y_{t+1};\theta|X_t) = \I(Y_{t+1};\theta) = \diffentropy(Y_{t+1}) - \diffentropy(Y_{t+1}|\theta) = \diffentropy(\theta+W_{t+1}) - \diffentropy(W_{t+1}).$$
    We first establish the lower bound:
    \begin{align*}
        \regret(1)
        & = \diffentropy(\theta + W_{t+1}) - \diffentropy(W_{t+1})\\
        & \overset{(a)}{\geq} \frac{1}{2}\ln\left(e^{2\diffentropy(\theta)} + e^{2\diffentropy(W_{t+1})}\right) - \frac{1}{2}\ln\left(2\pi e \sigma^2\right)\\
        & = \frac{1}{2}\ln\left(1 + \frac{e^{2\diffentropy(\theta)}}{\sigma^2 2\pi e}\right),
    \end{align*}
    where $(a)$ follows from the entropy power inequality.  We next establish the upper bound:
    \begin{align*}
    \regret(1)
    & =  \diffentropy(\theta + W_{t+1}) - \diffentropy(W_{t+1})\\
    & \overset{(a)}{\leq} \frac{1}{2}\ln\left(2\pi e (1 + \sigma^2)\right) - \frac{1}{2}\ln\left(2\pi e \sigma^2\right)\\
    & = \frac{1}{2}\ln\left(1 + \frac{1}{\sigma^2}\right),
    \end{align*}
    where $(a)$ follows from lemma \ref{le:max_entropy}.
\end{proof}

The upper and lower bounds suggest that $\regret(1)$ shrinks as the variance of the noise $\sigma^2$ increases. This may initially seem counterintuitive, but consider a situation in which $\sigma^2 \rightarrow 0$. In this case, $Y_{t} = \theta$ so $\I(Y_t;\theta) = \H(\theta) = \infty$ for continuous random variable $\theta$. Meanwhile, if $\sigma^2 \rightarrow \infty$, then $Y_t = W_t$ and so $\I(Y_t;\theta) = \I(W_t;\theta) = 0$ since $W_t \perp \theta$. $\regret(1)$ is larger for smaller $\sigma^2$ because with less noise, $Y_t$ conveys more about $\theta$.

An interesting case is when $\Pr(\theta\in\cdot)\sim\mathcal{N}(0, 1)$. In this setting, we have that the lower bound is:
$$\frac{1}{2}\ln\left(1 + \frac{e^{\ln(2\pi e)}}{\sigma^2 2\pi e}\right) = \frac{1}{2}\ln\left(1 + \frac{1}{\sigma^2}\right).$$
So for $\theta$ distributed Gaussian, $\regret(1) = \frac{1}{2}\ln\left(1 + \frac{1}{\sigma^2}\right)$ because the upper and lower bounds match. 

We now establish an upper bound on the \emph{rate-distortion function} that holds for all real-valued random variables $\theta$ with variance $1$.
\begin{theorem}{\bf(scalar estimation rate-distortion upper bound)}
    \label{th:scalar-estimation-general-upper-bound}
    For all $\sigma^2 \in \Re_{+}$, $\epsilon \in \left[0, \frac{1}{2}\ln\left(1 + \frac{1}{\sigma^2}\right)\right)$, and random variables $\theta:\Omega\mapsto\Re$ with variance $1$, if for all $x\in \Re$, $\environment(\cdot|x)\sim \mathcal{N}(\theta, \sigma^2)$, then
    $$\H_\epsilon(\environment) \leq \frac{1}{2}\ln\left(\frac{2\pi e}{e^{2\diffentropy(\theta)}} \cdot \frac{e^{2\regret(1)}-1}{e^{2\epsilon}-1}\right).$$
\end{theorem}
\begin{proof}
    Fix $\sigma^2 \in \Re_{++}, \epsilon\in \left[0, \frac{1}{2}\ln\left(1 + \frac{1}{\sigma^2}\right)\right)$, and consider a proxy $\proxytheta = \theta + V$ where $V\sim \mathcal{N}(0, \delta^2)$ for $\delta^2 = \frac{\sigma^2(e^{2\epsilon}-1)}{1 - \sigma^2(e^{2\epsilon}-1)}$ and $V \perp \theta$. Note that $\delta^2 \geq 0$ for all $\epsilon\in\left[0, \frac{1}{2}\ln\left(1 + \frac{1}{\sigma^2}\right)\right)$. We begin by upper bounding the rate of such a proxy:
    \begin{align*}
        \I(\theta;\proxytheta)
        & = \diffentropy(\proxytheta) - \diffentropy(\proxytheta|\theta)\\
        & = \diffentropy(\proxytheta) - \diffentropy(V)\\
        & \overset{(a)}{\leq} \frac{1}{2}\ln\left(2\pi e \left(\delta^2 + 1\right)\right) - \frac{1}{2}\ln\left(2\pi e \delta^2\right)\\
        & = \frac{1}{2}\ln\left(1 + \frac{1}{\delta^2}\right)\\
        & = \frac{1}{2}\ln\left(\frac{1}{\sigma^2\left(e^{2\regret\epsilon}-1\right)}\right),
    \end{align*}
    where $(a)$ lemma \ref{le:max_entropy}.\newline
    
    Now, we upper bound the distortion of the proxy:
        \begin{align*}
            \I(Y;\theta|\proxytheta, X)
            & = \diffentropy(Y_{t+1}|\proxytheta) - \diffentropy(Y_{t+1}|\theta)\\
            & = \diffentropy(W_{t+1} + \theta|\proxytheta) - \diffentropy(W_{t+1})\\
            & = \diffentropy\left(W_{t+1} + \left(\theta - \frac{1}{1+\delta^2}\proxytheta\right)\Big|\proxytheta\right) - \diffentropy(W_{t+1})\\
            & = \diffentropy\left(W_{t+1} + \left(\frac{\delta^2}{1+\delta^2}\theta + \frac{1}{1+\delta^2}V\right) \Big| \proxytheta\right) - \diffentropy(W_{t+1})\\
            & \leq \diffentropy\left(W_{t+1} + \left(\frac{\delta^2}{1+\delta^2}\theta + \frac{1}{1+\delta^2}V\right)\right) - \diffentropy(W_{t+1})\\
            & \overset{(a)}{\leq} \frac{1}{2}\ln\left(2\pi e\left(\sigma^2 + \frac{\delta^4}{(1 + \delta^2)^2} + \frac{\delta^2}{(1 + \delta^2)^2}\right)\right) - \frac{1}{2}\ln\left(2\pi e \sigma^2\right)\\
            & = \frac{1}{2}\ln\left(1 + \frac{\delta^2}{(1 + \delta^2)\sigma^2}\right)\\
            & = \frac{1}{2}\ln\left(e^{2\epsilon}\right)\\
            & = \epsilon,
        \end{align*}
    where $(a)$ follows from lemma \ref{le:max_entropy}.\newline
        
    It follows from our characterizations of rate and distortion that $\proxytheta \in \proxyset_\epsilon$ and the rate-distortion function is upper bounded as follows:
    \begin{align*}
    \H_\epsilon(\environment)
    & \leq \I(\theta;\proxytheta)\\
    & \leq \frac{1}{2}\ln\left(\frac{1}{\sigma^2(e^{2\epsilon}-1)}\right)\\
    & \overset{(a)}{\leq} \frac{1}{2}\ln\left(\frac{2\pi e}{e^{2\diffentropy(\theta)}} \frac{e^{2\regret(1)}-1}{e^{2\regret(1)\epsilon}-1}\right),
    \end{align*}
    where $(a)$ follows from the lower bound of Lemma \ref{le:scalar-estimation-regret}.
\end{proof}

We now study the special case where $\theta \sim \mathcal{N}(0, 1)$. In this case, we will see that Theorem \ref{th:scalar-estimation-general-upper-bound} is met with \emph{equality}. We show this by proving a matching lower bound. Note that while we study the case in which $\theta$ is distributed standard Gaussian, the results trivially extend to the cases in which $\theta$ is distributed Gaussian with arbitrary mean and variance.

\begin{restatable}{theorem}{scaest}{\bf (scalar estimation gaussian rate-distortion lower bound)}
    \label{th:scalar-estimation-rate-distortion-lower}
    For all $\sigma^2, \epsilon \in \Re_+$, if $\theta \sim \normal(0,1)$ and if for all $x\in \Re$, $\environment(\cdot|x)\sim \mathcal{N}(\theta, \sigma^2)$, then
    $$\ratedistortion \geq \frac{1}{2}\ln \frac{e^{2\regret(1)}-1}{e^{2 \epsilon}-1}.$$
\end{restatable}
\begin{proof}
    Fix $\sigma^2 \in \Re_{++}$, $\epsilon \in \mathbb{Z}_+$, and $\proxy \in \proxyset_\epsilon$. We have
    \begin{align*}
    \regret(1) \epsilon
    &\overset{(a)}{\geq} \I(Y_{t+1};\environment|\proxy, X_t) \\
    &= \diffentropy(Y_{t+1}|\proxy, X_t) - \diffentropy(Y_{t+1}|\environment,\proxy,X) \\
    &= \diffentropy(\theta + W_{t+1}|\proxy) - \diffentropy(W_{t+1}) \\
    &= \diffentropy(\theta+W_{t+1}|\proxy) - \frac{1}{2}\ln\left(2\pi e\sigma^2\right) \\
    &\overset{(b)}{\geq} \frac{1}{2}\ln\left(e^{2\diffentropy(W_{t+1})}+e^{2\diffentropy(\theta|\proxy)}\right) - \frac{1}{2}\ln\left(2\pi e\sigma^2\right)\\
    & = \frac{1}{2}\ln\left(1 + \frac{e^{2\diffentropy(\theta|\proxy)}}{2\pi e\sigma^2}\right).
    \end{align*}
    where (a) follows from the fact that $\proxy \in \proxyset_\epsilon$ and $(b)$ follows from the conditional entropy power inequality.  Rearranging the resulting inequality, we obtain
    \begin{equation}\label{eq:diff_ent_bound}
    \diffentropy(\theta|\proxy) \leq \frac{1}{2}\ln\left(2\pi e\sigma^2\left(e^{2\epsilon}-1\right)\right).  
    \end{equation}
    It follows that
    \begin{align*}
    \I(\environment;\proxy)
    &= \diffentropy(\theta)-\diffentropy(\theta|\proxy)\\
    &\geq \frac{1}{2}\ln\left(2\pi e\right) - \frac{1}{2}\ln\left(2\pi e\sigma^2\left(e^{2\epsilon} -1\right)\right) \\
    &= \frac{1}{2}\ln\left(\frac{1}{\sigma^2\left(e^{2\epsilon}-1\right)}\right)\\
    &= \frac{1}{2}\ln\left(\frac{e^{2\regret(1)}-1}{e^{2\epsilon}-1}\right).
    \end{align*}
    Since $\proxy$ is an arbitrary element of $\proxyset_\epsilon$, the result follows:
    $$\H_\epsilon(\environment) = \inf_{\proxy \in \proxyset_\epsilon} \I(\environment; \proxy) \geq \frac{1}{2}\ln\left(\frac{e^{2\regret(1)}-1}{e^{2\epsilon}-1}\right).$$
\end{proof}

\begin{figure}[!ht]
\centering
\includegraphics[scale=0.45]{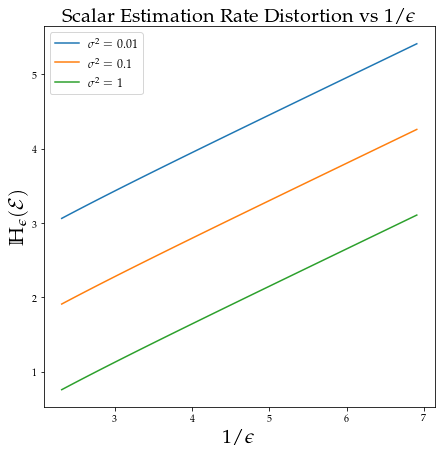}
\caption{The rate-distortion function for scalar estimation under various values of noise variance $\sigma^2$.}
\label{fig:scalar-estimation-rate-distortion}
\end{figure}

For $\theta$ distributed Gaussian, Theorems \ref{th:scalar-estimation-general-upper-bound} and \ref{th:scalar-estimation-rate-distortion-lower} establish matching upper and lower bounds. To succinctly present the rate-distortion results of this section, we provide the following corollary.
\begin{restatable}{corollary}{scaestRd}{\bf (scalar estimation rate-distortion function)}
\label{co:scalar-estimation-rate-distortion}
    For all $\sigma^2,  \in \Re_{++}$, $\epsilon \in \left[0, \frac{1}{2}\ln\left(1 + \frac{1}{\sigma^2}\right)\right)$, and random variables $\theta:\Omega\mapsto\Re$ with variance $1$, if for all $x \in \Re$, $\environment(\cdot|x) \sim \normal(\theta, \sigma^2)$, then
    $$\ratedistortion \leq \frac{1}{2}\ln\left(\frac{ 2\pi e}{e^{2\diffentropy(\theta)}} \cdot \frac{e^{2\regret(1)}-1}{e^{2 \epsilon}-1} \right).$$
    Further, if $\theta\sim \mathcal{N}(0, 1)$, then
    $$\ratedistortion = \frac{1}{2}\ln \frac{e^{2\regret(1)}-1}{e^{2 \epsilon}-1}.$$
\end{restatable}

Figure \ref{fig:scalar-estimation-rate-distortion} plots the rate-distortion function established by Corollary \ref{co:scalar-estimation-rate-distortion} for the case of $\Pr(\theta\in\cdot)\sim\mathcal{N}(0, 1)$ and noise variance $\sigma^2 = 0.1$.  As is to be expected, the rate monotonically decreases in the distortion.  Further, as the rate grows unbounded as the distortion vanishes.  From Figure \ref{fig:scalar-estimation-rate-distortion}, we notice that the rate-distortion function is roughly linear in $\frac{1}{\epsilon}$ and logarithmic in $\frac{1}{\sigma^2}$

For scalar estimation with Gaussian $\theta$, $\regret(1) = \frac{1}{2}\ln\left(1 + \frac{1}{\sigma^2}\right)$.  Note that $1/\sigma^2$ represents a signal-to-noise ratio (SNR).  For any given level of distortion $\epsilon$, the rate $\H_\epsilon(\environment)$ characterized by Corollary \ref{co:scalar-estimation-rate-distortion} increases with the SNR.  This is intuitive.  With zero SNR, $Y$ is unpredictable and knowledge of $\theta$ is not helpful, as reflected by the fact that $\regret = 0$.  On the other hand, when the SNR is asymptotically large, knowledge of $\theta$ enables perfect prediction of $Y$, which is infinitely better than what can be offered by an uninformed agent.

\subsection{Linear Regression}\label{subsec:lin_reg}
Let us next consider linear regression, where the environment $\environment$ is identified by a vector $\theta\in\Re^d$ with iid components each with unit variance.  Inputs and outputs are generated according to random vector $X$ with $\Pr(X\in\cdot)\sim \mathcal{N}(0, I_d)$ and $Y = \theta^\top X + W$ where $W$ is a random variable with $\Pr(W\in\cdot) \sim \mathcal{N}(0, \sigma^2)$ for some $\sigma^2 \in \Re_{++}$, and $W \perp \theta$. Hence, $\environment(\cdot|x) \sim \normal(\theta^\top x, \sigma^2)$. Note that the results and techniques developed in this section certainly extend to input distributions that are not Gaussian with slight modifications. We study the Gaussian case since it is a canonical example and often simplifies analysis. We first establish an analogue to the maximum differential entropy result of lemma \ref{le:max_entropy} that applies to random vectors with a fixed \emph{sum of} variances. 

\begin{lemma}
    \label{le:max_sum_ent}
    For all real-valued random vectors $X: \Omega\mapsto\Re^d$ where $\kappa = {\rm trace}({\rm cov}[X])$,
    $$\diffentropy(X) \leq \frac{d}{2}\ln\left(2\pi e\frac{\kappa}{d}\right),$$
    with equality iff $\Pr(X\in\cdot)\sim \mathcal{N}(\mu, \frac{\kappa}{d}I_d)$ for some $\mu\in\Re^d$.
\end{lemma}
\begin{proof}
    Let $K$ be the covariance matrix of $X$. Next, let $\lambda_1, \hdots, \lambda_d$ denote the eigenvalues of $K$. Then, we have that
    \begin{align*}
        \diffentropy(X)
        & \overset{(a)}{\leq} \frac{1}{2}\ln\left((2\pi e)^d |K|\right)\\
        & = \frac{1}{2}\ln\left((2\pi e)^d \prod_{i=1}^{d} \lambda_i\right)\\
        & \overset{(b)}{\leq} \frac{1}{2}\ln\left((2\pi e)^d \left(\frac{\kappa}{d}\right)^d\right)\\
        & = \frac{d}{2}\ln\left(2\pi e \frac{\kappa}{d}\right),
    \end{align*}
    where $(a)$ follows from lemma \ref{le:max_entropy} and $(b)$ follows from the fact that $\sum_{i=1}^{d}\lambda_i = \kappa$ and the fact that the product is maximized when all the $\lambda_i$ are equal. The equality result follows from applying lemma \ref{le:max_entropy} to a random vector with covariance $K = \frac{\kappa}{d}I_d$.
\end{proof}

We next establish upper and lower bounds for nominal regret $\regret(1)$.
\begin{lemma}
    \label{le:regret-lin-reg-general}
    For all $d \in \mathbb{Z}_{++}$ s.t. $d\geq 2$,  $\sigma^2 \in \Re_{++}$, and random vectors $\theta :\Omega\mapsto \Re^d$ with iid components, each with variance $1$, if $\Pr(X_0\in\cdot)\sim\normal(0, I_d)$, and if for all $x \in \Re^d$, $\environment(\cdot|x)\sim \mathcal{N}(\theta^\top x, \sigma^2)$, then
    $$\frac{1}{6}\ln\left(1 + \frac{d}{\sigma^2}\frac{e^{2\diffentropy(\theta)/d}}{2\pi e}\right) \leq \regret(1) \leq \frac{1}{2}\ln\left(1 + \frac{d}{\sigma^2}\right).$$
\end{lemma}
\begin{proof}
    We begin by proving the lower bound:
    \begin{align*}
        \I(Y_{t+1};\theta|X_t)
        & = \diffentropy(Y_{t+1}|X_t) - \diffentropy(Y_{t+1}|\theta, X_t)\\
        & = \diffentropy(W_{t+1} + \theta^\top X_t|X_t) - \diffentropy(W_{t+1})\\
        & \overset{(a)}{\geq} \E\left[\frac{1}{2}\ln\left(2\pi e\sigma^2 + \sum_{i=1}^{d}e^{2\diffentropy(\theta_i X_{t,i}|X=X)}\right)\right] - \frac{1}{2}\ln\left(2\pi e \sigma^2\right)\\
        & \overset{(b)}{=} \E\left[\frac{1}{2}\ln\left(1 + \frac{\sum_{i=1}^{d}e^{2\diffentropy(\theta)/d }|X_{t,i}|^2}{2\pi e \sigma^2}\right)\right]\\
        & \geq \frac{\Pr\left(\|X_t\|_2^2 \geq d\right)}{2}\ln\left(1 + \frac{d}{\sigma^2}\frac{e^{2\diffentropy(\theta)/d}}{2\pi e}\right)\\
        & \overset{(c)}{\geq} \frac{1}{6}\ln\left(1 + \frac{d}{\sigma^2}\frac{e^{2\diffentropy(\theta)/d}}{2\pi e}\right),\\
    \end{align*}
    where $(a)$ follows from the entropy power inequality and $\theta_i, X_{t,i}$ denote the $i$th component of $\theta$ and $X_t$ respectively, $(b)$ follows from the fact that for a constant $a$, $\diffentropy(a\theta) = \diffentropy(\theta) + \ln(|a|)$, and $(c)$ follows from the fact that for $d \geq 2$, $\Pr(\|X_t\|^2_2\geq d) \geq \frac{1}{3}$ for $\|X_t\|_2^2$ distributed $\chi^2_d$.\newline
    
    We next prove the upper bound.
    \begin{align*}
        \regret(1)
        &= \I(Y_{t+1};\theta|X_t)\\
        &= \diffentropy(Y_{t+1}|X_t) - \diffentropy(Y_{t+1}|\theta, X_t)\\
        &= \diffentropy(Y_{t+1}|X_t) - \diffentropy(W_{t+1})\\
        &\overset{(a)}{\leq} \E\left[\frac{1}{2}\ln\left(2\pi e\left(\sigma^2 + \|X_t\|^2_2\right)\right)\right] - \frac{1}{2}\ln\left(2\pi e \sigma^2\right)\\
        &= \E\left[\frac{1}{2}\ln\left(1 + \frac{\|X_{t}\|^2_2}{\sigma^2}\right)\right]\\
        &\overset{(b)}{\leq} \frac{1}{2}\ln\left(1 + \frac{\E\left[\|X_t\|^2_2\right]}{\sigma^2}\right)\\
        &\leq \frac{1}{2}\ln\left(1 + \frac{d}{\sigma^2}\right),\\
    \end{align*}
    where $(a)$ follows from lemma \ref{le:max_entropy} and $(b)$ follows from Jensen's inequality.
\end{proof}

Just as in scalar estimation, the upper and lower bounds suggest that $\regret(1)$ vanishes as the variance $\sigma^2$ of the noise increases because with less noise, $Y$ conveys more about $\theta$. The bounds also suggest that $\regret(1)$ grows with $d$, which is intuitive since $\theta$ encodes more information when $d$ is larger.

In the case where $\Pr(\theta\in\cdot)\sim \mathcal{N}(0, I_d)$, the lower bound becomes:
$$\frac{1}{6}\ln\left(1 + \frac{d}{\sigma^2}\right),$$
which closely resembles the upper bound.

We now derive an upper bound for the rate-distortion function in the linear regression setting.
\begin{theorem}{\bf (linear regression rate-distortion upper bound)}
\label{th:lin-reg-rd-general-upper-bound}
For all $d \in \mathbb{Z}_{++}$ s.t. $d \geq 2$, $\sigma^2 \in \Re_{+}$, $\epsilon \in \left[0, \frac{1}{2}\ln\left(1 + \frac{d}{\sigma^2}\right)\right)$, and random vectors $\theta: \Omega\mapsto\Re^d$ with iid components, each with variance $1$, if for all $t$, $\Pr(X_t\in\cdot)\sim \normal(0, I_d)$, and if for all $x \in \Re^d$, $\environment(\cdot|x) \sim \normal(\theta^\top x, \sigma^2)$, then
$$\H_{\epsilon}(\environment) \leq \frac{d}{2}\ln\left(\frac{e^{6\regret(1)}-1}{e^{2\epsilon}-1} \frac{2\pi e}{e^{2\diffentropy(\theta)/d}}\right).$$ 
\end{theorem}
\begin{proof}
    Fix $\sigma^2\in \Re_{++}$. Let the proxy $\tilde{\theta} = \theta + V$, where $V \sim \mathcal{N}(0, \delta^2I_d)$ for $\delta^2 = \frac{\sigma^2(e^{2\epsilon}-1)}{d - \sigma^2(e^{2\epsilon}-1)}$ and $V\perp \theta$. Note that $\delta^2 \geq 0$ for all $\epsilon \in \left[0, \frac{1}{2}\ln\left(1 + \frac{d}{\sigma^2}\right)\right)$. We first upper bound the rate of such proxy:
    \begin{align*}
        \I(\theta;\proxytheta)
        & = \diffentropy(\proxytheta) - \diffentropy(\proxytheta|\theta)\\
        & = \diffentropy(\proxytheta) - \diffentropy(V)\\
        & \overset{(a)}{\leq} \frac{d}{2}\ln\left(2\pi e \left(\delta^2 + 1\right)\right) - \frac{1}{2}\ln\left(\left(2\pi e \delta^2\right)\right)\\
        & = \frac{d}{2}\ln\left(1 + \frac{1}{\delta^2}\right)\\
        & = \frac{d}{2}\ln\left(\frac{d}{\sigma^2\left(e^{2\epsilon}-1\right)}\right),
    \end{align*}
    where $(a)$ follows from Lemma \ref{le:max_sum_ent}.
    
    Now, we upper bound the distortion of such proxy:
    \begin{align*}
        \I(Y;\theta|\proxytheta, X)
        & = \diffentropy(Y|\proxytheta, X) - \diffentropy(Y|\theta, X)\\
        & = \diffentropy(W + \theta^\top X|\proxytheta, X) - \diffentropy(W)\\
        & = \diffentropy(W + (\theta - \frac{1}{1+\delta^2}\proxytheta)^\top X|\proxytheta, X) - \diffentropy(W)\\
        & = \diffentropy(W + \left(\frac{\delta^2}{1+\delta^2}\theta + \frac{1}{1+\delta^2}V\right)^\top X|\proxytheta, X) - \diffentropy(W)\\
        & \leq \diffentropy(W + \left(\frac{\delta^2}{1+\delta^2}\theta + \frac{1}{1+\delta^2}V\right)^\top X|X) - \diffentropy(W)\\
        & \overset{(a)}{\leq} \E\left[\frac{1}{2}\ln\left(2\pi e\left(\sigma^2 + \left(\frac{\delta^4}{(1 + \delta^2)^2} + \frac{\delta^2}{(1 + \delta^2)^2}\right)\|X\|^2_2\right)\right)\right] - \frac{1}{2}\ln\left(2\pi e \sigma^2\right)\\
        & = \E\left[\frac{1}{2}\ln\left(1 + \frac{\delta^2\|X\|^2_2}{(1 + \delta^2)\sigma^2}\right)\right]\\
        & \overset{(b)}{\leq} \frac{1}{2}\ln\left(1 + \frac{d\delta^2}{(1+\delta^2)\sigma^2}\right)\\
        & = \frac{1}{2}\ln\left(e^{2\epsilon}\right)\\
        & = \epsilon,
    \end{align*}
    where $(a)$ follows from lemma \ref{le:max_entropy} and $(b)$ follows from Jensen's inequality.\newline
    
    Therefore, $\proxytheta \in \proxyset_\epsilon$ and the rate-distortion function is upper bounded as follows:
    \begin{align*}
        \H_\epsilon(\environment)
        & \leq \I(\theta;\proxytheta)\\
        & \leq \frac{d}{2}\ln\left(\frac{d}{\sigma^2(e^{2\epsilon}-1)}\right)\\
        & \overset{(a)}{\leq} \frac{d}{2}\ln\left(\frac{e^{6\regret(1)}-1}{e^{2\epsilon}-1} \frac{2\pi e}{e^{2\diffentropy(\theta)/d}}\right),
    \end{align*}
    where $(a)$ follows from the lower bound from Lemma \ref{le:regret-lin-reg-general}.
\end{proof}

The following results assume that $\theta:\Omega\mapsto\Re^d$ consists of iid $1$-subgaussian and symmetric elements. Under this assumption, we can establish both upper \emph{and lower} bounds on the rate-distortion function for linear regression. While this analysis trivially extends to the case in which $\theta$ has arbitrary mean (and is symmetric about that mean) and independent (but not necessarily identically distributed) components, for simplicity of notation, we study the zero-mean iid case.

We establish a lower bound by first finding a suitable lower bound for the distortion function. For subgaussian random vectors, the following lemma allows us to lower bound the expected KL-divergence distortion by a multiple of the mean squared error. We provide the proof for Lemma \ref{le:mse-mutual-info-inequality} and related lemmas in Appendix \ref{apdx:lin_reg_lb}.
\begin{restatable}{lemma}{mseInfoInequality}
    \label{le:mse-mutual-info-inequality}
    For all $\proxy\in\proxyset$, $d \in \mathbb{Z}_{++}$ and $\sigma^2 \in \Re_{++}$, if $\theta:\Omega\mapsto\Re^d$ consists of iid components each of which are $1$-subgaussian and symmetric, $\Pr(X\in\cdot)\sim \normal(0, I_d)$, and if $Y \sim \normal(\theta^\top X, \sigma^2)$, then
    $$\E\left[\frac{1}{2(4\|X\|_2^2 + \sigma^2)}\right]\E\left[\|\theta - \E[\theta|\proxy]\|^2_2\right] \leq \I(Y;\theta|\proxy, X).$$
\end{restatable}

With this result in place, we now provide a lower bound for the rate-distortion function.

\begin{restatable}{theorem}{linreg_lb}{\bf (subgaussian linear regression rate-distortion lower bound)}
\label{th:linear-regression-rate-distortion-lower-bound}
For all $d \in \mathbb{Z}_{++}$ s.t. $d > 2$, $\sigma^2 \geq 0$ and $\epsilon \in\left[0, \frac{1}{2(4d+\sigma^2)}\right]$, if $\theta:\Omega\mapsto\Re^d$ consists of iid components that are each $1$-subgaussian and symmetric, $\Pr(X\in\cdot)\sim \normal(0, I_d)$, and if $Y \sim \normal(\theta^\top X, \sigma^2)$, then
$$\H_{\epsilon}(\environment) \geq \frac{d}{2}\ln\left(\frac{d}{2\left(4d +\sigma^2\right)\epsilon}\right).$$ 
\end{restatable}
\begin{proof}
    Fix $\sigma^2 \in \Re_{++}$, $\epsilon \in \mathbb{Z}_+$, and a proxy $\proxy \in \proxyset_\epsilon$. Then,
    \begin{align*}
        \epsilon
        & \overset{(a)}{\geq}  \I(Y;\theta|\proxy, X)\\
        & \overset{(b)}{\geq} \E\left[\frac{\|X\|_2^2}{2(4\|X\|_2^2+\sigma^2)}\right]\E\left[\|\theta-\E[\theta|\proxy]\|^2_2\right],
    \end{align*}
    where $(a)$ follows from the fact that $\proxy \in \proxyset_\epsilon$ and $(b)$ follows from Lemma \ref{le:mse-mutual-info-inequality}. As a result, we have that $\proxy\in\proxyset_\epsilon$ implies the following:
    \begin{align*}
        \E\left[\|\theta-\E[\theta|\proxy]\|^2_2\right]
        & \leq \frac{1}{\E\left[\frac{1}{2(4\|X\|_2^2+\sigma^2)}\right]}\epsilon\\
        & \leq \E\left[8\|X\|^2_2 + 2\sigma^2\right]\epsilon\\
        & = 2(4d+\sigma^2)\epsilon,
    \end{align*}
    where $(a)$ follows from Jensen's inequality.
    
    Since the above condition is an implication that holds for arbitrary $\proxy\in\proxyset_\epsilon$, minimizing the rate $\I(\environment;\proxy)$ over the set of proxies that satisfy 
    $\E\left[\|\theta-\E[\theta|\proxy]\|^2_2\right] \leq 2\left(4d + \sigma^2\right) \epsilon$ will provide a lower bound. However, this is simply the rate-distortion problem for a multivariate source under squared error distortion which is a well known lower bound (Theorem 10.3.3 of \citep{10.5555/1146355}). The lower bound follows as a result.\newline
\end{proof}

\begin{figure}[!ht]
    \centering
    \includegraphics[scale=0.5]{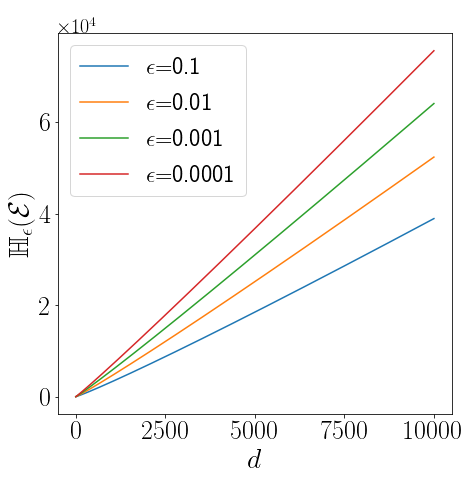}
    \caption{We plot the rate-distortion upper bound from Theorem \ref{th:lin-reg-rd-general-upper-bound} for $\theta\sim\normal(0, I_d)$ and $W \sim \normal(0, 0.1)$ as a function of the dimension $d$ for various levels of distortion $\epsilon$.  The plots suggest that $\H_\epsilon(\environment) = \tilde{\mathcal{O}}(d)$ and that rate increases as the distortion decreases. Note that since the y-axis is in log scale, the graph also suggests that the rate-distortion function has a logarithmic dependence on $\frac{1}{\epsilon}$}.
    \label{fig:lin_reg_rd_sc}
\end{figure}
Now, these results suggest the following sample complexity bounds for linear regression:
\begin{theorem}{\bf (subgaussian linear regression sample complexity bounds)}
    For all $d \in \mathbb{Z}_{++}$ s.t. $d > 2$, $\sigma^2 \geq 0$, $\epsilon \in\left[0, \frac{1}{2(4d+\sigma^2)}\right]$, and random vectors $\theta:\Omega\mapsto\Re^d$ consisting of iid components that are each $1$-subgaussian and symmetric, if for all $t$, $\Pr(X_t\in\cdot)\sim \normal(0, I_d)$, and if for all $x \in \Re^d$, $\environment(\cdot|x) \sim \normal(\theta^\top x, \sigma^2)$, then
    $$
    \frac{d}{2\epsilon}\ln\left(\frac{d}{2(4d+\sigma^2)\epsilon}\right)\ \leq\ T_\epsilon\ \leq\ \frac{d}{\epsilon}\ln\left(\frac{d}{\sigma^2\epsilon}\right)$$
\end{theorem}
\begin{proof}
    The result follows from Theorems \ref{th:lin-reg-rd-general-upper-bound}, \ref{th:linear-regression-rate-distortion-lower-bound}, and \ref{th:general-sample-bound}.
\end{proof}

\subsection{Linear Regression with a Misspecified Model}
\label{sec:misspecified_models}

In this section, we will study two instances of linear regression in which the model used by the agent is misspecified. We first study the case in which the agent's prior over the environment has an incorrect mean. As one may expect, in this case, with enough data, the agent will still be able to arrive at the correct model. In the second instance, the agent's prior will be missing a feature. In this instance, we will show that an irreducible error will linger even as $T\rightarrow \infty$.

Both of these instances will hinge upon the following result.
\begin{corollary} {\bf (misspecified/suboptimal prediction)}\label{cor:misspecification}
    For all $t\geq 0$ and $\pi$,
    $$\E[\KL(P_t^*\|P_t)\ |\ H_t] = \E[\KL(P^*_t\|\hat{P}_t)\ |\ H_t] + \E[\KL(\hat{P}_t\|P_t)\ |\ H_t],$$
    where $P_t = \pi(H_t, Z_t)$.
\end{corollary}
The first term on the RHS of Corollary \ref{cor:misspecification} is the error of the correctly specified optimal learner. To study the shortfall, we will bound the behavior of the second term $\E[\KL(\hat{P}_t\|P_t)|H_t]$ in the two aforementioned problem instances. Proofs for these results may be found in Appendix \ref{apdx_sec:misspecification}.

\subsubsection{Prior with Incorrect Mean}
Let the data generating process be the same as in Section \ref{subsec:lin_reg}. However, let the agent's prior $\Pr(\theta\in\cdot) \sim \mathcal{N}(\mu, I_d)$ for some $\mu\in\Re^d$. Note that when $\mu = 0$ the prior is \emph{correctly} specified. We will study the regret and sample complexity of this agent.

\begin{restatable}{theorem}{wrongMean}{\bf (incremental error of mean-misspecified agent)}
    \label{th:misspecified_1}
    For all $d, t \in \mathbb{Z}_{++}$ and $\mu \in \Re^d$, if $t \geq 4d$, then
    $$\E[\KL(\hat{P}_t \| P_t)] \leq d\|\mu\|_2^2\left(\frac{2}{t^2} + \frac{1}{2\sigma^2}e^{-\frac{\left(\frac{1}{2}\sqrt{t}-\sqrt{d}\right)^2}{2}}\right),$$
    where $P_t$ is the posterior distribution $\Pr(Y_{t+1}|H_t)$ with misspecified prior $\Pr(\theta)\sim\normal(\mu, I_d)$.
\end{restatable}
\begin{figure}
    \centering
    \includegraphics[scale=0.5]{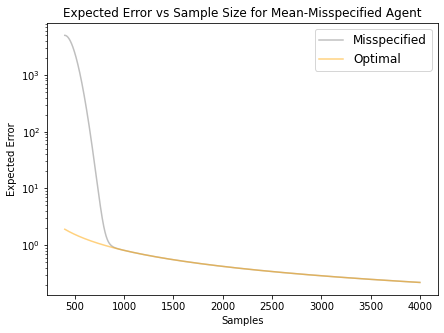}
    \caption{Here we plot $\E[\KL(P^*_t\|P_t)]$ for the mean-misspecified agent using Corollary \ref{cor:misspecification} and Theorem \ref{th:misspecified_1}. We take $d = 100, \sigma^2 = 0.01$, and $\|\mu\|^2_2 = 1$. We observe that the error incurred by misspecification rapidly decays to 0.}
    \label{fig:mean_misspecified}
\end{figure}
The proof can be found in Appendix \ref{apdx_sec:misspecification}. If we analyze Theorem \ref{th:misspecified_1} and Figure \ref{fig:mean_misspecified}, we see that despite having a misspecified model, the excess error $\E[\KL(\hat{P}_t \| P_t)]$ goes to $0$ at a rate of $\frac{1}{t^2}$. Also as expected, models with $\mu$ larger in magnitude (and hence greater misspecification) require more samples to wash out. With Theorem \ref{th:misspecified_1} and similar techniques, one can derive regret/sample complexity bounds for \emph{suboptimal} agents.

\subsubsection{Prior with Missing Feature}
We will now study another instance of misspecification under our framework for which the excess error does not decay to $0$ as $t\rightarrow\infty$. Let the data generating process be the same as in Section \ref{subsec:lin_reg}. For $i\in \{1, \hdots, d\}$, let $\theta_i$ denote the $i$th element of $\theta$. Let the agent's prior be correctly specified for $i \in \{1, \hdots, d-1\}$ ($\Pr(\theta_i\in\cdot) \sim \mathcal{N}(0, 1)$), but suppose the final element's prior is incorrectly $\Pr(\theta_d\in\cdot) = \mathbbm{1}[\theta_d = 0]$. We will study the regret and sample complexity of this agent.

\begin{restatable}{theorem}{missingFeature}{\bf (incremental error of missing feature agent)}
    \label{th:misspecified_2}
    For all $d, t \in \mathbb{Z}_{++}$ and $\mu \in \Re^d$, if $P_t(Y_t|H_t)$ is the postersior distribution of $Y_t$ conditioned on $H_t$ with the incorrect prior $\Pr(\theta_d\in\cdot)\sim \mathbbm{1}[\theta_d = 0]$, then
    $$\lim_{t\rightarrow\infty}\E\left[\KL(\hat{P}_t\|P_t)\right] = \frac{1}{2}\E\left[\ln\left(1 + \frac{X^2_{t,d}}{\sigma^2}\right)\right],$$
    where $X_{t,d}$ denotes the $d$th component of $X_t$.
\end{restatable}
Theorem \ref{th:misspecified_2} suggests that for an agent that is oblivious to one of the features, the incremental error will never go below $\frac{1}{2}\E\left[\ln\left(1 + \frac{X^2_{t,d}}{\sigma^2}\right)\right]$ in expectation. This is intuitive as the true label $Y$ depends on the omitted feature. By ignoring that feature, the agent always leaves a bit of performance at the table.

\section{Deep Neural Network Environments}
\label{sec:deep-neural-network}
In this section, we will focus on characterizing the rate-distortion function, and hence the sample complexity, of different deep neural network environments. As seen in the previous section, for the analysis of rate-distortion, it suffices to restrict attention to a representative input-output pair rather than a sequence, i.e., the distortion depends on one representative input-output pair and not the sequence $((X_t, Y_{t+1}): t\in\mathbb{Z}_+ )$.  We will denote this representative pair by $(X,Y)$.  The input $X$ is distributed $\Pr(X \in \cdot) \sim \mathcal{N}(0,I)$, while the conditional distribution of the output is $\Pr(Y \in \cdot | \environment, X) = \environment(\cdot|X)$.

For environments we consider in this section, $X$ takes values in $\Re^d$ and $Y$ take values in $\Re$, and $Y = f(X) + W$ for a random function $f$ and random variable $W \sim \mathcal{N}(0, \sigma^2 I_d)$.  We assume $X$, $f$, and $W$ are independent.  The environment is produced by composing $K$ independent and identically distributed random functions: $f = f_K \circ \cdots \circ f_1$.  In this sense, the environment is multilayer, with each $k$th layer represented by a function $f_k$.  We denote inputs and outputs of these functions by $U_0=X$ and $U_k=f_k(U_{k-1})$ for $k=1,\ldots,K$.  Hence, $Y = U_K + W$.  Figure \ref{fig:multilayer-environment} illustrates the structure of such an environment.

\begin{figure}[!ht]
\centering
\includegraphics[scale=0.35]{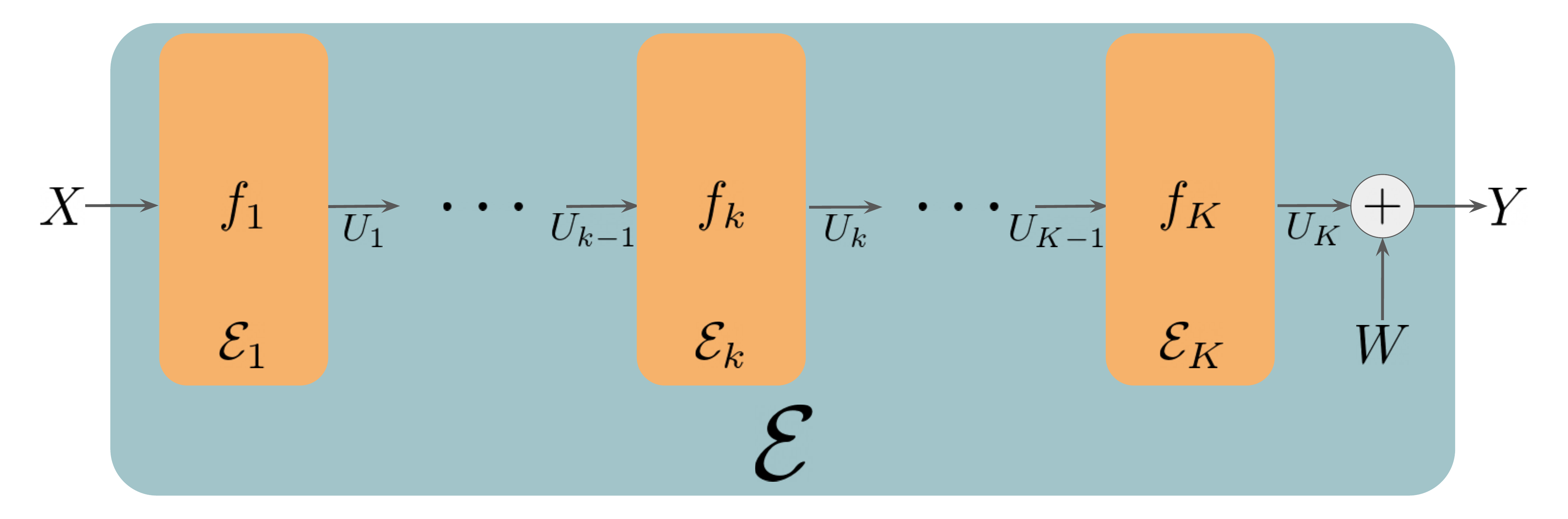}
\caption{A multilayer environment.}
\label{fig:multilayer-environment}
\end{figure}

Our analysis will relate the rate-distortion function $\H_\epsilon(\environment)$ of the multilayer environment to that of $K$ single-layer environments.  Each such single-layer environment, which we denote by $\environment_k$, takes the form $\environment_k(\cdot|u) \sim \mathcal{N}(f_k(u), \sigma^2 I_d)$.  In other words, conditioned on $\environment_k$ and the input $U_{k-1}$, the output of $\environment_k$ is distributed according to $U_k + W$.

To frame our our results, we define a class of proxies that decompose independently accross layers.  Recall that an {\it environment proxy} of an environment $\environment$ is a random variable $\proxy$ for which $\proxy \perp H_\infty | \environment$.  Similarly, an {\it environment proxy} of $\environment_k$ is a random variable $\proxy_k$ for which $\proxy_k \perp H_\infty | \environment_k$.  This definition allows for dependence between the proxies across layers even though we have assumed environments to be iid across layers.  To restrict attention to independent single-layer proxies, we define a {\it multilayer proxy} to be a tuple $\proxy = (\proxy_1,\ldots,\proxy_K)$ such that $\proxy_k \perp (\environment_{\lnot k}, \proxy_{\lnot k}, H_\infty) | \environment_k$, where $\environment_{\lnot k}$ and $\proxy_{\lnot k}$ denote tuples of single-layer environments and proxies, with the $k$th omitted.

\subsection{Prototypical Neural Network Environment}
In this section, we present two multilayer neural network environment for which  we will eventually study the sample complexity of.  We will see that the choice of prior will influence the types of bounds that are possible to derive.  We hope that these examples give the reader a wide enough breadth of techniques to analyze their own interesting multilayer environments.

\begin{figure}[!ht]
    \centering
    \includegraphics[scale=0.35]{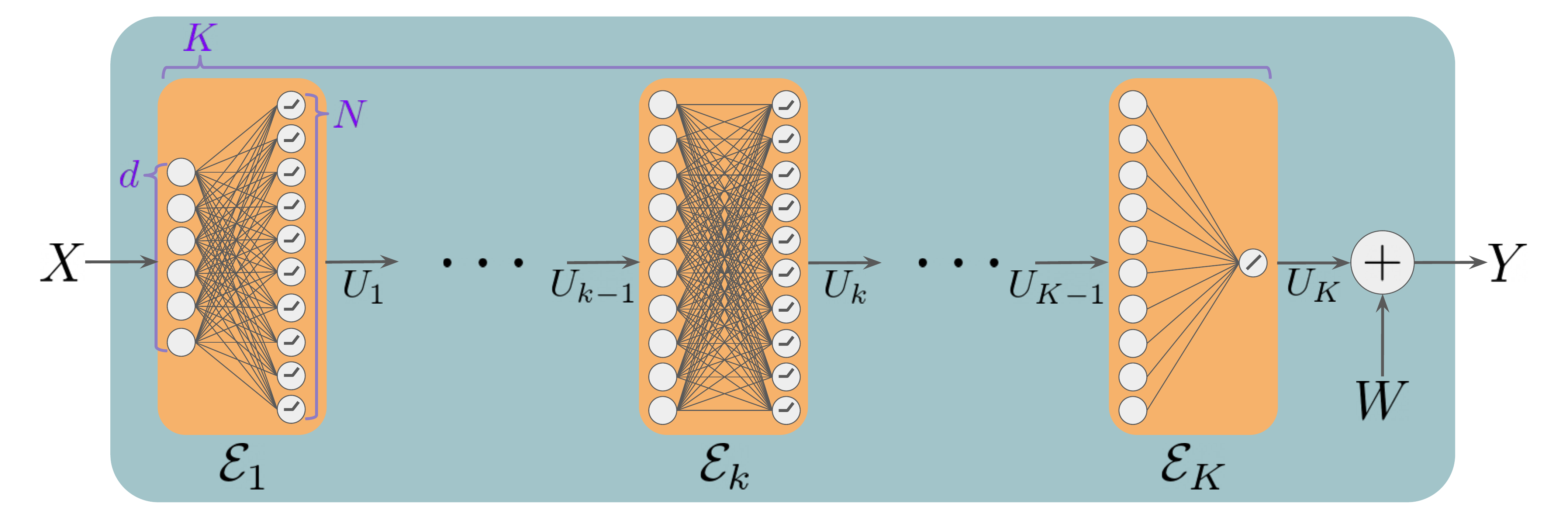}
    \caption{Prototypical Neural Network Environment}
    \label{fig:proto_env1}
\end{figure}

Our first prototypical neural network environment mirrors is a fully-connected feed-forward neural network with ReLU activations (Figure \ref{fig:proto_env1}). Let
$$U_1 = f_1(U_0) = \relu(A^{(1)}U_0),$$
where $A^{(1)}\in\Re^{N\times d}$.
For $k \in\{2, \hdots, K-1\}$, let
$$U_{k} = f_k(U_{k-1}) = \relu(A^{(k)}U_{k-1}),$$
where $A^{(k)} \in \Re^{N\times N}\in\Re^N$. For the final layer, let
$$U_K = f_K(U_{K-1}) = A^{(K)\top} U_{K-1},$$
where $A^{(K)}\in\Re^{N}$. In this environment, $\environment_k$ is identified by $(A^{(k)})$ for $k \in \{1, \hdots, K-1\}$ and $\environment_K$ is identified by $A^{(K)}$.

Note that this model can trivially incorporate \emph{biases} by appending a dimension to $U_{i}$ that is constant with value $1$. All results in this paper can incorporate this extension as well but we use the above formulation for notational simplicity. We will now introduce a prior distribution on the weight matrices that we will analyze.

\subsection{Independent Prior}
\label{subsec:independent-prior}
The first prior distribution we consider involves weights that are independent from one another. Formally:
$$\forall k \in \{1, \ldots, K\},\quad A^{(k)}_{i,j} \perp A^{(k)}_{l,m} \text{ for } (i,j) \neq (l,m).$$
Furthermore, for normalization purposes, we impose the following further constraints:
$$\forall k \in \{1, \ldots, K\},\quad \E\left[A_{i,j}^{(k)}\right] = 0;\quad \var\left[A_{i,j}^{(k)}\right] =
\left\{
    \begin{array}{lr}
        \frac{1}{d}, & \text{if } k = 1\\
        \frac{1}{N}, & \text{if } k \neq 1
    \end{array}
\right\}.$$
We will refer to this prior as the \emph{independent prior} since each weight in the neural network is independent.

\subsection{Dirichlet Prior}
\label{subsec:sparse-nonparametric-prior}
\begin{figure}[!ht]
    \centering
    \includegraphics[scale=0.35]{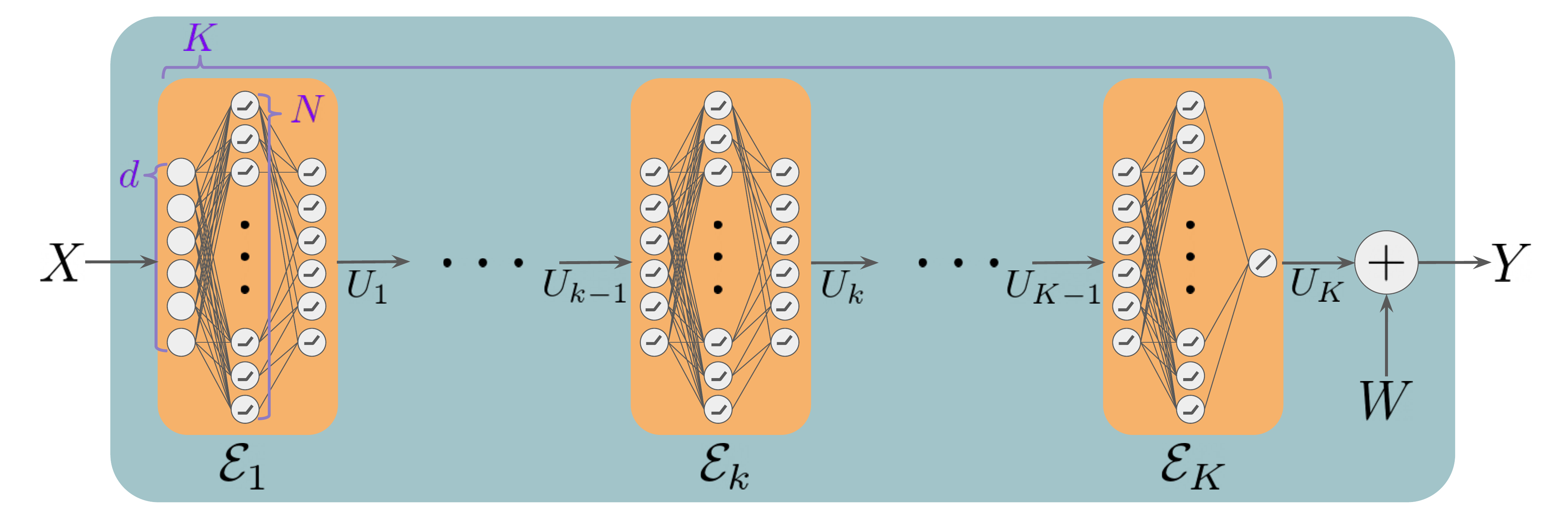}
    \caption{Dirichlet weights neural network environment. We let $d$ denote the input dimension, $K$ denote the depth, and $N$ denote the (potentially infinite) width. Note that despite the infinite width, structure is assumed via a Dirichlet prior which induces a sparsity-like effect.}
    \label{fig:proto_env2}
\end{figure}
The other distribution we will consider is the following.
Each layer $k$ consists of two matrices: $A^{(k)}\in\Re^{M\times d}, B^{(k)}\in\Re^{d\times M}$.
For all $k \in \{1, \ldots, K\}$, 
$$
    A^{(k)} =
    \begin{bmatrix}
        - & A_1^{(k)\top} & -\\
        \vdots & \vdots & \vdots\\
        - & A_M^{(k)\top} & -\\
    \end{bmatrix},
$$
where for all $i$, $A_i^{(k)}\sim\normal(0, I_d/d)$. Meanwhile, for all $k \in \{1, \ldots, K-1\}$
$$
    B^{(k)} =
    \begin{bmatrix}
        - & \sqrt{M}B_1^{(k)\top} & -\\
        \vdots & \vdots & \vdots\\
        - & \sqrt{M}B_d^{(k)\top} & -\\
    \end{bmatrix},
$$
where for all $i$, $\bar{B}_i^{(k)} \sim {\rm Dirichlet}(M/N, \ldots, M/N)$ and
$$B_{i}^{(k)} = 
    \begin{cases}
        \bar{B}_{i}^{(k)} & \text{w.p. } 0.5\\
        -\bar{B}_{i}^{(k)} & \text{w.p. } 0.5\\
    \end{cases}.
$$
Meanwhile, we let $B^{(K)} \in \Re^{N}$ have output dimension $1$. $B^{(K)}$ is also distributed Dirichlet with the same parameters. For all $k \in \{1, \ldots, K-1\}$, conditioned on $A^{(k)}$ and $B^{(k)}$, we have that
$$U_{k} = \relu(B^{(k)}\relu(A^{(k)}U_{k-1})).$$
Finally, for the output layer, we have that
$$U_{K} = B^{(K)\top}\relu(A^{(K)}U_{K-1}).$$

For this prior, we will assume that $M \ll N$. We will refer to this prior as the \emph{dirichlet prior}. This data generating process is effectively nonparametric, as $N$ can be taken to be arbitrarily large. We will show that despite this, the complexity of the environment is bounded by the scale parameter $M$ and entirely independent of the width of the network $(N)$. In the following section, we will formalize the above by deriving single-layer rate-distortion bounds that are \emph{independent} of the width $N$.

\subsection{Width-Independence in Rate Distortion Bounds}
\label{sec:single_layer_rd}

In this section, we demonstrate that the prior on the weights of the network can dramatically impact the rate-distortion function of a single layer. Notably, we can derive rate-distortion bounds that are dictated mainly by matrix norms and are width independent. This is an important characteristic that researchers have been exploring i.e moving beyond parameter-count based bounds to better explain the empirical behavior of large-scale neural networks which may be very wide. Proofs may be found in Appendix \ref{apdx:single_layer_rd}

We begin by presenting a simple parameter-count based rate-distortion bound for a single-layer relu neural network.

\begin{restatable}{theorem}{pLayerReLUNN}{\bf (independent single-layer rate-distortion bound)}
    \label{th:relu_singlelayer_rd}
    For all $d, N \in \mathbb{Z}_{++}$ and $\sigma^2, \epsilon \geq 0$, if $\environment$ is a single-layer neural network with the independent prior, then
    $$\H_\epsilon(\environment) \leq \frac{dN}{2}\ln\left(\frac{N}{2\sigma^2\epsilon}\right).$$
\end{restatable}

With the \emph{independent prior}, we cannot escape the linear width $N$ dependence in the rate-distortion function. However, when we consider different priors with dependence between weights, for example the \emph{dirichlet prior}, we can provide a rate-distortion bound that is independent of $N$.

\begin{theorem}{\bf (dirichlet single-layer rate-distortion bound)}
    \label{th:relu_nonparam_rd}
    For all $d, M, N\in\mathbb{Z}_{++}$ and $\sigma^2, \epsilon \geq 0$, if $\environment$ is a single-layer neural network environment with the dirichlet prior, then
    $$\H_\epsilon(\environment) \leq d^2M\ln^2\left(\frac{3d}{\sigma^2\epsilon}\right).$$
\end{theorem}

The rate-distortion bound is able to capture the fact that while $N$ could be infinite, the sparsity-like effect induced by the dirichlet prior fundamentally limits the complexity of learning. The bound instead depends only linearly on $d^2$, $M$ and logarithmically in the tolerance $\epsilon$.

\subsection{Avoiding Width Dependence with only Linear Depth Dependence}
While the flavor of results in section \ref{sec:single_layer_rd} have been heavily studied, a shortcoming of these results is that when adapting width-independent bounds to deep neural networks, the eventual sample complexity becomes \emph{exponential} in the depth. Meanwhile classical VC-dimension parameter count results provide sample complexity bounds that are $\mathcal{O}(K^2N)$ where $K$ and $N$ are the depth and width respectively.  While the depth dependence is much better than exponential, the width dependence is considered problematic by the research community.

A sample complexity bound that simultaneously provides good \emph{depth} and \emph{width} dependence has been illusive. While there are results such as those of \cite{NEURIPS2019_0e795480} which provide ``polynomial'' dependence on depth and input dimension, further dependence is hidden in so-called data-dependent quantities that are difficult to analyze and understand. Furthermore even high-order polynomial dependence on depth and width are prohibitive in the regime of modern deep learning.

\subsection{Decomposing the Error of Multilayer Environments}
\label{sec:error_decomp}

In this section, we provide sample complexity bounds that exhibit both favorable depth and width dependence. We are able to achieve this by leveraging an average-case framework and information-theoretic tools.
\label{subsec:single_multi_env}
\begin{figure}[!ht]
    \centering
    \includegraphics[scale=0.35]{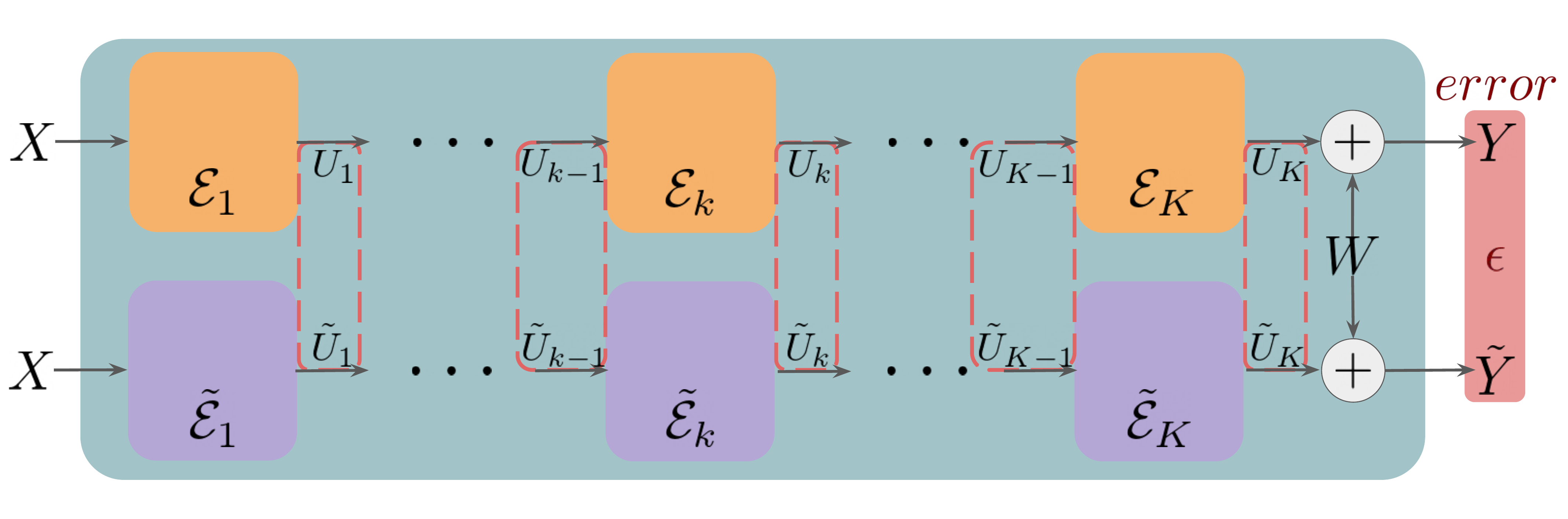}
    \caption{The error incurred by a multilayer proxy $\proxy_{K:1}$ measures the difference between true output $Y$ and the prediction $\tilde{Y}$ (shown in red box). This difference is the result of error that builds up through layers of the environment (denoted by red dotted outline).}
    \label{fig:full_error}
\end{figure}

For \emph{multilayer} environments, it becomes apparent that error is more cumbersome to reason about. Figure \ref{fig:full_error} depicts the error incurred by using a multilayer proxy $\environment_{K:1}$ to approximate multilayer environment $\environment_{K:1}$. Evidently, it seems tricky to reason about the error propagation through the layers of the environment. Many existing lines of analysis struggle on this front and result in sample complexity bounds that are exponential in the depth of the network \cite{bartlett2017spectrally}, \cite{pmlr-v75-golowich18a}. The techniques in these papers consider a worst-case reasoning under which an $\epsilon$ error between the first outputs $U_1$ and $\tilde{U}_1$ may blow up to a $\lambda^K\epsilon$ error when passed through remaining layers of the network (where $\lambda$ is a spectral radius).

\begin{figure}[!ht]
    \centering
    \includegraphics[scale=0.35]{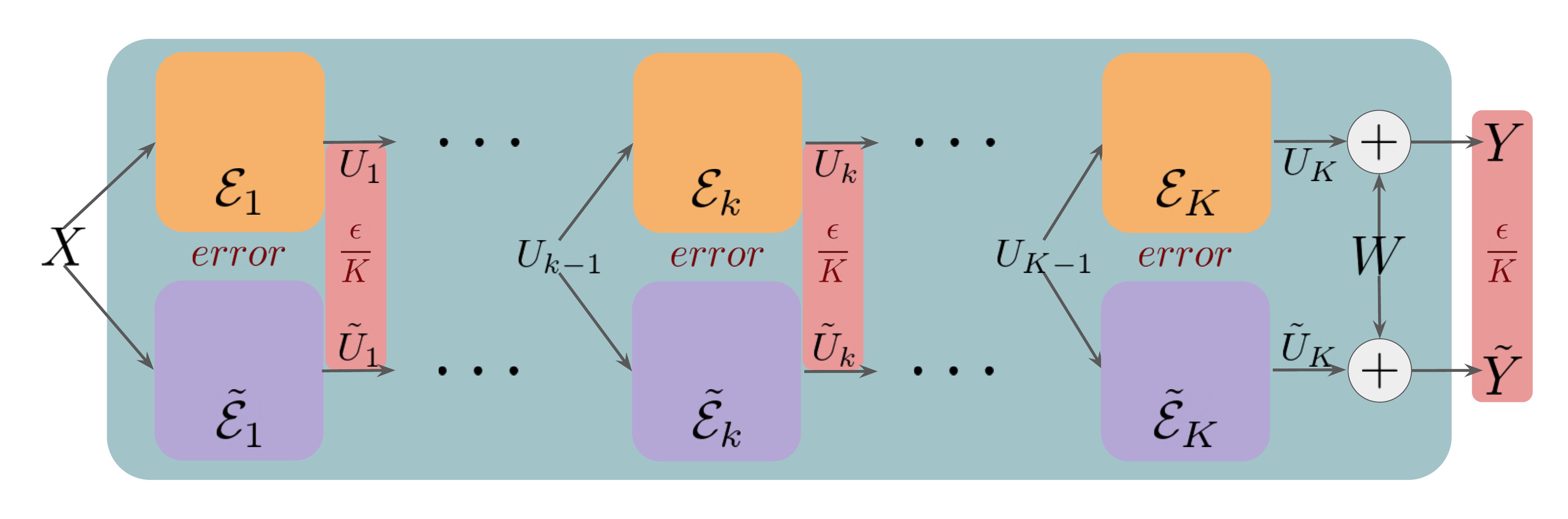}
    \caption{A much easier system to analyze is one in which we measure the incremental error at each stage of the multilayer environment. If each layer incurs an error of $\frac{\epsilon}{K}$, the total error incurred will be $\epsilon$. Note that the inputs to each layer are correctly specified in this system. We show that in many cases, analyzing the error of this system upper bounds the error of the system in Figure \ref{fig:full_error}.}
    \label{fig:incremental_error}
\end{figure}

It would be much simpler to instead independently analyze the incremental error incurred at each stage of the network. Figure \ref{fig:incremental_error} depicts this. We consider that at each layer, we know the true input $U_{k-1}$ and simply measure the immediate error incurred at the output $U_k$ as opposed to the error incurred at the final output of the network $Y$.

Mathematically, the error incurred by the full system from using proxy $\proxy_{K:1}$ can be expressed as $\I(Y;\environment_{K:1}|\proxy_{K:1}, X)$. By the chain rule, this error decomposes into:
\begin{equation}
    \label{eq:full_error}
    \I(Y;\environment_{K:1}|\proxy_{K:1}, X) = \sum_{k=1}^{K}\I(Y;\environment_{k}|\environment_{K:k+1}, \proxy_{k:1}, X).
\end{equation}
Therefore, the error incurred from layer $k$ can be expressed as $\I(Y;\environment_{k}|\environment_{K:k+1}, \proxy_{k:1}, X)$. This is cumbersome as we are not given the true input $U_{k-1}$ but rather an approximation from input $X$ and $\proxy_{k-1:1}$. Furthermore, we are measuring the error in the \emph{final} output $Y$ as opposed to the \emph{immediate} output $U_k$. It would be much more simple to analyze something like the following:
\begin{equation}
    \label{eq:incremental_error}
    \sum_{k=1}^{K}\I(U_{k}+W;\environment_k|\proxy_k, U_{k-1}),
\end{equation}
where $W$ is independent $0$-mean gaussian noise with variance $\sigma^2$ in each dimension. This sum is much easier to work with because the $k$th term only depends on $\environment_k, \proxy_k, U_{k-1},$ and $U_k$. There is no inter-layer dependence.

\subsubsection{Sufficient Conditions to Avoid Inter-Layer Dependence}
\label{sec:suff_conds}

The key insight is that in our deep neural network environment, something \emph{akin} to the following will hold:
\begin{equation}
    \label{eq:multi_to_single}
    \I(Y;\environment_{k}|\environment_{K:k+1}, \proxy_{k:1}, X) \leq \I(U_{k}+W;\environment_k|\proxy_k, U_{k-1}).
\end{equation}
As a result, the cumbersome sum \ref{eq:full_error} will be \emph{upper bounded} by the amenable sum \ref{eq:incremental_error}.

The condition in inequality \ref{eq:multi_to_single} involves two parts:
\begin{enumerate}
    \item $\I(Y;\environment_k|\environment_{K:k+1}, \proxy_{k:1}, X) \leq \I(Y;\environment_k|\environment_{K:k+1}, \proxy_k, U_{k-1})$
    \begin{itemize}
        \item Conditioning on the true input $U_{k-1}$ provides more information about $\environment_k$ than conditioning on an approximation $(\proxy_{k-1:1}, X)$.
    \end{itemize}
    \item $\I(Y;\environment_k|\environment_{K:k+1}, \proxy_k, U_{k-1})\leq \I(U_k+W;\environment_k|\proxy_k, U_{k-1})$
    \begin{itemize}
        \item The immediate output $U_k + W$ provides more information about $\environment_k$ than the final output $Y$ does.
    \end{itemize}
\end{enumerate}

$1)$ holds for all proxies of the form $\proxy = (\proxy_1, \hdots, \proxy_K)$ where $\proxy_i\perp\proxy_j$ for $i\neq j$. We prove this result explicitly in Lemma \ref{le:true_input_inequality}. It is rather intuitive that the pristine data pair $(U_{k-1}, Y)$ would provide more information about $\environment_k$ than $(U_0, \proxy_{k-1:1}, Y)$ would as some information ought to be lost through the imperfect reconstruction of $U_{k-1}$ from $U_0$ and $\proxy_{k-1:1}$.

$2)$ will not hold exactly for our environments. However, a slightly modified result of similar spirit will be shown. $\I(Y;\environment_{k}|\environment_{K:k+1}, U_{k-1})$ will be $\leq$ an upper bound of $\I(U_k+W, \environment_{k}|\environment_{k}, U_{k-1})$ which was already used to upper bound the distortion function in the single-layer results. We leave the details to Lemma \ref{le:dpe}.

The impact of inequality \ref{eq:multi_to_single} on the eventual sample-complexity bounds can be seen through the following result:

\begin{lemma}
    If $\environment_{K:1}$ is such that for all $k \in [K]$ and multilayer-proxies $\proxy = (\proxy_1, \hdots, \proxy_K)$,
    $$\I(Y;\environment_k|\environment_{K:k+1}, \proxy_k, U_{k-1})\leq \I(U_k+W;\environment_k|\proxy_k, U_{k-1}),$$
    then for all $\epsilon \geq 0$,
    $$\H_{\epsilon}(\environment_{K:1}) \leq \sum_{k=1}^{K}\H_{\frac{\epsilon}{K}}(\environment_k).$$
\end{lemma}

This result suggests that the rate-distortion (and hence sample complexity) of the full multilayer system is \emph{linear} in the depth $K$ and the rate-distortion function of the \emph{single layer} environment at distortion level $\frac{\epsilon}{K}$. Since in section \ref{sec:single_layer_rd}, we established that the single-layer rate-distortion function had favorable width dependence and at most logarithmic dependence on $\frac{1}{\epsilon}$, we are able to (and eventually will) establish sample complexity bounds that simultaneously exhibit negligible width dependence and at most \emph{quadratic} depth dependence.

\subsection{Establishing the Conditions for Favorable Depth Dependence}
In this section, we formally prove the properties from section \ref{sec:suff_conds} in the context of our prototypical neural network environment.
We begin by establishing property $1)$.

\begin{lemma}
{\bf (more is learned with the true input)}
\label{le:true_input_inequality}
Let $\proxy_{K:1}$ be a multilayer proxy.  Then, for all $k \in \{1, \hdots, K\}$,
$$\I(Y; \environment_k|\environment_{K:k+1}, \proxy_{k: 1}, X) \leq \I(Y;\environment_k|\environment_{K:k+1}, \proxy_k, U_{k-1}).$$
\end{lemma}
\begin{proof}
    \begin{align*}
        & \I(\environment_k; Y|\environment_{K:k+1}, \proxy_{k:1}, X)\\
        & = \I(\environment_k; \proxy_{k-1:1}, X, Y|\environment_{K:k+1},\proxy_k) - \I(\environment_k; \proxy_{k-1:1}, X|\environment_{K:k+1},\proxy_k)\\
        & = \I(\environment_k; \proxy_{j-1:1}, X, Y|\environment_{K:k+1},\proxy_k)\\
        & = \I(\environment_k;Y|\environment_{K:k+1}, \proxy_k) + \I(\environment_k; \proxy_{k-1:1}, X|\environment_{K:k+1}, \proxy_k, Y)\\
        & \overset{(a)}{\leq} \I(\environment_k;Y|\environment_{K:k+1}, \proxy_k) + \I(\environment_k; \environment_{k-1:1}, X|\environment_{K:k+1}, \proxy_k, Y)\\
        & = \I(\environment_k;\environment_{k-1, 1}, X, Y|\environment_{K:k+1}, \proxy_k)\\
        & \overset{(b)}{=} \I(\environment_k;Y|\environment_{K:k+1}, \proxy_k, \environment_{k-1:1}, X)\\
        & \overset{(c)}{=} \I(\environment_k;Y|\environment_{K:k+1}, \proxy_k, U_{k-1}),
    \end{align*}
    where $(a)$ follows from the fact that $\environment_k \perp \proxy_{k-1:1}|(X, Y, \environment_{k-1:1})$ and the data processing inequality, $(b)$ follows from the fact that $\I(\environment_k;\environment_{k-1:1}, X|\environment_{K:k+1}, \proxy_k) = 0$, and $(c)$ follows from the fact that $Y\perp (\environment_{k-1:1}, X)|U_{k-1}$.
\end{proof}

Lemma \ref{le:true_input_inequality} states that we learn more information about $\environment_k$ when we are given the true input $U_{k-1}$ than when we are given $(X, \proxy_{k-1:1})$ and have to infer $U_{k-1}$. This is intuitive as we should be able to recover more about $\environment_k$ when we observe its input exactly.

We now show a suitable alternative to property $2)$. Intuitively, the immediate output $U_k + W$ will provide more information about $\environment_k$ than the final output $Y$ so long as in expectation, the layers $f_{k+1}, \hdots, f_K$ don't amplify the scale of the output. If they were to amplify the scale, then the signal to noise ratio of $Y$ would look larger than that of $U_{k} + W$, leading to potentially more information. To express this idea mathematically, our results will rely on the following quantity for each layer $k \in [K]$:
$$L^{(k)} = \sup_{x,y}\E\left[\frac{\|f^{(k)}(x)-f^{(k)}(y)\|^2_2}{\|x-y\|^2_2}\Big|x=x,y=y\right].$$
$L^{(k)}$ is the expected squared lipschitz constant (averaged over the randomness in $f^{(k)}$) of the function at layer $k$. So long as the product $\prod_{i=k}^{K}L^{(i)} \leq 1$ for each $k \in [K]$, we will effectively have property $2)$.

The following result expresses $L$ concretely for a single layer of the prototypical neural network we study in this work.

\begin{lemma}{\bf (relu neural network layer stability)}
    \label{le:relu_stable}
    For all $N, M\in\mathbb{Z}_{++}$, if $A\in\Re^{M\times N}$ is a random matrix and $f(x) = \relu(Ax)$ for $x \in \Re^N$, then
    $$L = \sup_{x,y\in\Re^{N}}\E\left[\frac{\|f(x)-f(y)\|^2_2}{\|x-y\|^2_2}\Big|x=x, y=y\right] = \left\|\E\left[A^\top A\right]\right\|_\sigma,$$
    where $\|\cdot\|_\sigma$ denotes the operator norm.
\end{lemma}
\begin{proof}
    \begin{align*}
        L
        & = \sup_{x,y\in\Re^{N}}\E\left[\frac{\|f(x)-f(y)\|^2_2}{\|x-y\|^2_2}\Big|x=x, y=y\right]\\
        & = \sup_{x,y\in\Re^{N}}\E\left[\frac{\|\relu(Ax)-\relu(Ay)\|^2_2}{\|x-y\|^2_2}\Big|x=x, y=y\right]\\
        & \overset{(a)}{\leq} \sup_{x,y\in\Re^{N}}\E\left[\frac{\|A(x-y)\|^2_2}{\|x-y\|^2_2}\Big|x=x, y=y\right]\\
        & = \sup_{x,y\in\Re^{N}}\E\left[\frac{(x-y)^\top A^{\top}A(x-y)}{\|x-y\|^2_2}\Big|x=x, y=y\right]\\
        & = \sup_{x,y\in\Re^{N}} \frac{(x-y)^\top \E\left[A^{\top}A\right](x-y)}{\|x-y\|^2_2}\\
        & = \left\|\E\left[A^{\top}A\right]\right\|_\sigma,
    \end{align*}
    where $(a)$ follows from the fact that for all $x$ and $y$, $(\relu(x) - \relu(y))^2\leq (x-y)^2$.
\end{proof}

Evidently $L$ will depend on the data generating process's weight distribution. For independent prior, we have the following upper bound for $L$:

\begin{lemma}{\bf (independent stability)}
    \label{le:indep_stable}
    For all $d_{in}, d_{out}\in\mathbb{Z}_{++}$, if random matrix $A\in\Re^{d_{out}\times d_{in}}$ consists of independent elements with mean $0$ and variance $\frac{1}{d_{in}}$ and $f(X) = \relu(AX)$ for $X \in \Re^{d_{in}}$, then
    $$ L = \sup_{x,y\in\Re^{d_{in}}}\E\left[\frac{\|f(x)-f(y)\|^2_2}{\|x-y\|^2_2}\Big|x=x, y=y\right] \leq \frac{d_{out}}{d_{in}}.$$
\end{lemma}
\begin{proof}
    \begin{align*}
        L
        & \overset{(a)}{=} \left\|\E\left[A^{\top}A\right]\right\|_\sigma\\
        & = \left\|\frac{d_{out}}{d_{in}}I_{d_{in}}\right\|_\sigma\\
        & = \frac{d_{out}}{d_{in}},
    \end{align*}
    where $(a)$ follows from Lemma \ref{le:relu_stable}.
\end{proof}

Meanwhile, for the dirichlet prior, we have the following result:

\begin{lemma}{\bf (dirichlet stability)}
    \label{le:nonparam_stable}
    For all $N, M, d_{in}, d_{out}\in\mathbb{Z}_{++}$, if random matrices $(A\in\Re^{N\times d_{in}}, B\in\Re^{d_{out}, N})$ are distributed according to the dirichlet prior with $M \leq \sqrt{N}$ and $f(X) = \relu(B\ \relu(AX))$, then
    $$ L = \sup_{x,y\in\Re^{N}}\E\left[\frac{\|f(x)-f(y)\|^2_2}{\|x-y\|^2_2}\Big|x=x, y=y\right] \leq \frac{d_{out}}{d_{in}}.$$
\end{lemma}
\begin{proof}
    \begin{align*}
        L
        & \overset{(a)}{\leq} \left\|\E\left[(AB)^{\top}(AB)\right]\right\|_\sigma\\
        & \overset{}{=} \left\|\E\left[A^{\top}\left(\frac{d_{out}M(N+M)}{(M+1)N^2}I_{M}\right)A\right]\right\|_\sigma\\
        & \overset{}{=} \left(\frac{d_{out}M(N+M)}{(M+1)N^2}I_{M}\right)\left\|\E\left[A^\top A\right]\right\|_\sigma\\
        & \overset{}{=} \left(\frac{d_{out}M(N+M)}{(M+1)N^2}I_{M}\right)\left\|\frac{N}{d_{in}}I_{d_{in}}\right\|_\sigma\\
        & = \frac{d_{out}}{d_{in}}\frac{M}{M+1}\frac{N+M}{N}\\
        & \overset{(b)}{\leq} \frac{d_{out}}{d_{in}},
    \end{align*}
    where $(a)$ follows from the fact that for all $x$ and $y$, $(\relu(x) - \relu(y))^2\leq (x-y)^2$ and Lemma \ref{le:relu_stable}, and $(b)$ follows from the fact that $M \leq \sqrt{N}$.
\end{proof}

Since we assume that $M \ll N$, the above condition will hold for the data generating processes we are interested in. With these results in place, we now present a general bound for the distortion $\I(Y;\environment_k|\proxy_k, \environment_{K:k+1}, U_{k-1})$ in a multilayer environment.

\begin{lemma}{\bf (multilayer distortion bound)}
    \label{le:dpe}
    For all $k, K\in\mathbb{Z}_{++}$, intermediate layer dimensions $(N_0, N_1, \hdots, N_K)\in\mathbb{Z}_{++}^{K+1}$, and $\sigma^2 \geq 0$, if $N_K = 1$, and multilayer environment $\environment_{K:1}$ consists of single-layer environments $\environment_k$ that are each identified by a random function $f^{(k)}: \Re^{N_{k-1}}\mapsto\Re^{N_k}$, then for any proxy $\proxy_k$,
    $$\I(Y;\environment_k|\proxy_k, \environment_{K:k+1}, U_{k-1})\leq \frac{1}{2}\ln\left(1 + \frac{\left(\prod_{i=k+1}^{K}L^{(i)}\right)\cdot\E\left[\|U_k - \E[U_k| \proxy_k, U_{k-1}]\|^2_2\right]}{\sigma^2}\right),$$
    where
    $$L^{(k)} = \sup_{x,y\in\Re^{N_{k-1}}}\E\left[\frac{\|f^{(k)}(x)-f^{(k)}(y)\|^2_2}{\|x-y\|^2_2}\Big|x=x, y=y\right].$$
\end{lemma}
\begin{proof}
    \begin{align*}
        & \I(Y;\environment_k|\proxy_k, \environment_{K:k+1}, U_{k-1})\\
        & = \diffentropy(Y|\proxy_k, \environment_{K:k+1}, U_{k-1}) - \diffentropy(Y|\environment_K:k,U_{k-1})\\
        & = \diffentropy(Y|\proxy_k, \environment_{K:k+1}, U_{k-1}) - \diffentropy(W)\\
        & = \diffentropy\left(Y-(f_K\circ\hdots\circ f_{k+1})(\E[U_k|\proxy_k, U_{k-1}])|\proxy_k, \environment_{K:k+1}, U_{k-1}\right) - \diffentropy(W)\\
        & \leq \diffentropy\left(Y-(f_K\circ\hdots\circ f_{k+1})(\E[U_k|\proxy_k, U_{k-1}])\right) - \diffentropy(W)\\
        & \leq \frac{1}{2}\ln\left(1 + \frac{\var\left[U_K - (f_K\circ\hdots\circ f_{k+1})(\E[U_k|\proxy_k, U_{k-1}])\right]}{\sigma^2}\right)\\
        & \leq \frac{1}{2}\ln\left(1 + \frac{\E\left[(U_K - (f_K\circ\hdots\circ f_{k+1})(\E[U_k|\proxy_k, U_{k-1}]))^2\right]}{\sigma^2}\right)\\
        & \leq \frac{1}{2}\ln\left(1 + \frac{\E\left[\prod_{i=k+1}^{K}L^{(i)}\left\|U_k - \E[U_k|\proxy_k, U_{k-1}]\right\|^2\right]}{\sigma^2}\right)\\
    \end{align*}
\end{proof}

While the term $\prod_{i=k+1}^{K}L^{(i)}$ at the surface may look problematic, a quick inspection of the result with Lemmas \ref{le:indep_stable} and $\ref{le:nonparam_stable}$ show that this term is $\leq 1$ under very mild assumptions. For example, for both deep neural networks with independent zero-mean weights of variance $\frac{1}{N}$ and the dirichlet prior, $$\prod_{i=k+1}^{K}L^{(i)} = \frac{1}{N_{k}},$$ 
where $N_{k}$ is the input dimension to layer $k+1$.

We are able to escape exponential depth-dependence by adopting an average-case framework. For many reasonable data generating processes (such as the two described above),
$$\prod_{i=0}^{K} L^{(i)} = \mathcal{O}(1).$$
This is a \emph{much} weaker condition than $1$-lipschitzness of $f^{(k)}$ for all $k\in \{1, \ldots, K\}$, which is where many worst-case bounds falter. For example, the bounds in \cite{bartlett2017spectrally} involve the product $\prod_{k=1}^{K}\|A^{(k)}\|_\sigma$. If we were to assume that $A^{(k)}$ consisted of iid gaussian elements with variance $\frac{1}{N}$, then random matrix theory \citep{https://doi.org/10.48550/arxiv.1011.3027} would suggest that $\sqrt{\|A^{(k)}\|_\sigma} = 2 + o(1)$. The product $\prod_{k=1}^{K}\|A^{(k)}\|_\sigma$ would then be $\mathcal{O}(4^K)$, exponential in the depth.

\subsection{Sample Complexity Bounds for Multilayer Environments}
With the results established in the previous sections, we can now present the main rate-distortion and sample-complexity results for our neural network environment under the \emph{independent} and \emph{dirichlet} priors.

We first present rate-distortion and sample complexity bounds for the \emph{independent prior}.
\begin{restatable}{theorem}{reluMainResult}{\bf (independent network rate-distortion and sample complexity bounds)}
    For all $d, N, K\in\mathbb{Z}_{++}$ and $\sigma^2, \epsilon \geq 0$, if multilayer environment $\environment_{K:1}$ is the deep ReLU network with the independent prior and input $X:\Omega\mapsto\Re^d$ s.t. $\E[X_i^2] \leq 1$ for all $i \in [d]$ and output $Y\sim \normal(U_K, \sigma^2)$, then
    $$\H_{\epsilon}(\environment_{K:1}) \leq \left(\frac{KN^2+dN}{2}\right)\ln\left(\frac{K}{2\sigma^2\epsilon}\right),\quad T_\epsilon = \left(\frac{KN^2+dN}{\epsilon}\right)\ln\left(\frac{K}{\sigma^2\epsilon}\right).$$
\end{restatable}
\begin{proof}
    By Lemmas \ref{le:indep_stable} and \ref{le:dpe} and Theorem \ref{th:multilayer_rd}, for 
    $$\Delta(\environment_k, \proxy_k, U_{k-1}) = \frac{1}{2}\ln\left(1 + \frac{\E[\|U_k - \E[U_k|\proxy_k, U_{k-1}]\|^2_2]}{N\sigma^2}\right),$$
    we have that
    $$\H_\epsilon(\environment_{K:1}) \leq \sum_{k=1}^{K} \H_{\frac{\epsilon}{K}}(\environment_k, \Delta),$$
    where $\H_{\frac{d \epsilon}{K}}(\environment_k, \Delta)$ denotes the rate-distortion function for random variable $\environment_k$ under distortion function $\Delta(\environment_k,\proxy_k)$. As a result,
    \begin{align*}
        \H_\epsilon(\environment_{K:1})
        & \leq \sum_{k=1}^{K}\H_{\frac{\epsilon}{K}}(\environment_k, \Delta)\\
        & \overset{(a)}{=} (K-2)\frac{N^2}{2}\ln\left(\frac{K}{2\sigma^2\epsilon}\right) + \frac{dN}{2}\ln\left(\frac{K}{2\sigma^2\epsilon}\right) + \frac{N}{2}\ln\left(\frac{K}{2\sigma^2\epsilon}\right)\\
        & \leq \left(\frac{KN^2+dN}{2}\right)\ln\left(\frac{K}{2\sigma^2\epsilon}\right).
    \end{align*}
    where $(a)$ follows from the same proof techniques found in Theorem \ref{th:relu_singlelayer_rd}.
    The sample complexity result follows from applying Theorem \ref{th:general-sample-bound}.
\end{proof}
These sample complexity bounds show that in order to incur $\epsilon$ error, the optimal posterior predictive will need at most $\mathcal{O}\left(KN^2\ln\left(\frac{NK}{\epsilon}\right)\right)$ samples on average. This improves upon the prior results of \citep{bartlett1998almost, pmlr-v65-harvey17a} which prescribe an $\tilde{O}\left(K^2N^2\right)$ dependence.

Finally, we have the rate-distortion and sample complexity bounds for deep neural networks with the \emph{dirichlet prior}.
\begin{restatable}{theorem}{nonparamMainResult}{\bf (dirichlet network rate-distortion and sample complexity bounds)}
    For all $d, N, M, K\in\mathbb{Z}_{++}$, $\sigma^2, \epsilon \geq 0$, if multilayer environment $\environment_{K:1}$ is the deep ReLU network with the dirichlet prior and input $X:\Omega\mapsto\Re^d$ satisfies $\E\left[X_i^2\right] \leq 1$ for all $i\in[d]$, and output $Y\sim\normal(U_K, \sigma^2)$, then
    $$\H_{\epsilon}(\environment_{K:1}) \leq d^2MK\ln^2\left(\frac{3K}{\sigma^2\epsilon}\right),\quad T_\epsilon \leq \frac{2d^2MK}{\epsilon}\ln\left(\frac{6K}{\sigma^2\epsilon}\right).$$
\end{restatable}
\begin{proof}
    By Lemmas \ref{le:nonparam_stable} and \ref{le:dpe} and Theorem \ref{th:multilayer_rd}, for 
    $$\Delta(\environment_k, \proxy_k, U_{k-1}) = \frac{1}{2}\ln\left(1 + \frac{\E[\|U_k - \E[U_k|\proxy_k, U_{k-1}]\|^2_2]}{d\sigma^2}\right),$$
    we have that
    $$\H_\epsilon(\environment_{K:1}) \leq \sum_{k=1}^{K} \H_{\frac{\epsilon}{K}}(\environment_k, \Delta),$$
    where $\H_{\frac{d\epsilon}{K}}(\environment_k, \Delta)$ denotes the rate-distortion function for random variable $\environment_k$ under distortion function $\Delta(\environment_k,\proxy_k)$. As a result,
    \begin{align*}
        \H_\epsilon(\environment_{K:1})
        & \leq \sum_{k=1}^{K}\H_{\frac{\epsilon}{K}}(\environment_k, \Delta)\\
        & \overset{(a)}{=} \left(\sum_{k=1}^{K-1} d^2M\ln^2\left(\frac{3K}{\sigma^2\epsilon}\right)\right) + dM\ln^2\left(\frac{3K}{\sigma^2\epsilon}\right)\\
        & \leq d^2MK\ln^2\left(\frac{3K}{\sigma^2\epsilon}\right),
    \end{align*}
    where $(a)$ follows from Theorem \ref{th:relu_nonparam_rd} and the definition of $\Delta$.
    The sample complexity result follows from applying Theorem \ref{th:general-sample-bound}.
\end{proof}
We see that the sample complexity is independent of $N$ and is instead linear in the scale parameter $M$. This satisfies the favorable width dependence property i.e. regardless of how large $N$ may be, sample complexity is controlled by the finite scale parameter $M$. Furthermore, the dependence on depth $K$ is only $\mathcal{O}(K\ln K)$, so with our information-theoretic framework, we have delivered a bound on neural network sample complexity that simultaneously delivers favorable \emph{width} and \emph{depth} dependence.

\section{Empirical Analysis of the Sample Complexity of Gradient Descent}
The results derived in the previous section all upper bound the performance of an \emph{optimal} Bayesian learner. A natural question to ask is: to what degree do these results hold for a \emph{practical} agent? In this section, we empirically demonstrate that \emph{stochastic gradient descent} on neural networks nearly achieves the sample complexity rates prescribed for perfect Bayesian learners with data generated by single-layer networks with priors described in Sections \ref{subsec:independent-prior} and \ref{subsec:sparse-nonparametric-prior}.
The main results are described below, and readers are referred to Appendix~\ref{apdx:empir-perf-sgd} for further details.

\subsection{Experimental Setup}

\subsubsection{Teacher Network}
We consider a supervised learning setting where a set of $T$ i.i.d samples is generated by a single-layer neural network environment described in section~\ref{sec:deep-neural-network}.
In particular, we set $K=1$ and assume that
\[
  f(X) = B\relu(AX).
\]

We further restrict ourselves to $1$-dimensional outputs, i.e., $B\in\R^{1\times N}$, where $N$ is the width of the teacher network.
We consider two priors for $A$ and $B$, the independent prior (Appendix~\ref{apdx:indep-gauss-prior-def}) and the non-parametric prior (Appendix~\ref{apdx:nonparametric-prior-def}).
These are single-layer instances of the data generating processes we studied in section~\ref{sec:deep-neural-network}, but we defer concrete description of the prior to the appendix.

\subsubsection{Error}
In this setting, we fix the data set and assess the performance of an agent on the final error,
\(\KL(P^*_T \| P_T)\),
instead of cumulative error (\emph{regret}).
As discussed in section~\ref{sec:conn-mean-squar-err}, we assume that $P_T$ is also a Gaussian with variance $\sigma^2$, and KL-divergence simplifies to mean squared error (example~\ref{ex:guass_mse}):
\begin{equation}
  \label{eq:def-error-experiment}
  \KL(P^*_T \| P_T)
  = \frac{\E\left[\left(\hat{f}_T(X)-f(X)\right)^2 | f, H_{T} \right]}
   {2\sigma^2}
,
\end{equation}
where $\hat{f_T}$ is a neural network trained on $H_T$.
This is just the L2 error with respect to the \emph{noiseless} teacher network scaled inversely by the noise.

\subsubsection{Sample Complexity}
We adapt the definition of sample complexity in \ref{def:sample-complexity} to this setting:

for any $\epsilon>0$, the sample complexity $T_\epsilon$ of a training procedure is defined as the minimal number of samples $T$ such that after training on $T$ samples, the \emph{incremental} expected error is at most $\epsilon$:
\[
   T_\epsilon = \min \left\{
   T :
   \frac{\mathbb{E}\left[\left(\hat{f}_T(X)-f(X)\right)^2\right]}{2\sigma^2}
   \leq \epsilon
   \right\}
   .
\]

By Lemma~\ref{le:monotonic-error}, the error decreases at each time step.
Hence, this $T_{\epsilon}$ is a lower bound on the theoretical sample complexity defined with respect to cumulative error. For non-degenerate problems, we expect the two notions of sample complexity not to differ significantly.

\subsubsection{Training}
For different parameters of the teacher network and different number of samples $T$, we train single-hidden-layer neural networks with automatic width selection, and measure the final test error.

See Appendix~\ref{apdx:training} for details.
\subsection{Results}

\begin{figure}[htb]
  \centering
  \resizebox{\textwidth}{!}{%
    \includegraphics[height=3cm]{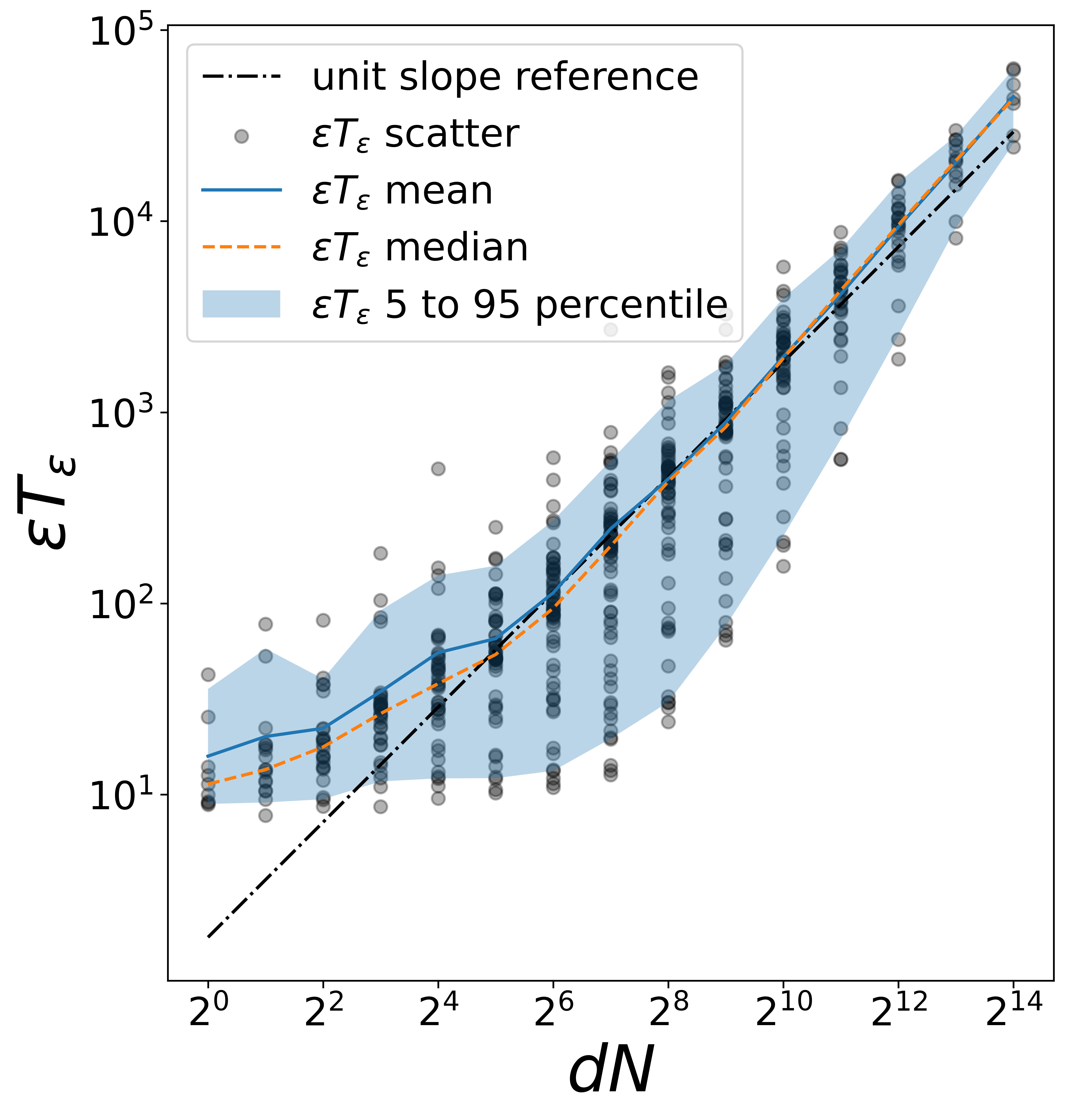}
    \includegraphics[height=3cm]{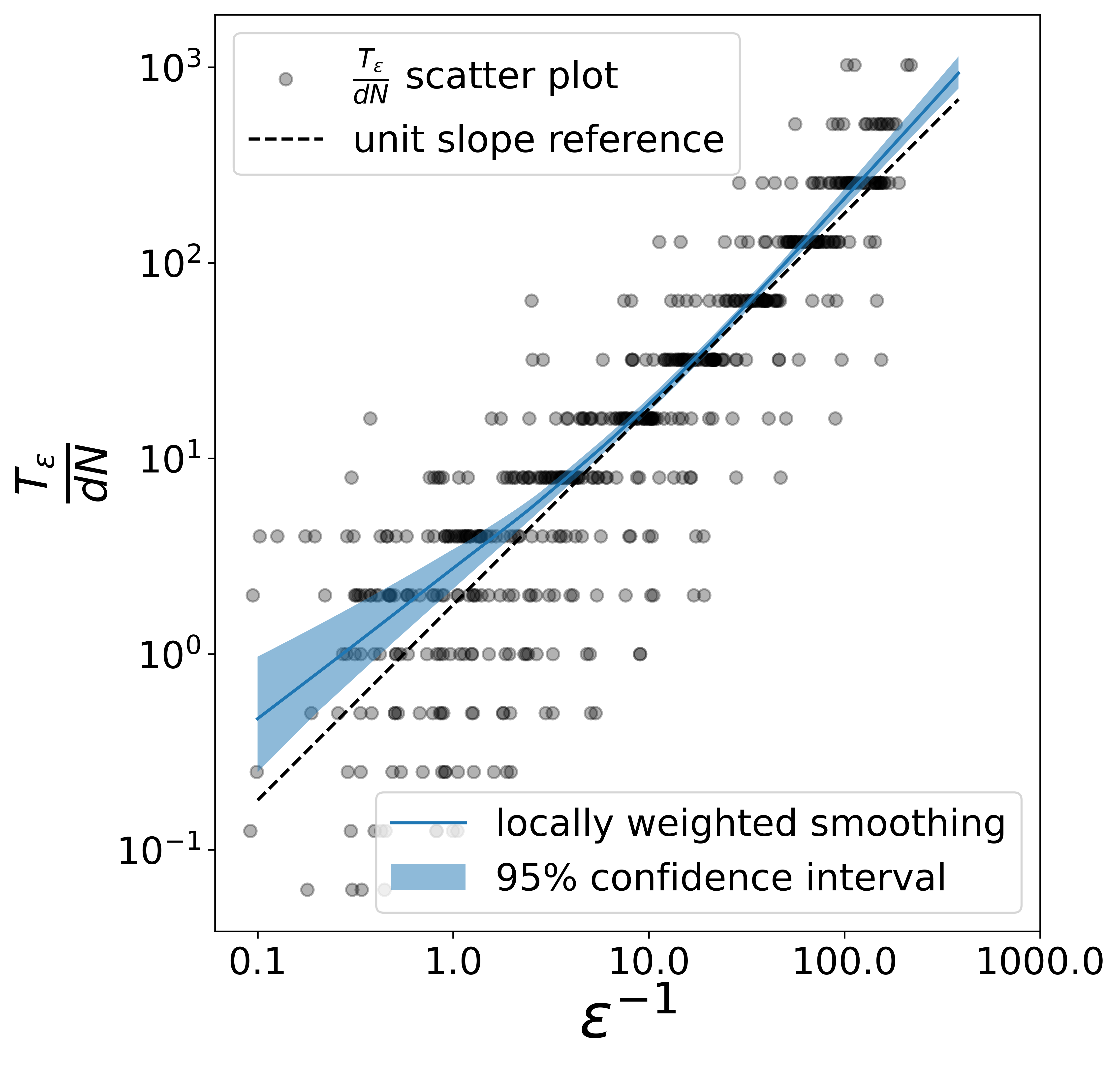}
  }
  \caption{
    For the independent Gaussian prior, the sample complexity is almost linear in $\frac{dN}{\epsilon}$ 
for a wide range of $d$, $N$, and $\epsilon$.
    $d$ is the input dimension, $\epsilon$ is the average test error, $N$ is the width of the hidden layer, and $T_\epsilon$ is the corresponding sample size.
    Here the noise $\sigma=0.1$.
    The reference lines correspond to $\epsilon T_\epsilon=1.79dN$.
    All vertical and horizontal axes are in the log scale, with equal aspect ratio.
    A unit slope reference is provided to indicate a linear relationship in the log scale.
    The confidence intervals on the right are generated by bootstrap resampling of two-thirds of the data.
  }
  \label{fig:sample-complexity-independent-main}
\end{figure}

\begin{figure}[htb]
  \centering
  \resizebox{\textwidth}{!}{%
    \includegraphics[height=3cm]{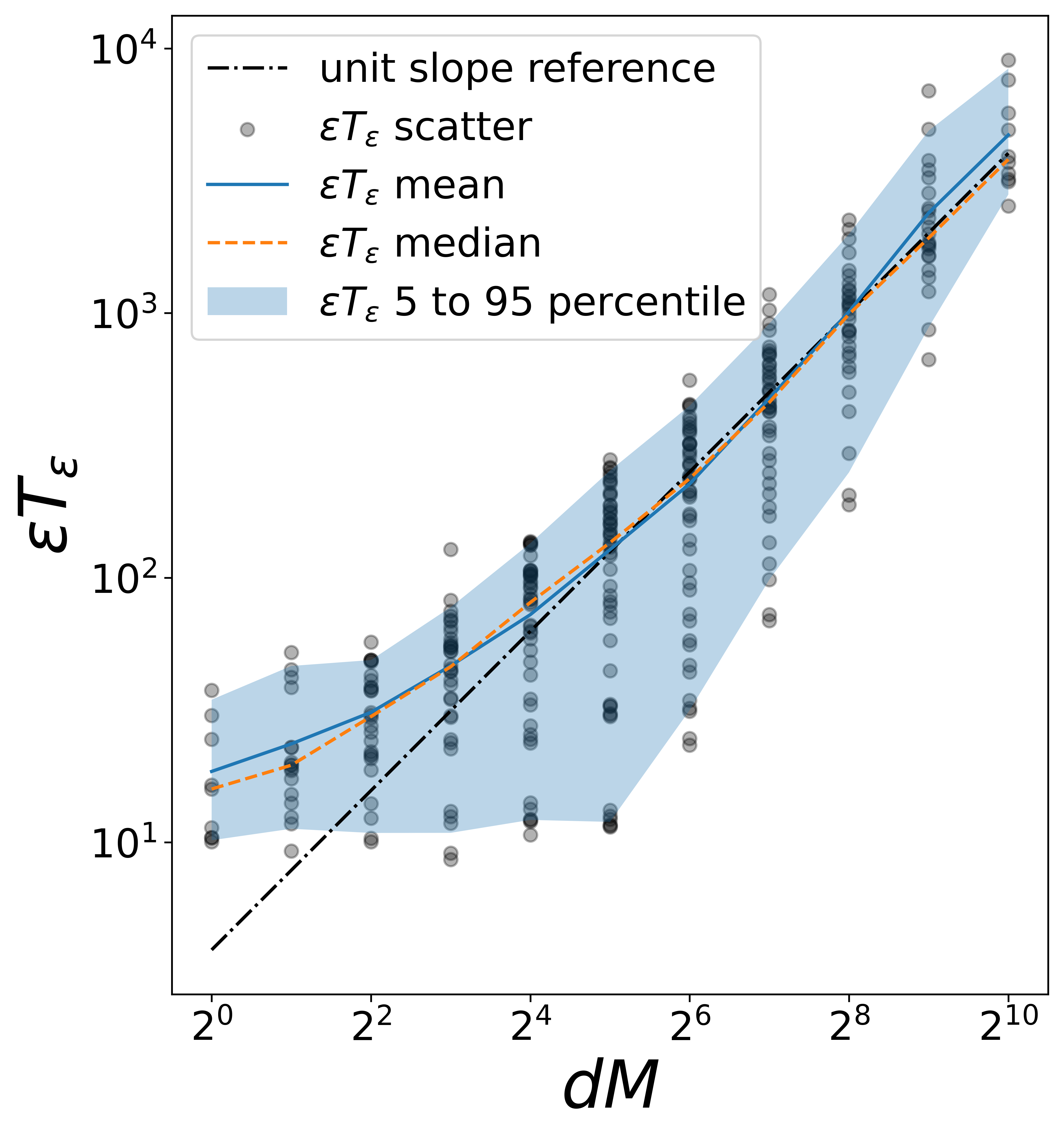}%
    \includegraphics[height=3cm]{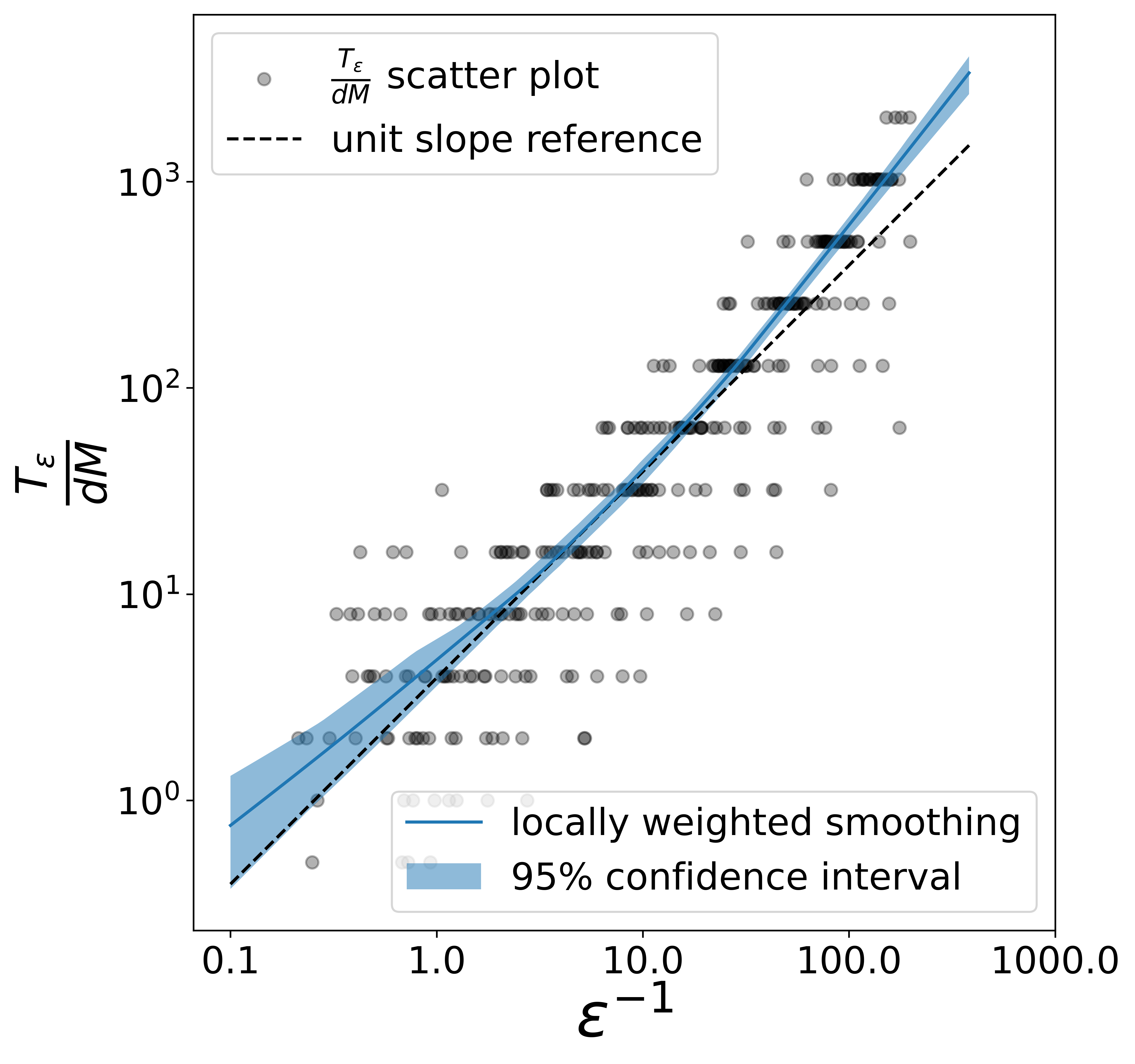}%
  }

  \caption{
    For the dirichlet prior, the sample complexity is almost linear in $\frac{dM}{\epsilon}$ 
    for a wide range of $d$, $M$, and $\epsilon$.
    $d$ is the input dimension, $\epsilon$ is the average test error, $M$ is the sparsity, and $T_\epsilon$ is the corresponding sample size.
    Here the noise $\sigma=0.1$.
    The reference lines correspond to $\epsilon T_\epsilon=3.92dM$.
    All vertical and horizontal axes are in the log scale, with equal aspect ratio.
    A unit slope reference is provided to indicate a linear relationship in the log scale.
    The confidence intervals on the right are generated by bootstrap resampling of two-thirds of the data.
  }
  \label{fig:sample-complexity-nonparametric-main}
\end{figure}

In Figures \ref{fig:sample-complexity-independent-main} and \ref{fig:sample-complexity-nonparametric-main}, we show the sample complexity of single-hidden-layer neural networks when the noise $\sigma=0.1$.
For the independent Gaussian prior (Figure~\ref{fig:sample-complexity-independent-main}),
we plot $\epsilon T_\epsilon$ against $dN$ and $\frac{T_\epsilon}{dN}$ against $\epsilon^{-1}$.
For the dirichlet prior (Figure~\ref{fig:sample-complexity-nonparametric-main}),
we plot $\epsilon T_\epsilon$ against $dM$ and $\frac{T_\epsilon}{dM}$ against $\epsilon^{-1}$.

In these plots, $d$ is the input dimension, $\epsilon$ is the average test error, $N$ is the width of the hidden layer, $M$ is the sparsity, and $T_\epsilon$ is the corresponding number of samples provided. 
Both the horizontal axis and the vertical axis are drawn in log scale, with equal aspect ratio.
In all plots, we included a scatter plot of the points, and a reference line of unit slope in the log plot, which corresponds to a linear fit of the data.
In the plots for $\epsilon T_{\epsilon}$ versus $dN$ and $dM$, we also plotted lines corresponding to the median, the mean, and the $95$ and $5$ percentiles.
In the plots for $\frac{T_{\epsilon}}{dN}$ and $\frac{T_{\epsilon}}{dM}$ versus $\epsilon^{-1}$, we use locally weighted smoothing \citep{1979-Cleveland-robust-locally-weighted-regression} to estimate the trend, and the $95\%$ confidence interval is produced by bootstrap resampling two-thirds of the data.

As we can see in the plots, for a wide range of $d$, $M$, $N$, and $\epsilon$,
for the independent Gaussian prior, $\epsilon T_{\epsilon}$ is almost proportional to $dN$;
and for the dirichlet prior, $\epsilon T_{\epsilon}$ is almost proportional to $dM$.
These matches the theoretical sample complexity implied by Theorems \ref{th:relu_singlelayer_rd} and \ref{th:nonparametric_singlelayer_rd}.
So our results indicate that SGD on neural networks (with automatic width selection) can achieve the theoretical sample complexity of ``optimal'' learners in the case of single-hidden-layer teacher network.

We note that while the dependence of $T_\epsilon$ on $dN$ and $dM$ is very close to linear, the dependence of $T_\epsilon$ on $\epsilon^{-1}$ is noticeably worse than linear for very small $\epsilon$.
We have discovered that the result is independent of noise.
Additional plots can be found in Appendix~\ref{apdx:empir-perf-sgd}.

\section{Closing Remarks}

We have introduced a novel and elegant information-theoretic framework for analyzing the sample complexity of data generating processes. We demonstrate its usefulness by proving a sample complexity bound that with simultaneously favorable width and depth dependence. These results suggests that it is indeed possible to learn efficiently from data generated by deep neural networks. Lastly, we verify that for single-layer data generating processes, the rates prescribed by our sample complexity bounds for an \emph{optimal} learner are achieved by Adam optimizer with automated width selection. This suggests that while the analysis is limited to an idealized learner, it may provide useful insight into the performance of practical algorithms.

Beyond the scope of this paper, we believe that the flexibility and simplicity of our framework will allow for the analysis of machine learning systems such as semi-supervised learning, multitask learning, bandits and reinforcement learning. We also believe that many of the nuances of empirical deep learning such as batch-normalization, pooling, and structured input distributions can be analyzed through the average-case nature of information theory and powerful tools such as the data processing inequality. Additionally, the advances in uncertainty quantification for neural networks (see, e.g., \citep{osband2021epistemic}) may provide a practical algorithm that can be analyzed under our framework.

\section*{Acknowledgements}

This research was supported by the Army Research Office (ARO) grant W911NF2010055.

\vskip 0.2in
\bibliography{biblio}

\appendix
\section{Proofs of linear regression rate-distortion lower bounds}
\label{apdx:lin_reg_lb}
We now introduce a lemma relating expected KL divergence to mean-squared error, a distortion measure that is prevalent in the literature. This relation will allow us to derive a lower bound for the rate-distortion function in the Gaussian linear regression setting.

\begin{lemma}
    \label{le:subgaussian1}
    For all $d \in \mathbb{Z}_{++}$ and $\sigma^2 \geq 0$, if $\theta:\Omega\mapsto\Re^d$ has iid components that are each $\subgauss$-subgaussian and symmetric, $X\sim\normal(0, I_d)$, and $Y \sim \normal(\theta^\top X, \sigma^2)$, then for all proxies $\proxytheta\in\proxyset$, $Y-\E[Y|\proxytheta, X]$ is $4\subgauss\|X\|^2_2 + \sigma^2$-subgaussian conditioned on $X$. 
\end{lemma}
\begin{proof}
    \begin{align*}
        \E\left[e^{\lambda(Y-\E[Y|\proxytheta, X])}\big|X\right]
        & \overset{(a)}{=} \E\left[e^{\lambda\left(Y-\E[Y|\theta,X]\right)}\big|X\right]\cdot\E\left[e^{\lambda\left(\E[Y|\theta,X] - \E[Y|\proxytheta, X]\right)}\big|X\right]\\
        & = e^{\frac{\lambda^2\sigma^2}{2}}\cdot\E\left[e^{\lambda\left((\theta-\E[\theta|\proxytheta])^\top X\right)}\big| X\right]\\
        & \overset{(b)}{\leq} e^{\frac{\lambda^2\sigma^2}{2}}\cdot\E\left[e^{-\lambda(\theta^\top X)}\cdot\E\left[e^{-\lambda(\theta^\top X)}|\proxytheta,X\right]\big|X\right]\\
        & \overset{(c)}{\leq} e^{\frac{\lambda^2\sigma^2}{2}}\E\left[\E[e^{-\lambda(\theta^\top X)}|\proxytheta, X]^2\big|X\right]\\
        &\leq e^{\frac{\lambda^2\sigma^2}{2}}\E\left[\E[e^{-2\lambda(\theta^\top X)}|\proxytheta, X]\big|X\right]\\
        & =e^{\frac{\lambda^2\sigma^2}{2}}\E\left[e^{-2\lambda(\theta^\top X)}\big|X\right]\\
        & \overset{(d)}{=}e^{\frac{\lambda^2\sigma^2}{2}}e^{2\lambda^2\subgauss\|X\|^2_2}\\
        & = e^{\frac{\lambda^2\left(\sigma^2+4\subgauss\|X\|^2_2\right)}{2}},
    \end{align*}
    where $(a)$ follows from $Y-\E[Y|\theta,X] = W$ which is independent from $\E[Y|\theta,X]-\E[Y|\proxy,X]$, $(b)$ follows from the fact that $\theta^\top X$ is symmetric conditioned on $X$ and Jensen's inequality, $(c)$ follows from $e^{\theta^\top X} = \E[e^{\theta^\top X}|\theta,\proxytheta, X]$, and $(d)$ follows from the fact that the components of $\theta$ are $\subgauss$-subgaussian.
\end{proof}

\begin{lemma}
    \label{le:subgaussian2}
    If $Y-\E[Y|\proxytheta, X]$ is $\subgauss$-subgaussian conditional on $X$, then for all $\alpha>1$, $Y-\E[Y|\proxytheta, X]$ is $\alpha\subgauss$-subgaussian conditional on $(\proxytheta, X)$.
\end{lemma}
\begin{proof}
    Assume that for some $\alpha > 1$, there exists an event $\mathcal{S}$ s.t. $\Pr(\proxytheta \in \mathcal{S}) > 0$ and $\proxytheta \in \mathcal{S}$ implies that $Y-\E[Y|\proxytheta, X]$ is not $\alpha\subgauss$-subgaussian conditioned on $(\proxytheta, X)$. We have that
    \begin{align*}
        \E\left[e^{\lambda(Y-\E[Y|\proxytheta, X])}\big|X\right]
        & \geq \Pr(\proxytheta\in\mathcal{S})\cdot \E\left[e^{\lambda(Y-\E[Y|\proxytheta, X])}|\proxytheta\in\mathcal{S}, X\right]\\
        & \overset{(a)}{>} e^{\ln\Pr(\proxytheta \in \mathcal{S})}\cdot e^{\frac{\alpha\lambda^2\subgauss}{2}}\\
        & = e^{\frac{\lambda^2\left(\alpha\subgauss + \frac{2}{\lambda^2}\ln\Pr(\proxytheta\in\mathcal{S})\right)}{2}}\\
        & = e^{\frac{\lambda^2\left(\subgauss + (\alpha-1)\subgauss + \frac{2}{\lambda^2}\ln\Pr(\proxytheta\in\mathcal{S})\right)}{2}},
    \end{align*}
    where $(a)$ holds for all $\lambda$ s.t. $|\lambda| \geq |\lambda^*|$ for some $\lambda^*$. Such $\lambda^*$ exists because of the fact that $\proxytheta\in\mathcal{S}$ implies that $Y-\E[Y|\proxytheta, X]$ is not $\alpha\subgauss$-subgaussian conditioned on $(\proxytheta, X)$. As a result, for $\lambda$ such that $\lambda^2 > \max\left\{\frac{2\ln\Pr(\proxytheta\in\mathcal{S})}{(1-\alpha)\subgauss}, \lambda_*^2\right\}$, we have that
    $$\E\left[e^{\lambda(Y-\E[Y|\proxytheta, X])}\big|X\right] > e^{\frac{\lambda^2\subgauss}{2}},$$
    which is a contradiction since $Y-\E[Y|\proxytheta, X]$ is $\subgauss$-subgaussian conditional on $X$. Therefore the assumption that there exists $\alpha > 1$ and $\proxytheta$ s.t. is not $\alpha\nu^2$-subgaussian conditional on $(X, \proxytheta)$ cannot be true. The result follows.
\end{proof}
\begin{lemma}
    \label{le:mse_kl_inequality}
    If $Y-\E[Y|\proxytheta, X]$ is $\subgauss$-subgaussian conditioned on $(\proxytheta, X)$, then
    $$\E\left[\frac{\left(\E[Y|\theta,X] - \E[Y|\proxytheta, X]\right)^2}{2\subgauss}\right] \leq \E\left[\KL(\Pr(Y\in\cdot|\theta, X)\|\Pr(Y\in\cdot|\proxytheta, X))\right].$$
\end{lemma}
\begin{proof}
We begin by stating a variational form of the KL-divergence. For all probability distributions $P$ and $Q$ over $\Re$ such that $P$ is absolutely continuous with respect to $Q$, 
$$\KL(P\|Q) = \sup_{g:\Re\rightarrow\Re} \left(\int_{y \in \Re} g(y) P(dy) - \ln \int_{y \in \Re} e^{g(y)} Q(dy)\right),$$
where the supremum is taken over measurable functions for which $\int_{y\in\Re} g(y) P(dy)$ is well-defined  and $\int_{y\in\Re} e^{g(y)} Q(dy)$ is finite.
    
Let $P = \Pr(Y\in\cdot|\theta, X_t), Q = \Pr(Y\in\cdot|\proxytheta, X),$ and $Z = Y - \E[Y|\proxytheta, X]$. Then, for arbitrary $\lambda \in \R$, applying the variational form of KL-divergence with $g(Y) = \lambda Z$ gives us
    \begin{align*}
        \KL(\Pr(Y\in\cdot|\theta, X)\|\Pr(Y\in\cdot|\proxytheta, X))
        & \overset{(a)}{=} \KL(\Pr(Y\in\cdot|\theta,\proxytheta,  X)\|\Pr(Y\in\cdot|\proxytheta, X))\\
        & \geq \lambda \E\left[Z|\theta, \proxytheta, X\right]-\ln\E\left[e^{\lambda Z}|\proxytheta, X \right]\\
        &\overset{(b)}{\geq} \lambda\left(\E[Y|\theta, X] -\E[Y|\proxytheta, X]\right) -\frac{\lambda^2\subgauss}{2},
    \end{align*}
    where $(a)$ follows from $Y\perp \proxytheta |(\theta, X)$ and $(b)$ follows from $Z$ being $\subgauss$-subgaussian conditioned on $(\proxytheta, X)$. Since the above holds for arbitrary $\lambda$, maximizing the RHS w.r.t $\lambda$ give us:
    $$\KL(\Pr(Y\in\cdot|\theta, X)\|\Pr(Y\in\cdot|\proxytheta, X)) \geq \frac{\left(\E[Y|\theta, X] - \E[Y|\proxytheta, X]\right)^2}{2\subgauss}.$$
    The result follows from taking an expectation on both sides.
\end{proof}
We now provide the proof of Lemma \ref{le:mse-mutual-info-inequality} from the main text.
\mseInfoInequality*
\begin{proof}
    $\theta^\top X$ is $\|X\|_2^2$-subgaussian conditioned on $X$ and by Lemma \ref{le:subgaussian1}, $Y - \E[Y|\proxy, X]\ |\ \|X\|_2^2$ is $4\|X\|_2^2+\sigma^2$-subgaussian. Lemma \ref{le:subgaussian2} then states that $Y - \E[Y|\proxy, X]$ is is $\alpha(4\|X\|_2^2+\sigma^2)$-subgaussian conditioned on $(\proxy, X)$ for all $\alpha > 1$. Therefore,
    \begin{align*}
        \E\left[\KL(\Pr(Y\in\cdot|\theta,X)\|\Pr(Y\in\cdot|\proxy, X))\right]
        & \overset{(a)}{\geq} \lim_{\alpha\downarrow 1} \E\left[\frac{\E\left[\left((\theta - \E[\theta|\proxy])^\top X\right)^2\right]}{2\alpha(4\|X\|_2^2 + \sigma^2)}\right]\\
        & \overset{(b)}{=} \E\left[\frac{1}{2(4\|X\|_2^2 + \sigma^2)}\right]\E\left[\|\theta - \E[\theta|\proxy]\|^2_2\right]\\
    \end{align*}
    where $(a)$ follows from Lemma \ref{le:mse_kl_inequality} and $(b)$ follows from the fact that $X\sim \normal(0, I_d)$.
\end{proof}

\section{Proofs of misspecified linear regression results.}
\label{apdx_sec:misspecification}
\begin{lemma}
    \label{le:rand_matrix_th}
    For all $t, d\in\mathbb{Z}_{++}$, If $t \geq d$, then with probability at least $1 - e^{-\frac{s^2}{2}}$,
    $$\inf_{u: \|u\|_2=1} u^\top \left(\sum_{i=0}^{t-1}X_iX_i^\top\right)\ u\ \geq \left(\sqrt{t}-\sqrt{d} - s\right)^2.$$
\end{lemma}
\begin{proof}
    The result follows directly from Corollary 5.25 of \cite{https://doi.org/10.48550/arxiv.1011.3027}.
\end{proof}

\begin{lemma}{\bf (mean-misspecified error)}
    \label{le:kl_matrix}
    For all $d, t \in \mathbb{Z}_{++}$ and $\mu \in \Re^d$, if $P_t(Y_t|H_t)$ is the postersior distribution of $Y_t$ conditioned on $H_t$ with the incorrect prior $\Pr(\theta\in\cdot)\sim \normal(\mu, I_d)$, then
    $$\E\left[\KL(\hat{P}_t\|P_t)\right] = \frac{1}{2}\E\left[\mu^\top\left(\frac{\Sigma^{-1}_{t-1}XX_t^\top\Sigma^{-1}_{t-1}}{\sigma^2 + X_t^\top \Sigma^{-1}_{t-1} X_t}\right)\mu\right],$$
    where for all $t$, $\Sigma_t = I_d + \frac{\sum_{i=0}^{t} X_iX_i^\top}{\sigma^2}$.
\end{lemma}
\begin{proof}
    For all $t \in \mathbb{Z}_{++}$, let
    $$\mathbb{Y}_{t} =
        \begin{bmatrix}
            Y_1\\
            Y_2\\
            \vdots\\
            Y_t
        \end{bmatrix};
        \quad \mathbb{X}_t = 
        \begin{bmatrix}
            X_0^\top\\
            X_1^\top\\
            \vdots\\
            X_t^\top\\
        \end{bmatrix};
        \quad \Sigma_t = I_d + \frac{\sum_{i=0}^{t} X_iX_i^\top}{\sigma^2}.
    $$
    With this notation in place, we have that
    \begin{align*}
        & \E\left[\KL(\hat{P}_t\|P_t)\right]\\
        & = \E\left[\ln\left(\frac{\hat{P}_t}{P_t}\right)\right]\\
        & = \E\left[\ln\left(\frac{d\Pr(Y_{t+1}|H_t)}{dP(Y_{t+1}|H_t)}\right)\right]\\
        & = \E\left[\ln\left(\frac{\int_{\theta\in\Re^d}d\Pr(Y_t|\theta, X_t)d\Pr(H_t|\theta)d\Pr(\theta)d\theta}{\int_{\theta\in\Re^d}d\Pr(H_t|\theta)d\Pr(\theta) d\theta} \cdot \frac{\int_{\theta\in\Re^d}d\Pr(H_t|\theta)dP(\theta) d\theta}{\int_{\theta\in\Re^d}d\Pr(Y_t|\theta, X_t)d\Pr(H_t|\theta)dP(\theta)d\theta}\right)\right]\\
        & = \E\left[\ln\left(\frac{\int_{\theta\in\Re^d}e^{-\frac{\|\mathbb{Y}_{t+1} - \mathbb{X}_{t}\theta\|^2_2}{2\sigma^2}}\cdot e^{-\frac{\theta^\top\theta}{2}}d\theta}{\int_{\theta\in\Re^d}e^{-\frac{\|\mathbb{Y}_{t+1} - \mathbb{X}_{t}\theta\|^2_2}{2\sigma^2}}\cdot e^{-\frac{(\theta-\mu)^\top(\theta-\mu)}{2}}d\theta}\cdot \frac{\int_{\theta\in\Re^d}e^{-\frac{\|\mathbb{Y}_{t} - \tilde{\mathbb{X}}_{t-1}\theta\|^2_2}{2\sigma^2}}\cdot e^{-\frac{(\theta-\mu)^\top(\theta-\mu)}{2}}d\theta}{\int_{\theta\in\Re^d}e^{-\frac{\|\mathbb{Y}_{t} - \tilde{\mathbb{X}}_{t-1}\theta\|^2_2}{2\sigma^2}}\cdot e^{-\frac{\theta^\top\theta}{2}}d\theta}\right)\right]\\
        & \overset{(a)}{=} \frac{1}{2}\E\left[\left(-2\mathbb{Y}_{t}\tilde{\mathbb{X}}_{t-1} + \mu\right)^\top\left(I_d + \frac{\tilde{\mathbb{X}}_{t-1}^\top\tilde{\mathbb{X}}_{t-1}}{\sigma^2}\right)^{-1}\mu - \left(-2\mathbb{Y}_{t+1}\mathbb{X}_{t} + \mu\right)^\top\left(I_d + \frac{\mathbb{X}_{t}^\top\mathbb{X}_{t}}{\sigma^2}\right)^{-1}\mu\right]\\
        & \overset{(b)}{=} \frac{1}{2}\E\left[\mu^\top \left(\left(I_d + \frac{\tilde{\mathbb{X}}_{t-1}^\top\tilde{\mathbb{X}}_{t-1}}{\sigma^2}\right)^{-1}- \left(I_d + \frac{\mathbb{X}_t^\top\mathbb{X}_t}{\sigma^2}\right)^{-1}\right)\mu\right] \\
        & \overset{(c)}{=} \frac{1}{2}\E\left[\mu^\top\left(\frac{\Sigma^{-1}_{t-1}XX_t^\top\Sigma^{-1}_{t-1}}{\sigma^2 + X_t^\top \Sigma^{-1}_{t-1} X_t}\right)\mu\right],\\
    \end{align*}
    where $(a)$ follows from completing the square, $(b)$ follows from the fact that $\E[\theta] = 0$, and $(c)$ follows from the Sherman-Morrison formula.
\end{proof}

\wrongMean*
\begin{proof}
    \begin{align*}
        \E\left[\KL(\hat{P}_t\|P_t)\right]
        & \overset{(a)}{=} \frac{1}{2}\E\left[\mu^\top\left(\frac{\Sigma^{-1}_{t-1}XX_t^\top\Sigma^{-1}_{t-1}}{\sigma^2 + X_t^\top \Sigma^{-1}_{t-1} X_t}\right)\mu\right]\\
        & \leq \frac{1}{2}\E\left[\left\|\frac{\Sigma^{-1}_{t-1}XX_t^\top\Sigma^{-1}_{t-1}}{\sigma^2 + X_t^\top \Sigma^{-1}_{t-1} X_t}\right\|_2\right]\|\mu\|_2^2\\
        & \leq \frac{1}{2}\E\left[\frac{\left\|\Sigma^{-\frac{1}{2}}_{t-1}\right\|_2 \left\|\Sigma^{-\frac{1}{2}}_{t-1}XX_t^\top\Sigma^{-\frac{1}{2}}_{t-1}\right\|_2\left\|\Sigma^{-\frac{1}{2}}_{t-1}\right\|_2}{\sigma^2 + X_t^\top \Sigma^{-1}_{t-1} X_t}\right]\|\mu\|_2^2\\
        & \overset{(b)}{=} \frac{1}{2}\E\left[\frac{\left\|\Sigma^{-1}_{t-1}\right\|_2}{\frac{\sigma^2}{X_t^\top \Sigma^{-1}_{t-1}X_t} + 1}\right]\|\mu\|^2_2\\
        & \leq \frac{1}{2}\E\left[\frac{{\rm Trace}\left[\Sigma^{-1}_{t-1}\right]\cdot\left\|\Sigma^{-1}_{t-1}\right\|_2}{\sigma^2}\right]\|\mu\|^2_2\\
        & \overset{(c)}{\leq} \frac{d}{2\sigma^2}\E\left[\|\Sigma^{-1}_{t-1}\|^2_2\right]\|\mu\|^2_2\\
        &  = \frac{d\|\mu\|^2_2}{2\sigma^2}\E\left[\left(\frac{1}{1 + \frac{\lambda_{{\rm min}, t-1}}{\sigma^2}}\right)^2\right]\\
        & \overset{(d)}{\leq} \frac{d\|\mu\|^2_2}{2\sigma^2}\left(\frac{1}{\left(1 + \frac{t}{4\sigma^2}\right)^2} + e^{-\frac{\left(\frac{1}{2}\sqrt{t}-\sqrt{d}\right)^2}{2}}\right)\\
        & \leq \frac{d\|\mu\|^2_2}{2\sigma^2}\left(\frac{4\sigma^4}{t^2} + e^{-\frac{\left(\frac{1}{2}\sqrt{t}-\sqrt{d}\right)^2}{2}}\right)\\
        & = d\|\mu\|_2^2\left(\frac{2}{t^2} + \frac{1}{2\sigma^2}e^{-\frac{\left(\frac{1}{2}\sqrt{t}-\sqrt{d}\right)^2}{2}}\right),
    \end{align*}
    where (a) follows from Lemma \ref{le:kl_matrix}, $(b)$ follows from the fact that $\left\|\Sigma^{-\frac{1}{2}}_{t-1}XX_t^\top\Sigma^{-\frac{1}{2}}_{t-1}\right\|_2 = X_t^\top \Sigma^{-1}_{t-1}X_t$ since the matrix is rank 1, $(c)$ follows from the fact that ${\rm Trace}[\Sigma] \leq d\|\Sigma\|_2$, and $(d)$ follows from Lemma \ref{le:rand_matrix_th} with $s = \frac{1}{2}\sqrt{t} - \sqrt{d}$.
\end{proof}

\begin{lemma}
    \label{le:missing_feature_pdf}{\bf (missing feature error)}
     For all $d, t \in \mathbb{Z}_{++}$ and $\mu \in \Re^d$, if $P_t(Y_t|H_t)$ is the postersior distribution of $Y_t$ conditioned on $H_t$ with the incorrect prior $\Pr(\theta_d\in\cdot)\sim \mathbbm{1}[\theta_d = 0]$, then
     \begin{align*}
        \ln\left(\frac{\hat{P}_t}{P_t}\right)
         & = \ln\left(\sqrt{1 + \frac{X_{t, d}^2}{\sigma^2}}\cdot\frac{\int_{\theta\in\Re^d}e^{-\frac{\|\mathbb{Y}_{t+1} - \tilde{\mathbb{X}}_{t}\theta\|^2_2}{2\sigma^2}}\cdot e^{-\frac{\theta^\top\theta}{2}}d\theta}{ \int_{\theta\in\Re^d}e^{-\frac{\|\mathbb{Y}_{t} - \tilde{\mathbb{X}}_{t-1}\theta\|^2_2}{2\sigma^2}}\cdot e^{-\frac{\theta^\top\theta}{2}}d\theta }\right)\\
        &\quad + \ln\left(
        \frac{\int_{\proxytheta\in\Re^{d-1}}e^{-\sum_{i=0}^{t-1}\frac{(\mathbb{Y}_{i+1} - \tilde{X}_{i}\proxytheta)^2}{2(\sigma^2 + X_{i,d}^2)}}\cdot e^{-\frac{\proxytheta^\top\proxytheta}{2}}d\proxytheta}{\int_{\proxytheta\in\Re^{d-1}}e^{-\sum_{i=0}^{t}\frac{(\mathbb{Y}_{i+1} - \tilde{X}_{i}\proxytheta)^2}{2(\sigma^2 + X_{i,d}^2)}}\cdot e^{-\frac{\proxytheta^\top\proxytheta}{2}}d\proxytheta}\right).\\
     \end{align*}
\end{lemma}
\begin{proof}
    \begin{align*}
        \ln\left(\frac{\hat{P}_t}{P_t}\right)
        & = \ln\left(\frac{d\Pr(Y_{t+1}|H_t)}{dP(Y_{t+1}|H_t)}\right)\\
        & = \ln\left(\frac{\int_{\theta\in\Re^d}d\Pr(Y_t|\theta, X_t)d\Pr(H_t|\theta)d\Pr(\theta)d\theta}{\int_{\theta\in\Re^d}d\Pr(H_t|\theta)d\Pr(\theta) d\theta}\right)\\
        &\quad - \ln\left(\frac{\int_{\theta\in\Re^d}d\Pr(Y_t|\theta, X_t)d\Pr(H_t|\theta)dP(\theta)d\theta}{\int_{\theta\in\Re^d}d\Pr(H_t|\theta)dP(\theta) d\theta}\right)\\
        & = \ln\left(\frac{\int_{\theta\in\Re^d}\frac{e^{-\frac{\|\mathbb{Y}_{t+1} - \tilde{\mathbb{X}}_{t}\theta\|^2_2}{2\sigma^2}}}{\left(\sqrt{2\pi \sigma^2}\right)^{t+1}}\cdot \frac{e^{-\frac{\theta^\top\theta}{2}}}{\left(\sqrt{2\pi}\right)^{d}}d\theta}{ \int_{\theta\in\Re^d}\frac{e^{-\frac{\|\mathbb{Y}_{t} - \tilde{\mathbb{X}}_{t-1}\theta\|^2_2}{2\sigma^2}}}{\left(\sqrt{2\pi\sigma^2}\right)^t}\cdot \frac{e^{-\frac{\theta^\top\theta}{2}}}{\left(\sqrt{2\pi}\right)^d}d\theta }\cdot \frac{\int_{\proxytheta\in\Re^{d-1}}\frac{e^{-\sum_{i=0}^{t-1}\frac{(\mathbb{Y}_{i+1} - \tilde{X}_{i}\proxytheta)^2}{2(\sigma^2 + X_{i,d}^2)}}}{\prod_{i=1}^{t-1}\sqrt{2\pi(\sigma^2 + X_{i,d}^2)}}\cdot \frac{e^{-\frac{\proxytheta^\top\proxytheta}{2}}}{\left(\sqrt{2\pi}\right)^{d-1}}d\proxytheta}{\int_{\proxytheta\in\Re^{d-1}}\frac{e^{-\sum_{i=0}^{t}\frac{(\mathbb{Y}_{i+1} - \tilde{X}_{i}\proxytheta)^2}{2(\sigma^2 + X_{i,d}^2)}}}{\prod_{i=1}^{t}\sqrt{2\pi(\sigma^2 + X_{i,d}^2)}}\cdot \frac{e^{-\frac{\proxytheta^\top\proxytheta}{2}}}{\left(\sqrt{2\pi}\right)^{d-1}}d\proxytheta}\right)\\
        & = \ln\left(\sqrt{1 + \frac{X_{t, d}^2}{\sigma^2}}\cdot\frac{\int_{\theta\in\Re^d}e^{-\frac{\|\mathbb{Y}_{t+1} - \tilde{\mathbb{X}}_{t}\theta\|^2_2}{2\sigma^2}}\cdot e^{-\frac{\theta^\top\theta}{2}}d\theta}{ \int_{\theta\in\Re^d}e^{-\frac{\|\mathbb{Y}_{t} - \tilde{\mathbb{X}}_{t-1}\theta\|^2_2}{2\sigma^2}}\cdot e^{-\frac{\theta^\top\theta}{2}}d\theta }\right)\\
        &\quad + \ln\left(
        \frac{\int_{\proxytheta\in\Re^{d-1}}e^{-\sum_{i=0}^{t-1}\frac{(\mathbb{Y}_{i+1} - \tilde{X}_{i}\proxytheta)^2}{2(\sigma^2 + X_{i,d}^2)}}\cdot e^{-\frac{\proxytheta^\top\proxytheta}{2}}d\proxytheta}{\int_{\proxytheta\in\Re^{d-1}}e^{-\sum_{i=0}^{t}\frac{(\mathbb{Y}_{i+1} - \tilde{X}_{i}\proxytheta)^2}{2(\sigma^2 + X_{i,d}^2)}}\cdot e^{-\frac{\proxytheta^\top\proxytheta}{2}}d\proxytheta}\right).\\
    \end{align*}
\end{proof}

\begin{lemma}
    \label{le:trace_eq}
    For all $t, d\in \mathbb{Z}_{+}$, $\theta\sim\normal(0, I_d)$, and $\mathbb{X}_{t}\in\Re^{t+1, d}$ with iid $\normal(0, 1)$ elements, if $\mathbb{Y}_{t+1} = \mathbb{X}_t \theta + \mathbb{W}_{t+1}$ where $\mathbb{W}_{t+1}\sim\normal(0, \sigma^2I_{t+1})$, then
    $$\E\left[- \frac{{\rm Tr}[\Sigma_{t}^{-1}\mathbb{X}_{t}^\top\mathbb{Y}_{t+1}\mathbb{Y}_{t+1}^\top\mathbb{X}_{t}]}{2\sigma^4}\right] = -\frac{(t+1)d}{2\sigma^2},$$
    where $\Sigma^{-1}_t = \left(I_d + \frac{\mathbb{X}_t^\top\mathbb{X}_t}{\sigma^2}\right)^{-1}$.
\end{lemma}
\begin{proof}
    \begin{align*}
        \E\left[- \frac{{\rm Tr}[\Sigma_{t}^{-1}\mathbb{X}_{t}^\top\mathbb{Y}_{t+1}\mathbb{Y}_{t+1}^\top\mathbb{X}_{t}]}{2\sigma^4}\right]
        & = \E\left[- \frac{{\rm Tr}[\Sigma_{t}^{-1}\mathbb{X}_{t}^\top\left( \mathbb{X}_{t}\mathbb{X}_{t}^\top + \mathbb{W}_{t+1}\mathbb{W}_{t+1}^\top \right)\mathbb{X}_{t}]}{2\sigma^4}\right]\\
        & = \E\left[-\frac{{\rm Tr}[\Sigma^{-1}_{t}\mathbb{X}_{t}^\top \mathbb{X}_{t}\mathbb{X}_{t}^\top\mathbb{X}_{t}]}{2\sigma^4} - \frac{{\rm Tr}[\Sigma^{-1}_{t}\mathbb{X}_{t}^\top\mathbb{X}_{t}]}{2\sigma^2}\right]\\
        & = \E\left[-\frac{{\rm Tr}\left[\Sigma_{t}^{-1}\left(I_d + \frac{\mathbb{X}^\top_{t}\mathbb{X}_{t}}{\sigma^2}\right)\mathbb{X}_{t}^\top\mathbb{X}_{t}\right]}{2\sigma^2}\right]\\
        & = -\frac{(t+1)d}{2\sigma^2}
    \end{align*}
\end{proof}

\begin{lemma}
    \label{le:trace_diag_eq}
    For all $t, d\in \mathbb{Z}_{+}$, $\theta\sim\normal(0, I_d)$, and $\mathbb{X}_{t}\in\Re^{t+1, d}$ with iid variance $1$ elements, if $\mathbb{Y}_{t+1} = \mathbb{X}_t \theta + \mathbb{W}_{t+1}$ where $\mathbb{W}_{t+1}\sim\normal(0, \sigma^2I_{t+1})$, then
    $$\E\left[- \frac{{\rm Tr}[\Sigma_{t}^{-1}\tilde{\mathbb{X}}_{t}^\top\Lambda_{t+1}\mathbb{Y}_{t+1}\mathbb{Y}_{t+1}^\top\Lambda_{t+1}\tilde{\mathbb{X}}_{t}]}{2}\right] = -\frac{(t+1)(d-1)}{2}\E\left[\frac{1}{\sigma^2 + (X_{t,d})^2}\right],$$
    where $\tilde{\Sigma}^{-1}_t = \left(I_{d-1} + \tilde{\mathbb{X}}_t^\top\Lambda_{t+1}\tilde{\mathbb{X}}_t\right)^{-1}$, $\tilde{\mathbb{X}}_t \in\Re^{t+1, d-1}$ is $\mathbb{X}_t$ with the final column omitted, and $\Lambda_t = {\rm diag}\left[\frac{1}{\sigma^2 + (X_{0,d})^2},\hdots, \frac{1}{\sigma^2 + (X_{t,d})^2} \right]$.
\end{lemma}
\begin{proof}
    \begin{align*}
        & \E\left[- \frac{{\rm Tr}[\Sigma_{t}^{-1}\tilde{\mathbb{X}}_{t}^\top\Lambda_{t+1}\mathbb{Y}_{t+1}\mathbb{Y}_{t+1}^\top\Lambda_{t+1}\tilde{\mathbb{X}}_{t}]}{2}\right]\\
        & = \E\left[- \frac{{\rm Tr}[\Sigma_{t}^{-1}\tilde{\mathbb{X}}_{t}^\top\Lambda_{t+1}\left( \mathbb{X}_{t}\mathbb{X}_{t}^\top + \mathbb{W}_{t+1}\mathbb{W}_{t+1}^\top \right)\Lambda_{t+1}\tilde{\mathbb{X}}_{t}]}{2}\right]\\
        & = \E\left[\frac{-{\rm Tr}[\Sigma^{-1}_{t}\tilde{\mathbb{X}}_{t}^\top \Lambda_{t+1}\tilde{\mathbb{X}}_{t}\tilde{\mathbb{X}}_{t}^\top\Lambda_{t+1}\tilde{\mathbb{X}}_{t}] - {\rm Tr}[\Sigma^{-1}_{t}\tilde{\mathbb{X}}_{t}^\top\Lambda_{t+1} \left(\mathbb{X}_{t,d}\mathbb{X}_{t,d}^\top + \mathbb{W}_{t+1}\mathbb{W}_{t+1}^\top\right) \Lambda_{t+1}\tilde{\mathbb{X}}_{t}]}{2}\right]\\
        & = \E\left[\frac{-{\rm Tr}[\Sigma^{-1}_{t}\tilde{\mathbb{X}}_{t}^\top \Lambda_{t+1}\tilde{\mathbb{X}}_{t}\tilde{\mathbb{X}}_{t}^\top\Lambda_{t+1}\tilde{\mathbb{X}}_{t}]-{\rm Tr}[\Sigma^{-1}_{t}\tilde{\mathbb{X}}_{t}^\top \Lambda_{t+1}\tilde{\mathbb{X}}_{t}]}{2}\right]\\
        & = \E\left[-\frac{{\rm Tr}\left[\Sigma_{t}^{-1}\left(I_{d-1} + \tilde{\mathbb{X}}^\top_{t}\Lambda_{t+1}\tilde{\mathbb{X}}_{t}\right)\tilde{\mathbb{X}}_{t}^\top\Lambda_{t+1}\tilde{\mathbb{X}}_{t}\right]}{2}\right]\\
        & = \E\left[-\frac{\tilde{\mathbb{X}}_{t}^\top\Lambda_{t+1}\tilde{\mathbb{X}}_{t}}{2}\right]\\
        & = -\frac{(t+1)(d-1)}{2}\E\left[\frac{1}{\sigma^2 + (X_{t,d})^2}\right]
    \end{align*}
\end{proof}

\missingFeature*
\begin{proof}
    For all $t \in \mathbb{Z}_{++}$, let
    $$\mathbb{Y}_{t} =
        \begin{bmatrix}
            Y_1\\
            Y_2\\
            \vdots\\
            Y_t
        \end{bmatrix};
        \quad
        X_t = \begin{bmatrix}
            X_{t,1}\\
            \vdots\\
            X_{t,d}\\
        \end{bmatrix};
        \quad \mathbb{X}_t = 
        \begin{bmatrix}
            X_0^\top\\
            X_1^\top\\
            \vdots\\
            X_t^\top\\
        \end{bmatrix};
        \quad \Sigma_t = I_d + \frac{\sum_{i=0}^{t} X_iX_i^\top}{\sigma^2}.
    $$
    Meanwhile, let
    $$
        \tilde{X}_t =
        \begin{bmatrix}
            X_{t,1}\\
            \vdots\\
            X_{t,d-1}\\
        \end{bmatrix};\quad
        \tilde{\mathbb{X}}_t =
        \begin{bmatrix}
            \tilde{X}_0^\top\\
            \vdots\\
            \tilde{X}_t^\top\\
        \end{bmatrix};\quad
        \tilde{\Sigma}_t = I_{d-1} + \tilde{\mathbb{X}}_t\Lambda_{t}\tilde{\mathbb{X}}_t;
    $$
    $$
        \Lambda_t = {\rm diag}\left[\frac{1}{\sigma^2 + (X_{0,d})^2},\hdots, \frac{1}{\sigma^2 + (X_{t,d})^2} \right].
    $$
    With this notation in place, we have that
    \begin{align*}
        & \lim_{t\rightarrow\infty}\E\left[\KL(\hat{P}_t\|P_t)\right]\\
        & = \lim_{t\rightarrow\infty}\E\left[\ln\left(\frac{\hat{P}_t}{P_t}\right)\right]\\
        & = \lim_{t\rightarrow\infty}\\
        & \E\left[\ln\left(\sqrt{1 + \frac{X_{t, d}^2}{\sigma^2}}\cdot\frac{\int_{\theta\in\Re^d}e^{-\frac{\|\mathbb{Y}_{t+1} - \tilde{\mathbb{X}}_{t}\theta\|^2_2}{2\sigma^2}}\cdot e^{-\frac{\theta^\top\theta}{2}}d\theta}{ \int_{\theta\in\Re^d}e^{-\frac{\|\mathbb{Y}_{t} - \tilde{\mathbb{X}}_{t-1}\theta\|^2_2}{2\sigma^2}}\cdot e^{-\frac{\theta^\top\theta}{2}}d\theta }\cdot \frac{\int_{\proxytheta\in\Re^{d-1}}e^{-\sum_{i=0}^{t-1}\frac{(\mathbb{Y}_{i+1} - \tilde{X}_{i}\proxytheta)^2}{2(\sigma^2 + X_{i,d}^2)}}\cdot e^{-\frac{\proxytheta^\top\proxytheta}{2}}d\proxytheta}{\int_{\proxytheta\in\Re^{d-1}}e^{-\sum_{i=0}^{t}\frac{(\mathbb{Y}_{i+1} - \tilde{X}_{i}\proxytheta)^2}{2(\sigma^2 + X_{i,d}^2)}}\cdot e^{-\frac{\proxytheta^\top\proxytheta}{2}}d\proxytheta}\right)\right]\\
        & \overset{(a)}{=} \frac{1}{2}\E\left[\ln\left(1 + \frac{X_{t, d}^2}{\sigma^2}\right)\right] + \lim_{t\rightarrow\infty}\E\left[\ln\left(\frac{\int_{\theta\in\Re^d}e^{-\frac{\|\mathbb{Y}_{t+1} - \tilde{\mathbb{X}}_{t}\theta\|^2_2}{2\sigma^2}}\cdot e^{-\frac{\theta^\top\theta}{2}}d\theta}{ \int_{\theta\in\Re^d}e^{-\frac{\|\mathbb{Y}_{t} - \tilde{\mathbb{X}}_{t-1}\theta\|^2_2}{2\sigma^2}}\cdot e^{-\frac{\theta^\top\theta}{2}}d\theta }\right)\right]\\
        &\quad + \lim_{t\rightarrow\infty}\E\left[\ln\left(\frac{\int_{\proxytheta\in\Re^{d-1}}e^{-\frac{(\mathbb{Y}_t-\tilde{\mathbb{X}}_{t-1}\proxytheta)^\top\Lambda_{t-1}(\mathbb{Y}_{t}-\tilde{\mathbb{X}}_{t-1}\proxytheta)}{2}}\cdot e^{-\frac{\proxytheta^\top\proxytheta}{2}}d\proxytheta}{\int_{\proxytheta\in\Re^{d-1}}e^{-\frac{(\mathbb{Y}_{t+1}-\tilde{\mathbb{X}}_{t}\proxytheta)^\top\Lambda_{t}(\mathbb{Y}_{t+1}-\tilde{\mathbb{X}}_{t}\proxytheta)}{2}}\cdot e^{-\frac{\proxytheta^\top\proxytheta}{2}}d\proxytheta}\right)\right]\\
        & \overset{(b)}{=} \frac{1}{2}\E\left[\ln\left(1 + \frac{X_{t, d}^2}{\sigma^2}\right)\right] + \lim_{t\rightarrow\infty}\E\left[\ln\left(\frac{e^{-\frac{Y_{t-1}^2}{2\sigma^2}}e^{\frac{\mathbb{Y}_{t+1}^\top\tilde{\mathbb{X}}_{t}\Sigma_t^{-1}\tilde{\mathbb{X}}_{t}^\top\mathbb{Y}_{t+1}}{2\sigma^4}}\sqrt{|\Sigma_t|}}{e^{\frac{\mathbb{Y}_t^\top\tilde{\mathbb{X}}_{t-1}\Sigma^{-1}_{t-1}\tilde{\mathbb{X}}_{t-1}^\top\mathbb{Y}_t}{2\sigma^4}}\sqrt{|\Sigma_{t-1}|}}\right)\right]\\
        &\quad + \lim_{t\rightarrow\infty}\E\left[\ln\left(\frac{e^{\frac{\mathbb{Y}_t^\top\Lambda_{t-1}\tilde{\mathbb{X}}_{t-1}\tilde{\Sigma}^{-1}_{t-1}\tilde{\mathbb{X}}_{t-1}^\top\Lambda_{t-1}\mathbb{Y}_t}{2}}\sqrt{|\tilde{\Sigma}_{t-1}|}}{e^{-\frac{Y^2_{t+1}}{2(\sigma^2 + X_{t,d}^2)}}e^{\frac{\mathbb{Y}_{t+1}^\top\Lambda_{t}\tilde{\mathbb{X}}_{t}\tilde{\Sigma}^{-1}_{t}\tilde{\mathbb{X}}_{t}^\top\Lambda_{t}\mathbb{Y}_{t+1}}{2}}\sqrt{|\tilde{\Sigma}_t|}}\right)\right]\\
        & \overset{(c)}{=} \frac{1}{2}\E\left[\ln\left(1 + \frac{X_{t, d}^2}{\sigma^2}\right)\right] + \lim_{t\rightarrow\infty}\E\left[-\frac{Y^2_{t+1}}{2\sigma^2} + \frac{\mathbb{Y}_{t+1}^\top\tilde{\mathbb{X}}_{t}\Sigma^{-1}_t\tilde{\mathbb{X}}_{t}^\top\mathbb{Y}_{t+1}}{2\sigma^4} - \frac{\mathbb{Y}_{t}^\top\tilde{\mathbb{X}}_{t-1}^\top \Sigma^{-1}_{t-1} \tilde{\mathbb{X}}_{t-1}\mathbb{Y}_t}{2\sigma^4}\right]\\
        &\quad + \lim_{t\rightarrow\infty}\E\left[\frac{Y^2_{t+1}}{2(\sigma^2 + X^2_{t, d})} - \frac{\mathbb{Y}^\top_{t+1}\Lambda_{t}\tilde{\mathbb{X}}_{t}\Sigma^{-1}_t\mathbb{X}^\top_t\Lambda_{t}\mathbb{Y}_{t+1}}{2} + \frac{\mathbb{Y}^\top_{t}\Lambda_{t-1}\tilde{\mathbb{X}}_{t-1}\Sigma^{-1}_{t-1}\mathbb{X}^\top_{t-1}\Lambda_{t-1}\mathbb{Y}_{t}}{2}\right]\\
        & \overset{(d)}{=} \frac{1}{2}\E\left[\ln\left(1 + \frac{X_{t, d}^2}{\sigma^2}\right)\right] +\E\left[-\frac{Y^2_{t+1}}{2\sigma^2}\right] + \frac{(t+1)d}{2\sigma^2}-\frac{td}{2\sigma^2}\\
        &\quad + \E\left[\frac{Y^2_{t+1}}{2(\sigma^2 + X_{t,d}^2)}\right] -\frac{(t+1)(d-1)}{2}\E\left[\frac{1}{\sigma^2 + X_{t,d}^2}\right] + \frac{(t)(d-1)}{2}\E\left[\frac{1}{\sigma^2 + X_{t,d}^2}\right]\\
        & = \frac{1}{2}\E\left[\ln\left(1 + \frac{X_{t, d}^2}{\sigma^2}\right)\right] +\E\left[-\frac{Y^2_{t+1}}{2\sigma^2}\right] + \frac{d}{2\sigma^2} + \E\left[\frac{Y^2_{t+1}}{2(\sigma^2 + X_{t,d}^2)}\right] - \frac{d-1}{2}\E\left[\frac{1}{\sigma^2 + X^2_{t,d}}\right]\\
        & = \frac{1}{2}\E\left[\ln\left(1 + \frac{X_{t, d}^2}{\sigma^2}\right)\right] -\frac{1}{2} + \frac{1}{2}\\
    \end{align*}
    where $(a)$ follows from the definitions of $\tilde{X}_t$ and $\Lambda_{t}$, $(b)$ follows from completing the square and evaluating the Gaussian integral, $(c)$ follows from the fact that $\lim_{t\rightarrow\infty}\frac{|\Sigma_t|}{|\Sigma_{t-1}|} = 1$ and $\lim_{t\rightarrow\infty}\frac{|\tilde{\Sigma}_{t-1}|}{|\tilde{\Sigma}_{t}|} = 1$, and $(d)$ follows from Lemmas \ref{le:trace_eq} and \ref{le:trace_diag_eq}. The result follows.
\end{proof}

\section{Proofs of single-layer rate-distortion bounds}
\label{apdx:single_layer_rd}

\pLayerReLUNN*
\begin{proof}
    We use $A_i$ to denote the $i$th row of $A$. Let $\proxy = \tilde{A}$ where $\tilde{A} = A + V$ where $V \perp A$ and $V\in\Re^{N\times d}$ with elements that are iid $\mathcal{N}(0, \delta^2)$ where $\delta^2 = \frac{\sigma^2\left(e^{\frac{2\epsilon}{N}}-1\right)}{d}$. We have that
    \begin{align*}
        \I(\environment;\proxy)
        & = \I(\environment;\proxy)\\
        & = \I(A;\tilde{A})\\
        & \leq \diffentropy(\tilde{A}) - \diffentropy(\tilde{A}|A)\\
        & = \frac{dN}{2}\ln\left(2\pi e\left(\delta^2 + \frac{1}{d}\right)\right) - \diffentropy(V)\\
        & = \frac{dN}{2}\ln\left(1 + \frac{1}{d\delta^2}\right)\\
        & = \frac{dN}{2}\ln\left(1 +\frac{1}{\sigma^2\left(e^{\frac{2\epsilon}{N}}-1\right)} \right)\\
        & \leq \frac{dN}{2}\ln\left(\frac{N}{2\sigma^2\epsilon}\right).
    \end{align*}
    We now verify that our choice of $\proxy_k$ satisfies the distortion constraint:
    \begin{align*}
        \I(Y;\environment|\proxy, X)
        & = \I(Y;A|\tilde{A}, S, X)\\
        & = \diffentropy(Y|\tilde{A}, S, X) - \diffentropy(Y|A, S, X)\\
        & = N\left(\diffentropy(Y_i|\tilde{A}, S, X) - \diffentropy(Y_i|A, S, X)\right)\\
        & \overset{(a)}{=} N\left(\diffentropy(Y_i|\tilde{A}_i, S_i, X) - \diffentropy(W_i)\right)\\
        & = N\left(\diffentropy(Y_i - \relu(\tilde{A}_i^\top X)|\proxytheta_i, S_i, X) - \diffentropy(W_i)\right)\\
        & \overset{(b)}{\leq} N\left(\diffentropy(Y_i - \relu(\tilde{A}_i^\top X )|X) - \diffentropy(W_i)\right)\\
        & \overset{(c)}{\leq} \E\left[N\left(\frac{1}{2}\ln\left(2\pi e \left(\sigma^2 + \var[\relu(A_i^\top X) - \relu(\tilde{A}_i^\top X|X]\right)\right) - \diffentropy(W_i)\right)\right]\\
        & \overset{(d)}{\leq} \frac{N}{2}\ln\left(1 + \frac{\E\left[\left(\relu(A_i^\top X) - \relu(\tilde{A}_i^\top X)\right)^2\right]}{\sigma^2}\right)\\
        & \overset{(e)}{\leq} \frac{N}{2}\ln\left(1 + \frac{\E\left[\left(A_i^\top X - \tilde{A}_i^\top X\right)^2 \right]}{\sigma^2}\right)\\
        & = \frac{N}{2}\ln\left(1 + \frac{\E\left[\left(V^\top X\right)^2\right]}{\sigma^2}\right)\\
        & = \frac{N}{2}\ln\left(1 + \frac{\delta^2 \sum_{i=1}^{d}\E[X_i^2]}{\sigma^2}\right)\\
        & \leq \frac{N}{2}\ln\left(1 + \frac{d \delta^2}{\sigma^2}\right)\\
        & = \epsilon
    \end{align*}
    where in $(a)$, $W_i\sim\mathcal{N}(0, \sigma^2)$, $(b)$ follows from the fact that conditioning reduces entropy, $(c)$ follows from Lemma \ref{le:max_entropy}, $(d)$ follows from Jensen's Inequality and the law of total variance, and $(e)$ follows from the fact that for all $x, y \in\Re$, $(\relu(x) - \relu(y))^2 \leq (x-y)^2$.
\end{proof}

\begin{lemma}{\bf (multinomial proxy squared error)}\label{le:multi_proxy}
    For all $r, M, N \in \mathbb{Z}_{++}$ and $c \in \Re_{++}$, if $A \in \Re^{M\times N}$ is a random matrix
    $$A = 
        \begin{bmatrix}
            - & A_1 & -\\
            \vdots & \vdots & \vdots\\
            - & A_M & -\\
        \end{bmatrix},
    $$
    where for all $i \in \{1, \ldots, M\}$,
    $$A_i = \sqrt{c}\delta_i\cdot B_i,$$
    for $B_i \overset{iid}{\sim} {\rm Dir}(\frac{c}{M}, \ldots, \frac{c}{M})$ and $\delta_i \overset{iid}{\sim} {\rm Rademacher}$, $\tilde{A}\in \mathbb{Q}^{M\times N}$ is a random matrix\\
    $$A = 
        \begin{bmatrix}
            - & \tilde{A}_1 & -\\
            \vdots & \vdots & \vdots\\
            - & \tilde{A}_M & -\\
        \end{bmatrix},
    $$
    where for all $i \in \{1, \ldots, M\}$,
    $$\tilde{A}_i = \frac{\sqrt{c}}{r}\cdot \sign(A_i)\odot \tilde{B}_i,$$
    for $\tilde{B}_i \overset{iid}{\sim} {\rm Multi}(r, |A_i|)$, and $X:\Omega\mapsto\Re^N$ is a random vector, then
    $$\E\left[\|AX - \tilde{A}X\|^2\right] \leq \frac{\sqrt{c}M}{rN}\cdot\E\left[X^\top X\right]$$
\end{lemma}
\begin{proof}
    \begin{align*}
        \E\left[\left\|AX - \tilde{A}X\right\|^2\right]
        & = \E\left[X^\top \left(A^\top A - 2A^\top \tilde{A} + \tilde{A}^\top \tilde{A}\right)X\right]\\
        & \overset{(a)}{=} \E\left[\E\left[X^\top \left(A^\top A - 2A^\top \tilde{A} + \tilde{A}^\top \tilde{A}\right)X|X, A\right]\right]\\
        & = \E\left[\E\left[X^\top\left(\tilde{A}^\top\tilde{A} - A^\top A\right) X|X, A\right]\right]\\
        & = \E\left[\E\left[X^\top\left(\sum_{i=1}^{M}\tilde{A_i}\tilde{A_i}^\top - A_i A_i^\top \right) X|X, A\right]\right]\\
        & = M \cdot \E\left[\E\left[X^\top\left(\tilde{A_i}\tilde{A_i}^\top - A_i A_i^\top \right) X|X, A\right]\right]\\
        & = M \cdot \E\left[\E\left[X^\top\left(\frac{1}{r}{\rm diag}(|A_i|) + \frac{r-1}{r}A_iA_i^\top - A_i A_i^\top \right) X|X, A\right]\right]\\
        & = \frac{M}{r} \cdot \E\left[X^\top\left({\rm diag}(|A_i|) - A_iA_i^\top\right) X\right]\\
        & \overset{(b)}{\leq} \frac{M}{r} \cdot \E\left[X^\top{\rm diag}(|A_i|) X\right]\\
        & \overset{(c)}{=}\frac{\sqrt{c}M}{rN}\cdot\E\left[X^\top X\right],
    \end{align*}
    where $(a)$ follows from the tower property, $(b)$ follows from the fact that $A_iA_i^\top$ is positive semi-definite, and $(c)$ follows from the fact that $\E[{\rm diag}(|A_i|)] = \frac{\sqrt{c}}{N}I_N$
    
\end{proof}

\begin{theorem}{\bf (dirichlet rate distortion)}
    For all $N, M\in\mathbb{Z}_{++}$ and $\sigma^2, \epsilon \geq 0$, if $\environment$ is identified by a random matrix $A\in\Re^{M\times N}$ for which $\|A\|_2 = 1$ and a random matrix $B\in\Re^{d\times N}$ for which each row is sampled from the dirichlet prior, $X:\Omega\mapsto\Re^N$ is a random vector for which $\E[X^2_i]\leq 1$ for all $i\in[N]$, and $Y\sim\normal(B\relu(AX), \sigma^2I_M)$, then
    $$\H_\epsilon(\environment) \leq d^2K\ln^2\left(\frac{6}{\sigma^2(e^{2\epsilon/d}-1)}\right).$$
\end{theorem}
\begin{proof}
    To upper bound the rate distortion, we establish the rate of a particular proxy that attains distortion level $\epsilon$. 
    For $j \in \{1, \ldots, d\}$, consider the following random functions: $v^{(j)}_1, \ldots, v^{(j)}_n$ for which
    $$v^{(j)}_i \overset{iid}{\sim}
    \begin{cases}
        \sign(B_{j,n})\cdot\hat{\phi}_n(x) & \text{ w.p. } |B_{j,n}|\\
    \end{cases},$$
    where $\mathcal{A}_{\epsilon'}$ is an $\epsilon'$-cover of $\sphere$ and 
    $$\hat{\phi}_n(x) = \relu\left((\argmin_{\hat{A}\in\mathcal{A}_{\epsilon'}}\|A_n-\hat{A}\|_2)^\top x\right).$$
    $\hat{\phi}_n$ is a quantization of basis function $\phi_n$. Let $\proxy^{(j)}_{n,\epsilon'} = \frac{\sqrt{K}}{n}\sum_{i=1}^{n} v_i$ and $\proxy_{n, \epsilon'} = (\proxy^{(1)}_{n,\epsilon'}, \ldots, \proxy^{(d)}_{n,\epsilon'})$. We first bound the distortion of $\proxy_{n, \epsilon'}$.
    \begin{align*}
        \I(Y_{t+1};\environment|\proxy_{n, \epsilon'}, X_t)
        & = \diffentropy(Y_{t+1}|\proxy_{n, \epsilon'}, X_t) - \diffentropy(Y_{t+1}|\environment, X_t)\\
        & = \diffentropy(Y_{t+1} - \proxy_{n, \epsilon'}(X_t)|\proxy_{n, \epsilon'}, X_t) - \frac{d}{2}\ln(2\pi e \sigma^2)\\
        & \leq \frac{d}{2}\ln\left(2\pi e\left(\sigma^2 + (B_j^\top\phi(X_t) - \proxy^{(j)}_{n,\epsilon'}(X_t)\right)^2\right)- \frac{d}{2}\ln(2\pi e \sigma^2)\\
        & = \frac{d}{2}\ln\left(1 + \frac{\left(B_j^\top \phi(X_t) - \proxy^{(j)}_{n, \epsilon'}(X_t)\right)^2}{\sigma^2}\right)\\
        & \leq \frac{d}{2}\ln\left(1 + \frac{2\left(B_j^\top \phi(X_t) - \hat{B}^\top_j\phi(X_t)\right)^2 + 2\left(\hat{B}^\top_j\phi(X_t) - \hat{B}^\top_j\hat{\phi}(X_t)\right)^2}{\sigma^2}\right)\\
        & \leq \frac{d}{2}\ln\left(1 + \frac{\frac{2K}{n} + 2\epsilon'^2}{\sigma^2}\right)
    \end{align*}
    
    We now bound the rate of $\proxy_{n, \epsilon'}$
    \begin{align*}
        \I(\environment;\proxy_{n, \epsilon'})
        & \leq \H(\proxy_{n, \epsilon'})\\
        & \leq d\H(\proxy^{(j)}_{n, \epsilon'})\\
        & \leq d \E\left[\sum_{i=1}^{N} \mathbbm{1}_{|\hat{B}_{j,i}| > 0} \left(\ln(2n) + d\ln\left(\frac{3}{\epsilon'}\right)\right) \right]\\
        & \leq dK\ln\left(1+\frac{n}{K}\right)\left(\ln(2n) + d\ln\left(\frac{3}{\epsilon'}\right)\right)
    \end{align*}
    Now, suppose we let $\epsilon' = \sqrt{\frac{2K}{n}}$ and $n = \frac{4K}{\sigma^2\left(e^{2\epsilon/d}-1\right)}$. Then, $\I(Y_{t+1};\environment|\proxy_{n,\epsilon'}, X_t) \leq \epsilon$ and 
    $$\I(\environment;\proxy_{n,\epsilon'}) \leq d^2K\ln\left(1+\frac{4}{\sigma^2\left(e^{2\epsilon/d}-1\right)}\right)\ln\left(\frac{6}{\sigma^2(e^{2\epsilon/d}-1)}\right).$$
    The result follows.
\end{proof}

\begin{corollary}{\bf (teacher network dirichlet rate-distortion)}
\label{th:nonparametric_singlelayer_rd}
    For all $N, M\in\mathbb{Z}_{++}$ and $\sigma^2, \epsilon \geq 0$, if $\environment$ is identified by a random matrix $A\in\Re^{N\times d}$ for which $\|A\|_2 = 1$ and a random vector $B\in\Re^{1\times N}$ distributed according to the dirichlet prior, $X:\Omega\mapsto\Re^d$ is a random vector for which $\E[X^2_i]\leq 1$ for all $i\in[d]$, and $Y\sim\normal(B \relu(AX), \sigma^2)$, then
    $$\H_\epsilon(\environment) \leq dM\ln^2\left(\frac{3d}{\sigma^2\epsilon}\right).$$
\end{corollary}

\section{Proofs of multilayer data processing inequalities}
\label{apdx:relu_dpe}
\begin{lemma}
    \label{le:relu_variance}
    For all real-valued random variables $X$,
    $\var[\relu(X)]\leq \var[X].$
\end{lemma}
\begin{proof}
    \begin{align*}
        \var[X]
        & = \E[X^2] - \E[X]^2\\
        & \overset{(a)}{\geq} \E[\relu(X)^2] - \E[X]^2\\
        & \overset{(b)}{\geq} \E[\relu(X)^2] - \E[\relu(X)]^2\\
        & = \var[\relu(X)]
    \end{align*}
    where $(a)$ follows from the fact that for all $x\in\Re$, $x^2 \geq \relu(x)^2$ and $(b)$ follows from the fact that for all $x\in\Re$, $\relu(x) \geq x$.
\end{proof}

\begin{lemma}
    \label{le:multilayer_distortion_bound}{\bf (multilayer distortion bound)}
    For all $K \in \mathbb{Z}_{++}$, if $\environment_{K:1}$ is a multilayer environment and $\proxy_{K:1}$ is a multilayer proxy for which 
    $$\I(Y;\environment_{k}|\environment_{K:k+1}, \proxy_{k:1}, X) \leq \I(U_{k}+W;\environment_k|\proxy_k, U_{k-1}),$$
    then
    $$\I(Y;\environment_{K:1}|\proxy_{K:1}, X) \leq \sum_{k=1}^{K}\I(U_k+W;\environment_{k}|\proxy_{k}, U_{k-1}).$$
\end{lemma}
\begin{proof}
    We have that
    \begin{align*}
        \I(Y; \environment_{K:1}|\proxy_{K:1}, U_0)
        & \overset{(a)}{=} \sum_{k=1}^{K} \I(Y; \environment_k|\environment_{K:k+1}, \proxy_{K: 1}, U_0)\\
        & \overset{(b)}{=} \sum_{k=1}^{K} \I(Y; \environment_k|\environment_{K:k+1}, \proxy_{k: 1}, U_0)\\
        & \overset{(c)}{\leq} \sum_{k=1}^{K}\I(Y;\environment_k|\environment_{K:k+1}, \proxy_k, U_{k-1}),\\
        & \overset{(d)}{\leq} \sum_{k=1}^{K}\I(U_k+W;\environment_k|\proxy_k, U_{k-1}),
    \end{align*}
    where $(a)$ follows from the chain rule of mutual information, $(b)$ follows from the fact that $(U_K+W) \perp \proxy_{K:k+1}|(U_0, \environment_{K:k+1})$, $(c)$ follows from Lemma \ref{le:true_input_inequality}, and $(d)$ follows the assumption of the lemma statement.
\end{proof}

Lemma \ref{le:multilayer_distortion_bound} demonstrates that for the multilayer processes that we consider in this paper, the distortion incurred by multilayer proxy $\proxy_{K:1}$ is upper bounded by the total distortion of each single-layer proxy conditioned on the true input.

\begin{restatable}{theorem}{multilayerRd}{\bf multilayer rate-distortion bound}
    \label{th:multilayer_rd}
    For all $K \in \mathbb{Z}_{++}$, $\sigma^2, \epsilon \geq 0$, if $\environment_{K:1}$ is a multilayer environment such that there exists a real-valued function $d$ s.t.
    for all $k\in\{1, \hdots, K\}$ and $\delta \geq 0$, there exist $\proxy_k$ s.t.
    $$\I(Y;\environment_{k}|\environment_{K:k+1}, \proxy_{k:1}, X) \leq \Delta(\environment_k, \proxy_k, U_{k-1}) \leq \delta,$$
    where $W \sim \normal(0, \sigma^2I)$, then
    $$\H_\epsilon(\environment_{K:1}) \leq \sum_{k=1}^{K} \H_{\frac{\epsilon}{K}}(\environment_k, \Delta),$$
    where $\H_\epsilon(\environment, \Delta)$ is the rate-distortion function for environment $\environment$ w.r.t distortion function $\Delta$.
\end{restatable}
\begin{proof}
    Let
    \begin{align*}
        \proxyset^{K:1}_\epsilon
        & = \{\proxy\in\proxyset_\epsilon: \proxy = (\proxy_1, \hdots, \proxy_K); \proxy_i \perp \proxy_j \land \proxy_i \perp \environment_j \text{ for } i\ne j\},\\
    \end{align*}
    We have that
    \begin{align*}
        \inf_{\proxy\in\proxyset_\epsilon} \I(\environment;\proxy)
        & = \inf_{\proxy\in\proxyset_\epsilon} \sum_{k=1}^{K}\I(\environment_k;\proxy|\environment_{K:k+1})\\
        & \overset{(a)}{\leq} \inf_{\proxy\in\proxyset^{K:1}_\epsilon} \sum_{k=1}^{K}\I(\environment_k;\proxy|\environment_{K:k+1})\\
        & \overset{(b)}{=} \inf_{\proxy\in\proxyset^{K:1}_\epsilon} \sum_{k=1}^{K}\I(\environment_k;\proxy_k)\\
        & \overset{(c)}{\leq} \sum_{k=1}^{K} \inf_{\proxy_k \in \proxyset^{(k)}_{ \frac{\epsilon}{K}}}\I(\environment_k, \proxy_k)\\
        & = \sum_{k=1}^{K}\H_{\frac{\epsilon}{K}}(\environment_k, \Delta)
    \end{align*}
    where $(a)$ follows from the fact that $\proxyset^{K:1}_\epsilon\subset\proxyset_\epsilon$, $(b)$ follows from the fact that for $\proxy\in\proxyset_{\epsilon}^{K:1}$, $\proxy_i \perp \environment_j$ for $i \neq j$, and $(c)$ follows from the fact that for $\proxyset^{(k)}_{\epsilon/K} := \{\proxy_k \in \proxyset: d(\environment_k, \proxy_k, U_{k-1}) \leq \frac{\epsilon}{K}\}$,
    $\proxyset^{(1)}_{\epsilon/K}\times\hdots\times \proxyset^{(K)}_{\epsilon/K} \subset \proxyset^{K:1}_\epsilon$ since we assumed $\I(Y;\environment_{k}|\environment_{K:k+1}, \proxy_{k:1}, X) \leq d(\environment_k, \proxy_k, U_{k-1})$ for all $k$ and so
    \begin{align*}
        \I(Y;\environment_{K:1}|\proxy_{K:1}, X)
        & = \sum_{k=1}^{K}\I(Y;\environment_k|\proxy_{K:k+1}, \proxy_{k:1}, X)\\
        & \leq \sum_{k=1}^{K}\Delta(\environment_k, \proxy_k, U_{k-1})\\
        & \leq \sum_{k=1}^{K}\frac{\epsilon}{K}\\
        & = \epsilon.
    \end{align*}
   
\end{proof}

\section{Empirical Performance of SGD}
\label{apdx:empir-perf-sgd}
% Temp variables for automatically resizing figures
\newlength{\tempht}
\newsavebox{\tempbox}
\subsection{Teacher Networks}
\label{apdx:teacher-networks-1-hidden}

We assume $T$ i.i.d samples are generated by a single-layer neural network environment described in section~\ref{sec:deep-neural-network}.
In particular, we set
\[
  f(X) = B\relu(AX) \in\R,
\]
where $A\in\R^{N\times d}$ and $B\in \R^{1\times N}$ are parameters describing the teacher network.
We assume that $X\sim \mathcal{N}(0,I_d)$, and $Y= f(X)+W$ with $W\sim \mathcal{N}(0, \sigma^2)$.
We use $d$ to denote the input dimension and $N$, the width of the teacher network.

\subsubsection{Independent Gaussian Prior}
\label{apdx:indep-gauss-prior-def}
In this setting, we append a constant $1$ to the last dimension of $X$, making it $(d+1)$ dimensional.
This is equivalent to adding a constant term in the teacher network.
We further assume that $A\sim\normal(0, \frac{1}{d+1}I_{N\times(d+1)})$, $B\sim\normal(0, \frac{1}{N} I_{1\times N} )$, and that they are independent.
The addition of the constant term does not change the asymptotic sample complexity, and by Theorem~\ref{th:relu_singlelayer_rd} and \ref{th:general-sample-bound}, the sample complexity $T_{\epsilon}$ is bounded by
\[
  T_{\epsilon} \le   \frac{(d+1)N}{\epsilon} \log \left( \frac{N}{\sigma^2\epsilon} \right) 
  .
\]
This is almost linear in the product of the input dimension and the number of hidden units.
We denote the hyperparameters for the teacher network by $\gamma:=(d, N, \sigma)$.
\subsubsection{Dirichlet Prior}
\label{apdx:nonparametric-prior-def}
In this setting, we assume that each row of $A\in N\times d$ is drawn uniformly from the unit sphere, and each row of $B$ is distributed according to $\sqrt{M} {\rm Dirichlet}(M/N, \ldots, M/N)$, with independent random sign flips for each entry.
By Theorem~\ref{th:nonparametric_singlelayer_rd} and \ref{th:general-sample-bound}, the sample complexity $T_{\epsilon}$ is bounded by
\[
  T_{\epsilon} \le  \frac{2dM}{\epsilon}\log^2 \left( \frac{6d}{\sigma^2\epsilon} \right)
  .
\]
This is almost linear in the product of the input dimension and the sparsity.
We denote the hyperparameters for the teacher network by $\gamma:=(d, N, M, \sigma)$.
\subsection{Experiment Setup}
In this section we describe how the experiments are conducted.
We first describe the experiment pipeline and then discuss the various components.

\begin{algorithm}[htb!]
  \begin{algorithmic}[1]
    \For{each data generating hyperparameter $\gamma\in\Gamma$}
      \For{each sample number  $T\in\mathcal{T}$}
        \For{$i \in [\text{num trials}]$}
          \State{$f\gets\text{sample\_f}(\gamma)$ }
          \State{Sample $T$ i.i.d. $(x_1, x_2, ..., x_T)$ according to $N(0, I_d)$}
          \State{$\forall j\in [T]$, calculate $y_{j\ \text{noiseless}}\gets f(x_j)$}
          \State{$\forall j\in [T]$, calculate $y_{j} \gets y_{j\ \text{noiseless}} + w_j$, where $w_j\overset{iid}{\sim} N(0, \sigma^2)$.}
          \State{Set $S\gets \{ (x_j, y_j) | j\in [T]\}$}
          \State{$\hat{f}_S \gets \text{train}(\gamma, S)$, logging the number of queries to data points $Q_{\gamma, N, i}$.}
          \State{Evaluate error according to \eqref{eq:def-error-experiment}:
            \[
              \text{error}_{\gamma, T, i}\gets \frac{\mathbb{E}\left[\left(\hat{f}_S(X)-f(X)\right)^2 | f, S \right]}{2\sigma^2}
            \]  }
        \EndFor
        \State{
          Average over experiments: let
          \[
            \text{error}_{\gamma,T}\gets \frac{1}{\text{num trials}}\sum_{i\in [\text{num trials}]}\text{error}_{\gamma, T, i}
          \] and
          \[
            Q_{\gamma, T}\gets \frac{1}{\text{num trials}}\sum_{i\in [\text{num trials}]}Q_{\gamma, T, i}
            .
          \]
        }

      \EndFor
      \State{Calculate $T_{\gamma, \epsilon}$:
        \[
          T_{\gamma, \epsilon} \gets \min\{T\in\mathcal{T}:
          \text{error}_{\gamma,T} \leq \epsilon
          \}
        \]
      }
    \EndFor
  \end{algorithmic}
  \caption{Experimental Data Generation Algorithm}
  \label{alg:exp-data-gen}
\end{algorithm}

\subsubsection{Experiment Pipeline}
The experiment pipeline is outlined in Algorithm~\ref{alg:exp-data-gen}, and the corresponding code is available online (Appendix~\ref{sec:appendix-code-url}).
The definition of various parameters are shown in Table~\ref{tab:exp-params}, and the respective values chosen for the experiments are summarized in Table~\ref{tab:exp-params-value}.

\begin{table}[htbp]
  \begin{center}
    \begin{tabular}{ c l}
      \hline
      Parameter                        & Descriptions                              
      \\ \hline
      $\gamma\in\Gamma$ & hyperparameters of teacher network          
      \\
      $d$                              & input dimension            
      \\
      $N$                              & number of hidden neurons       
      \\
      $M$                              & sparsity                 
      \\
      $\sigma$                         & standard deviation of noise         
      \\ \hline
      $\epsilon$                       & target test error             
      \\ \hline
      $T\in\mathcal{T}$                & number of samples  
      \\ \hline
      num trials                       & number of trials to run for each configuration 
      \\ \hline
    \end{tabular}
  \end{center}
  \caption{Summary of parameters in experiment}
  \label{tab:exp-params}
\end{table}

\begin{table}[htbp]
  \begin{center}
    \begin{tabular}{ c c l}
      \hline
      Parameter                        & Values chosen & Prior
      \\ \hline
      $\gamma$          & $(d, N, \sigma)\in\Gamma$  & independent Gaussian
      \\
                                       & $(d, N, M, \sigma)\in\Gamma$  & dirichlet
      \\ \hline
      $d$                                                          & $\{1, 2, 4, ..., 2^7=128\}$ & independent Gaussian
      \\
                                       & $\{1, 2, 4, ..., 2^5=32\}$ & dirichlet
      \\ \hline
      $N$                                              & $\{1, 2, 4, ..., 2^7=128\}$ & independent Gaussian
      \\
                                       & $4\max(d, M)$ & dirichlet
      \\ \hline
      $M$                                             & $\{1, 2, 4, ..., 2^5=32\}$ & dirichlet
      \\
      $\sigma$                                         & $\{0.1, 0.2\}$ & all
      \\ \hline
      $\epsilon$                                                 & \begin{tabular}{@{}c@{}}
                                                              $1$ to $0.01$ for $d, m\leq 64$ \\
                                                              $1$ to $0.1$ for $\max(d, m) =128$
                                                            \end{tabular}
      & independent Gaussian
      \\
                                       & $1$ to $0.01$ & dirichlet
      \\ \hline
      $T\in\mathcal{T}$                                         & successive powers of two to reach all target $\epsilon$ & all
      \\ \hline
      num trials                    & at least 32 & all
      \\ \hline
    \end{tabular}
  \end{center}
  \caption{Parameter values in experiments}
  \label{tab:exp-params-value}
\end{table}

To experimentally verify the dependence of sample and computational complexity on the hyperparameters of the teacher network ($d$ and $N$ for the independent Gaussian prior, and $d$ and $M$ for the dirichlet prior), we generate teacher-networks where these parameters are increasing powers of two:
$(d, N) \in\{1,2,4,...,128\}^2$ for the independent Gaussian prior, 
and $(d, M) \in\{1,2,4,...,32\}^2$ for the dirichlet prior.
Then, we estimate the sample complexity $T_\epsilon$ for target error $\epsilon$ spanning two orders of magnitude ($\epsilon\in\{1, 0.1, 0.01\}$) with a training algorithm that automatically tunes the width.
We run the training algorithm on samples $S$ of increasing size $T$ until the test error is below the specified $\epsilon$, and set the smallest such $T$ as $T_\epsilon$.
By choosing $T$ to double each time, we estimate $T_\epsilon$ within a factor of 2.
The above procedure is performed for noise $\sigma=0.1$ and $\sigma=0.2$; and for each configuration, at least $32$ trials are performed to reduce the noise in gathered data.

\subsubsection{Computational Complexity}

In addition to sample complexity, we also log computational complexity.
We use the total number of queries to the training data points as a proxy for computational complexity, which we denote by $Q$.
More concretely, if the algorithm is trained on $m$ batches of size $n$, then the number of queries to the training data points would be $nm$.
When each data point is queried, it generates a forward pass and a backward pass.
So the actual computation complexity of the algorithm is the product of $Q$ and a scaling factor that depends on the fitting model size.
\subsubsection{Training}
\label{apdx:training}
We split the samples $S$ into an internal training set $S_t$ and a validation set $S_v$ using a $80/20$ ratio.
We train single-hidden-layer neural networks of different widths on $S_t$ using golden-section search, and select the model with the best performance on the validation set $S_v$.
Various details are described below.

\paragraph{Architecture of Fitting Network}
The fitting network is an single-hidden-layer ReLU network, the same as the teacher network, but with different widths (number of hidden neurons).
No explicit form of regularization like dropout or weight decay is used.

To find the best width, we perform golden-section search (\verb|scipy.optimize.golden|) on widths ranging from $2$ to $32+8\cdot \max(T, \sqrt{dN} + \max(d, N))$.
This maximum width is chosen to allow ample over-fitting, considering either the number of provided samples, or the architecture of the teacher network.
Golden-section search is performed on the \emph{logarithm (base 2)} of the width, with tolerance set to $0.25$.
The motivation behind this scheme is to get close to a good width by searching few points.
For example, at most $8$ steps are needed to search through widths from $2$ to $1000$ in this scheme
(the number of steps is at most $\frac{\ln(\text{initial range}/\text{tolerance})}{\ln(\phi)}$).
We believe that model performance should roughly be a unimodal function of width.
So golden-section search should find widths near the optimum.
The number of queries $Q$ is the sum of the number of queries for each searched width.

See Appendix~\ref{apdx:arch-fitt-netw} for experiments on different fitting network architectures.

\paragraph{Optimization}
To train the network, we use Adam \citep{2015-Kingma-adam} with respect to L2 loss.
Aside from the learning rate, We use the default parameters from the PyTorch implementation ($\beta_1=0.9, \beta_2=0.999$).
As empirical evidence suggests that small batch sizes generalize better
(for example, see \citet{2017-Keskar-on-large-batch-training}),
we set the batch size to $64$ for a balance between model performance and training speed.

To automatically set the initial learning rate, we adapt the method first proposed in \citet{2017-Smith-cyclical-learning-rates}.
We start with a very small learning rate (1e-8) and exponentially increase it until the model starts to diverge.
We adapt three methods implemented in the fastai library\footnote{https://docs.fast.ai/} to estimate the best learning rate \footnote{steep, where the loss as the steepest descent; minimum, for a learning rate $1/20$ of where the loss is the smallest; and valley, when the loss is in the middle of its longest valley}, and use their medium as the initial learning rate.
The queries to the data points in the phase are included in the calculation of $Q$.

During training, we reduce the learning rate by a factor of $10$ when the validation loss plateaus using \verb|ReduceLROnPlateau| from PyTorch (mode=`min' and patience=12).

We stop training whenever the best validation loss fails to decrease relatively by more than $1\%$ in $24$ epochs, and use the model corresponding to the best validation loss.
For each fitting network, there is a hard cap of 1500 epochs of training, which is typically never reached.

\subsection{Independent Gaussian Priors}
\label{apdx:independent-gaussian-prior}

\begin{figure}[htbp]
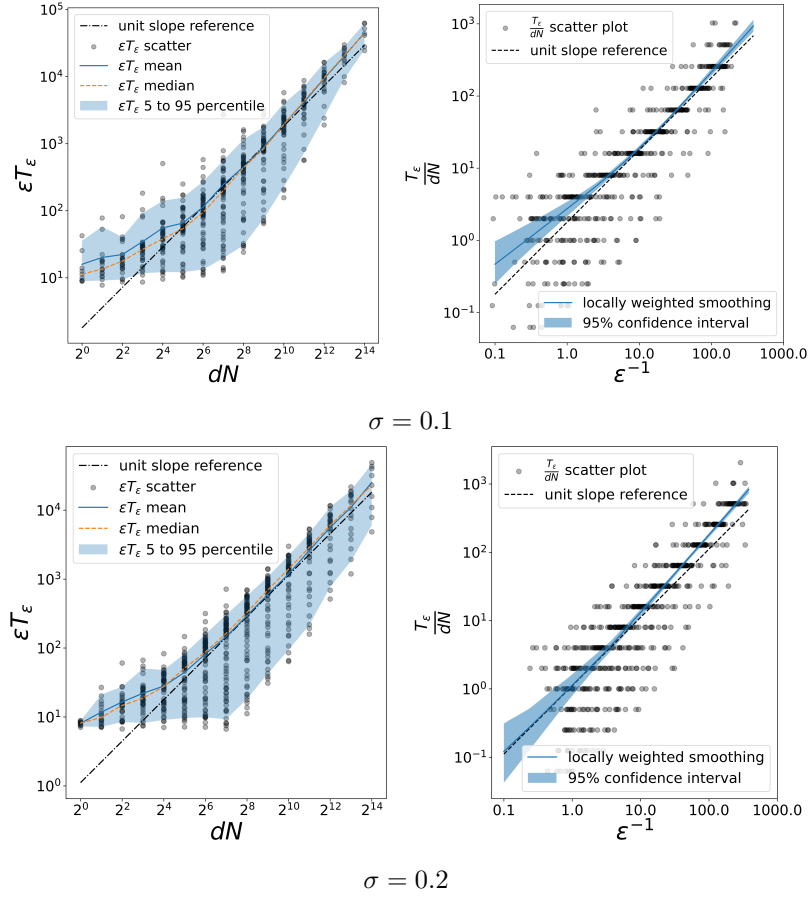

  \centering
  \foreach \noise in {0.1, 0.2} {%
    \begin{subfigure}{0.7\linewidth}%
      \centering
      \resizebox{1.0\linewidth}{!}{%
        \includegraphics[height=3cm]{independent-noise-\noise-sample-complexity-dM.png}
        \includegraphics[height=3cm]{independent-noise-\noise-sample-complexity-epsilon.png}%
      }
      \caption*{$\sigma=\noise$}
    \end{subfigure}
  }

  \caption{
    For the independent Gaussian prior, the sample complexity is almost linear in $\frac{dN}{\epsilon}$ 
for a wide range of $d$, $N$, and $\epsilon$.
    $d$ is the input dimension, $\epsilon$ is the average test error, $N$ is the width of the hidden layer, and $T_\epsilon$ is the corresponding sample size.
    All vertical and horizontal axes are in the log scale, with equal aspect ratio.
    A unit slope reference is provided to indicate a linear relationship in the log scale.
    The confidence intervals for the locally weighted smoothing curves are generated by bootstrap resampling of two-thirds of the data.
    For $\sigma=0.1$ (top), the reference lines correspond to $\epsilon T_\epsilon=1.79dN$;
    for $\sigma=0.2$ (bottom), the reference lines correspond to $\epsilon T_\epsilon=1.11dN$.
  }
  \label{fig:sample-complexity-full-independent}

\end{figure}

This section contains additional results on the sample and computational complexity in the case of the independent Gaussian prior.
\subsubsection{Additional Plots of Sample Complexity}
\label{apdx:independent-addit-plots-sample}
The sample complexity for $\sigma=0.1$ and $\sigma=0.2$ are shown in Figure~\ref{fig:sample-complexity-full-independent}.

We also plotted the dependence of $T_\epsilon$ on $d$ and $N$ for $\epsilon=1, 0.1, 0.01$
(see Figure~\ref{fig:sample-complexity-noise-0.1-independent} and \ref{fig:sample-complexity-noise-0.2-independent}).
Both the horizontal axis and the vertical axis are drawn in log scale, with equal aspect ratio.
The dependence on $N$ for different $d$ is shown on the left, and the dependence on $d$ for different $N$ is shown on the right.
We use shaded areas to indicate the range from
$
   \max \left\{
   T :
   \frac{\mathbb{E}\left[\left(\hat{f}_T(X)-f(X)\right)^2\right]}{2\sigma^2}
   > \epsilon
   \right\}
$
to
$
   \min \left\{
   T :
   \frac{\mathbb{E}\left[\left(\hat{f}_T(X)-f(X)\right)^2\right]}{2\sigma^2}
   \leq \epsilon
   \right\}
$.
Since we only run experiments where the sample size $T$ is a power of $2$, these two different $T$s always differ by a factor of $2$.

\foreach \noise in {0.1, 0.2}{
  \begin{figure}[htb]
    \centering
    \begin{subfigure}{0.9\linewidth}
      \sbox\tempbox{%
        \resizebox{\textwidth}{!}{%
          \foreach \ep in {1, 0.1, 0.01}{%
            \includegraphics[height=5cm]{independent-noise-\noise-epsilon-\ep-sample-complexity-M.png}
          }
        }%
      }
      \setlength{\tempht}{\ht\tempbox}
      \resizebox{\textwidth}{!}{
        \foreach \ep in {1, 0.1, 0.01}{%
          \subcaptionbox*{$\epsilon=\ep$\qquad\quad}{%
            \includegraphics[height=\tempht]{independent-noise-\noise-epsilon-\ep-sample-complexity-M.png}
          }%
        }
      }
      \caption{Dependence of $T_{\epsilon}$ on $N$ for different $\epsilon$ and $d$}
    \end{subfigure}

    \begin{subfigure}{0.9\linewidth}
      \sbox\tempbox{%
        \resizebox{\textwidth}{!}{%
          \foreach \ep in {1, 0.1, 0.01}{%
            \includegraphics[height=5cm]{independent-noise-\noise-epsilon-\ep-sample-complexity-d.png}
          }
        }%
      }
      \setlength{\tempht}{\ht\tempbox}
      \resizebox{\textwidth}{!}{
        \foreach \ep in {1, 0.1, 0.01}{%
          \subcaptionbox*{$\epsilon=\ep$\qquad\quad}{%
            \includegraphics[height=\tempht]{independent-noise-\noise-epsilon-\ep-sample-complexity-d.png}
          }%
        }
      }
      \caption{Dependence of $T_{\epsilon}$ on $d$ for different $\epsilon$ and $N$}
    \end{subfigure}

    \caption{
      Sample complexity $T_\epsilon$ is almost linear in $d$ and $N$ for different $\epsilon$ when $\sigma=\noise$.
      All vertical and horizontal axes are in log scale, with equal aspect ratio.
      The shaded areas indicate that the estimate of the sample complexity $T_\epsilon$ is within a factor of $2$.
    }
    \label{fig:sample-complexity-noise-\noise-independent}
  \end{figure}

}

From the plots we can see that the dependence of $T_\epsilon$ on $d$ eventually becomes linear (unit slope in our plots) for big $N$.
For big $d$, the dependence of $T_\epsilon$ on $N$ eventually becomes slightly worse than linear, but no worse than quadratic (corresponds to slope being $2$ in our plots).
In addition, these observations hold for $\epsilon$ that spans more than two orders of magnitude.

\subsubsection{Computation Complexity}

We plot the number of queries $Q$ against the number of samples $T$ in Figure~\ref{fig:independent-time-complexity-main}.
As in the previous plots, both the horizontal and vertical axis are in log scale and have equal aspect ratio. We include a reference line of unit slope in the log plot, which corresponds to a linear fit of the data.

From Figure~\ref{fig:independent-time-complexity-main} we can see that the dependence of $Q$ on $T$ is slightly less than linear and so $Q = O(T)$.
We already demonstrated that $T_\epsilon$, the number of samples necessary to achieve test error within $\epsilon$ tolerance, appears to be $O(\frac{dN}{\epsilon})$.
Therefore, $T_\epsilon$, the total number of queries to datapoints to achieve $\epsilon$ tolerance, is also approximately proportional to $\frac{dN}{\epsilon}$.
This implies that for all $T$, the average number of times \emph{each single} data point is queried is bounded above by a constant.

In our experiments, the width of the fitting network is $O(d+N)$.
Since each query of a data point corresponds to at most one forward pass and one backward pass, the overall computational complexity is $O(Td(d+N))=\tilde{O}(d^2N(d+N))$ for fixed $\epsilon$.
We hypothesize that by tightening the upper bound on the fitting network's width to $O(N)$, the current results would still hold, and the corresponding computational complexity could be improved to $\tilde{O}(d^2N^2)$.

\begin{figure}[htb]
  \centering
  \foreach \noise in {0.1, 0.2} { %
    \begin{subfigure}{0.48\linewidth}
      \includegraphics[width=\columnwidth]{independent-noise-\noise-time-complexity-N.png}
      \caption{Computational complexity when $\sigma=\noise$}
    \end{subfigure} %
  }

  \caption{
    Total number of queries to datapoints is sublinear in sample size.
    All vertical and horizontal axes are in the log scale, with equal aspect ratio.
    A unit slope reference is provided to indicate a linear relationship in the log scale.
    For $\sigma=0.1$ (left), the reference line corresponds to $Q=1940T$;
    for $\sigma=0.2$ (right), the reference line corresponds to $Q=1622T$.
    The sublinear relationship indicates that the average number of times \emph{each single} data point is queried is $O(1)$ for all $T$. 
  }
  \label{fig:independent-time-complexity-main}
\end{figure}

\subsection{Dirichlet Prior}
\label{apdx:nonparametric-prior}

\begin{figure}[htbp]
  \centering
  \foreach \noise in {0.1, 0.2} {%
    \begin{subfigure}{0.7\linewidth}%
      \resizebox{1.0\linewidth}{!}{%
        \includegraphics[height=3cm]{nonparametric-noise-\noise-sample-complexity-dM.png}
        \includegraphics[height=3cm]{nonparametric-noise-\noise-sample-complexity-epsilon.png}%
      }
      \caption*{$\sigma=\noise$}
    \end{subfigure}
  }
  \caption{
    For the dirichlet prior, the sample complexity is almost linear in $\frac{dM}{\epsilon}$ 
    for a wide range of $d$, $M$, and $\epsilon$.
    $d$ is the input dimension, $\epsilon$ is the average test error, $M$ is the sparsity, and $T_\epsilon$ is the corresponding sample size.
    All vertical and horizontal axes are in the log scale, with equal aspect ratio.
    A unit slope reference is provided to indicate a linear relationship in the log scale.
    The confidence intervals for the locally weighted smoothing curves are generated by bootstrap resampling of two-thirds of the data.
    For $\sigma=0.1$ (top), the reference lines correspond to $\epsilon T_\epsilon=3.92dM$;
    for $\sigma=0.2$ (bottom), the reference lines correspond to $\epsilon T_\epsilon=2.18dM$.
  }
  \label{fig:sample-complexity-full-nonparametric}
\end{figure}

This section contains additional results on the sample and computational complexity in the case of the dirichlet prior.

\subsubsection{Additional Plots of Sample Complexity}
The sample complexity for $\sigma=0.1$ and $\sigma=0.2$ are shown in Figure~\ref{fig:sample-complexity-full-nonparametric}.

We also plotted the dependence of $T_\epsilon$ on $d$ and $M$ for $\epsilon=1, 0.1, 0.01$
(see Figure~\ref{fig:sample-complexity-noise-0.1-nonparametric} and \ref{fig:sample-complexity-noise-0.2-nonparametric}).
For these experiments, the width of the teacher network is chosen to be $4\max(d, M)$ (Table~\ref{tab:exp-params-value}).
The figure format is the same as in Appendix~\ref{apdx:independent-addit-plots-sample} for Figure~\ref{fig:sample-complexity-noise-0.1-independent} and \ref{fig:sample-complexity-noise-0.2-independent}.

\foreach \noise in {0.1, 0.2}{
  \begin{figure}[htb]
    \centering
    \begin{subfigure}{0.9\linewidth}
      \sbox\tempbox{%
        \resizebox{\textwidth}{!}{%
          \foreach \ep in {1, 0.1, 0.01}{%
            \includegraphics[height=5cm]{nonparametric-noise-\noise-epsilon-\ep-sample-complexity-M.png}
          }
        }%
      }
      \setlength{\tempht}{\ht\tempbox}
      \resizebox{\textwidth}{!}{
        \foreach \ep in {1, 0.1, 0.01}{%
          \subcaptionbox*{$\epsilon=\ep$\qquad\quad}{%
            \includegraphics[height=\tempht]{nonparametric-noise-\noise-epsilon-\ep-sample-complexity-M.png}
          }%
        }
      }
      \caption{Dependence of $T_{\epsilon}$ on $M$ for different $\epsilon$ and $d$}
    \end{subfigure}

    \begin{subfigure}{0.9\linewidth}
      \sbox\tempbox{%
        \resizebox{\textwidth}{!}{%
          \foreach \ep in {1, 0.1, 0.01}{%
            \includegraphics[height=5cm]{nonparametric-noise-\noise-epsilon-\ep-sample-complexity-d.png}
          }
        }%
      }
      \setlength{\tempht}{\ht\tempbox}
      \resizebox{\textwidth}{!}{
        \foreach \ep in {1, 0.1, 0.01}{%
          \subcaptionbox*{$\epsilon=\ep$\qquad\quad}{%
            \includegraphics[height=\tempht]{nonparametric-noise-\noise-epsilon-\ep-sample-complexity-d.png}
          }%
        }
      }
      \caption{Dependence of $T_{\epsilon}$ on $d$ for different $\epsilon$ and $M$}
    \end{subfigure}

    \caption{
      Sample complexity $T_\epsilon$ is almost linear in $d$ and $M$ for different $\epsilon$ when $\sigma=\noise$.
      All vertical and horizontal axes are in log scale, with equal aspect ratio.
      The shaded areas indicate that the estimate of the sample complexity $T_\epsilon$ is within a factor of $2$.
    }
    \label{fig:sample-complexity-noise-\noise-nonparametric}
  \end{figure}

}

From the plots we can see that the dependence of $T_\epsilon$ on $M$ is at most linear (unit slope in our plots) for fixed $d$.
The dependence of $T_\epsilon$ on $d$ eventually becomes slightly worse than linear, but no worse than quadratic (corresponds to slope being $2$ in our plots).
In addition, these observations hold for $\epsilon$ that spans more than two orders of magnitude.

\subsubsection{Computation Complexity}

We plot the number of queries $Q$ against the number of samples $T$ in Figure~\ref{fig:nonparametric-time-complexity-main}.
As in the previous plots, both the horizontal and vertical axis are in log scale and have equal aspect ratio. We include a reference line of unit slope in the log plot, which corresponds to a linear fit of the data.

From Figure~\ref{fig:nonparametric-time-complexity-main}, again we can see that the dependence of $Q$ on $T$ is slightly less than linear and so $Q = O(T)$.
This implies that for all $T$, the average number of times \emph{each single} data point is queried is bounded above by a constant.

\begin{figure}[htb]
  \centering
  \foreach \noise in {0.1, 0.2} { %
    \begin{subfigure}{0.48\linewidth}
      \includegraphics[width=\columnwidth]{nonparametric-noise-\noise-time-complexity-N.png}
      \caption{Computational complexity when $\sigma=\noise$}
    \end{subfigure} %
  }

  \caption{
    Total number of queries to datapoints is sublinear in sample size.
    All vertical and horizontal axes are in the log scale, with equal aspect ratio.
    A unit slope reference is provided to indicate a linear relationship in the log scale.
    For $\sigma=0.1$ (left), the reference line corresponds to $Q=2006T$;
    for $\sigma=0.2$ (right), the reference line corresponds to $Q=2174T$.
    The sublinear relationship indicates that the average number of times \emph{each single} data point is queried is $O(1)$ for all $T$. 
  }
  \label{fig:nonparametric-time-complexity-main}
\end{figure}

\subsubsection{Independence on Width of Teacher Network}
In analyzing the dirichlet prior, we assumed that $M \ll N$ (Section~\ref{subsec:sparse-nonparametric-prior}), where $M$ is the sparsity and $N$ is the width of the teacher network.
Under this assumption, we derived theoretical bounds on sample complexity that are independent of $N$ (Corollary~\ref{th:nonparametric_singlelayer_rd}).
In this section, we empirically verify that the sample complexity obtained by SGD is independent of $N$ for large $N$.

\begin{figure}
  \centering
  
  \foreach \noise in {0.1, 0.2} {
    \begin{subfigure}{1.0\linewidth}
      \centering
      \resizebox{\textwidth}{!}{%
        \foreach \d/\M in {4/4, 16/8, 16/32} {
          \includegraphics[height=3cm]{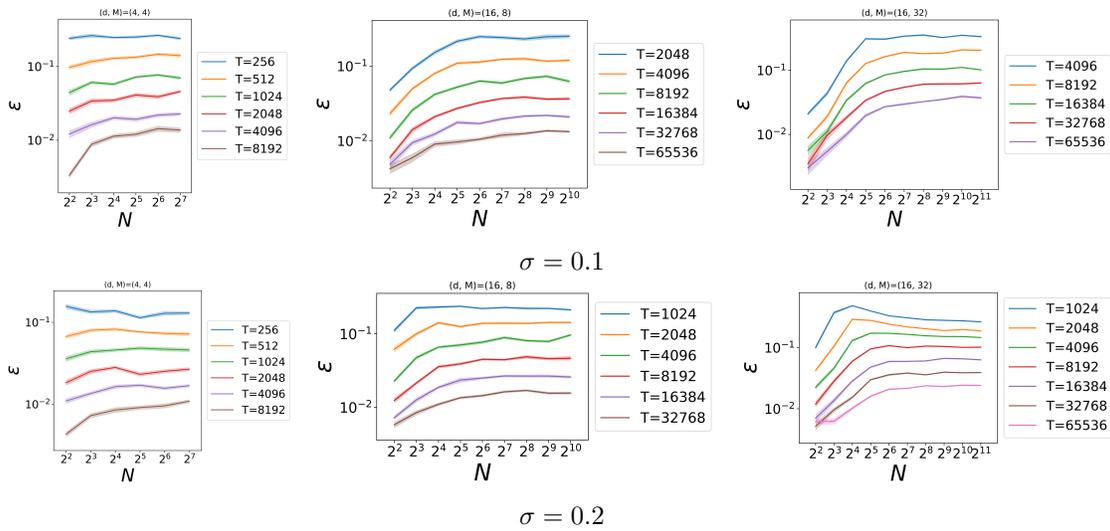}
        }
      }
      \caption*{$\sigma=\noise$}
    \end{subfigure}%
  }
  \caption{
    In the case of the dirichlet prior, the average test error eventually becomes independent of the width of the teacher network $N$ for fixed input dimension $d$ and sparsity $M$.
    Both axes are in log scale.
    $T$ is the number of training samples, $\epsilon$ is the final test error, and $\sigma$ is the noise of the training data.
    The shaded regions indicate plus or minus one standard deviation.
  }
  \label{fig:nonparametric-independent-width}
\end{figure}

In Figure~\ref{fig:nonparametric-independent-width}, we plot the average test error $\epsilon$ versus the width of the teacher network $N$ for different choices of input dimension $d$, sparsity $M$, and noise $\sigma$.
Both axes are in log scale,
and the shaded regions indicate plus or minus one standard deviation.
We observe that as $N$ increases, the average test error $\epsilon$ approaches a constant value.
This aligns with our theoretical findings, which indicate that the sample complexity is independent of $N$.
This also supports our decision to use a fixed value of $N=4\max(d,M)$ in our other experiments.

\subsection{Architecture of Fitting Network}
\label{apdx:arch-fitt-netw}
In this section we study how different fitting network architectures influence performance.
We assume the data is generated by the single hidden layer teacher network with independent Gaussian prior (Appendix~\ref{apdx:indep-gauss-prior-def}).

\subsubsection{Number of layers}
Our fitting algorithm searches over neural networks of depth $1, 2,$ and $3$. When the number of hidden layers is $2$ or $3$, we set the number of neurons in each hidden layer to be the same.
In all cases, we use golden-section search to find the best width.
In this search, the minimal width is $2$, and the maximum widths are given in Table~\ref{tab:maximum-widths} as it will vary depending on the complexity of the data generating process and the size of the dataset.
Again, the maximum width are chosen to allow ample over-fitting, considering either the number of provided samples, or the architecture of the teacher network.

\begin{table}[t]
  \begin{center}
    \begin{tabular}{ll}
      Number of hidden layers & Maximum width
      \\ \hline \\
      1                       & $32 + 8\min\left(T, \sqrt{dN}+\max(d,N)\right)$           \\
      2                       & $32 + 2\min\left(2\sqrt{T}, 2\sqrt{dN}+2\max(d,N)\right)$ \\
      3                       & $16 + 2\min\left(2\sqrt{T}, 2\sqrt{dN}+2\max(d,N)\right)$ \\
    \end{tabular}
  \end{center}
  \caption{Maximum widths for different number of hidden layers.}
  \label{tab:maximum-widths}
\end{table}

\begin{figure}[!htb]
  \centering
  \foreach \layers in {2, 3} {%
    \begin{subfigure}{0.8\linewidth}
      \centering
      \foreach \noise in {0.1, 0.2} {%
        \begin{subfigure}{0.5\linewidth}
          \centering
          \includegraphics[width=\columnwidth]{noise-\noise-hidden-layers-\layers.png}
          \caption*{\qquad\qquad$\sigma=\noise$}
        \end{subfigure}%
      }
      \caption{$\layers$-hidden-layer test error.}
    \end{subfigure}
  }
  \caption{
    Fitting networks with single hidden layer perform better than those with multiple hidden layers.
    We plot the inverse of the test error when fitting network has multiple hidden layers ($\epsilon_2^{-1}$ for 2 hidden layers, and $\epsilon_3^{-1}$ for 3) against the inverse of the test error when fitting network has one hidden layer $\left(\epsilon_1^{-1}\right)$.
    All axes are in log scale, with equal aspect ratio.
    A reference line corresponding to equal error is plotted.
    The region below this line corresponds to single-hidden-layer fitting networks having superior performance.
    The confidence intervals are generated by bootstrap resampling of two-thirds of the data.
    In all cases multiple-hidden-layer fitting networks perform slightly worse than single-hidden-layer fitting networks, especially when the test error is small.
  }
  \label{fig:fitting-layers}
\end{figure}

\begin{table}[htbp]
  \begin{center}
    \begin{tabular}{ c c c c }
      \hline
      Noise                         & Value                   & Geometric mean & Median
      \\ \hline
      \multirow{2}{*}{$\sigma=0.1$} & $\epsilon_2/\epsilon_1$ & $1.22$         & $1.21$
      \\
                                    & $\epsilon_3/\epsilon_1$ & $1.38$         & $1.39$
      \\ \hline
      \multirow{2}{*}{$\sigma=0.2$} & $\epsilon_2/\epsilon_1$ & $1.14$         & $1.16$
      \\
                                    & $\epsilon_3/\epsilon_1$ & $1.25$         & $1.25$
    \end{tabular}
  \end{center}
  \caption{Geometric mean and median of test error ratio}
  \label{tab:fitting-layers}
\end{table}

The results are plotted in Figure~\ref{fig:fitting-layers}, and summary statistics are shown in Table~\ref{tab:fitting-layers}.
We see that on average, having only one hidden layer in the fitting network has slightly better performance than having two hidden layers, which in turn has slightly better performance than having three hidden layers.
This corresponds well with the idea that the fitting network should have similar architecture as the teacher network.

\subsubsection{Width of hidden layer}
In this part, we fix the fitting network to have only one hidden layer and
study the performance of the fitting algorithm for different widths (number of hidden neurons).
We use four different schemes to select the width of the fitting network, which we describe in Table~\ref{tab:fitting-width-schemes}.

\begin{table}[htbp]
  \begin{center}
    \begin{tabular}{ c p{10cm} }
      \hline
      Name          & Description
      \\ \hline
      \textbf{same} & The width of the fitting network is $N$, same as in the teacher network.
      \\
      \textbf{4N}   & The width of the fitting network is $4N$, corresponding to 4x over-parametrization.
      \\
      \textbf{tune} & The width of the fitting network is tuned using golden-section search on the logarithm of the width. The range of widths searched is
                      $$[2,32+8\cdot \max(T, \sqrt{dN} + \max(d, N))].$$
      \\
      \textbf{best} & Use the median of the widths found by the tune method across trials.
    \end{tabular}
  \end{center}
  \caption{Different schemes of selecting the fitting network width.}
  \label{tab:fitting-width-schemes}
\end{table}

We set the width tuning scheme as the baseline, and plot the relative performance of the other schemes in Figure~\ref{fig:fitting-width}, with summary statistics given in Table~\ref{tab:fitting-different-width-stats}.

\begin{figure}[htb]
  \centering
  \foreach \noise in {0.1, 0.2} {
    \begin{subfigure}{1.0\linewidth}
      \centering
      \foreach \width in {same, 4N, best} {%
        \begin{subfigure}{0.33\linewidth}
          \centering
          \includegraphics[width=\columnwidth]{noise-\noise-width-scheme-\width.png}
          \caption*{\qquad\quad\textbf{\width} scheme test error.}
        \end{subfigure}%
      }
      \caption{Effect of width of fitting network when $\sigma=\noise$.}
    \end{subfigure}

  }
  \caption{
    The \textbf{tune} scheme has best performance, followed by the \textbf{best} and \textbf{4N} schemes.
    The \textbf{same} scheme has worst performance.
    We plot the inverse of the test error when using different schemes to select the width of the fitting network
    ($\epsilon_\text{same}^{-1}, \epsilon_{4N}^{-1}, \epsilon_\text{best}^{-1}$)
    against the inverse of the test error when fitting network automatically tunes its width
    ($\epsilon_\text{tune}^{-1}$).
    All axes are in log scale, with equal aspect ratio.
    A reference line corresponding to equal error is plotted.
    The region below this line corresponds to the width tuning scheme having superior performance.
    The confidence intervals are generated by bootstrap resampling of two-thirds of the data.
  }
  \label{fig:fitting-width}
\end{figure}

\begin{table}[htbp]
   \begin{center}
      \begin{tabular}{ c c c c }
         \hline
         Noise                         & Value                                       & Geometric mean & Median
         \\ \hline
         \multirow{3}{*}{$\sigma=0.1$} & $\epsilon_\text{same}/\epsilon_\text{tune}$ & $2.77$         & $1.51$
         \\
                                       & $\epsilon_{4N}/\epsilon_\text{tune}$        & $1.48$         & $1.38$
         \\
                                       & $\epsilon_\text{best}/\epsilon_\text{tune}$ & $1.59$         & $1.21$
         \\ \hline
         \multirow{3}{*}{$\sigma=0.2$} & $\epsilon_\text{same}/\epsilon_\text{tune}$ & $1.97$         & $1.23$
         \\
                                       & $\epsilon_{4N}/\epsilon_\text{tune}$        & $1.26$         & $1.18$
         \\
                                       & $\epsilon_\text{best}/\epsilon_\text{tune}$ & $1.28$         & $1.10$
      \end{tabular}
   \end{center}
   \caption{Effect of different width tuning schemes}
   \label{tab:fitting-different-width-stats}
\end{table}

The width tuning scheme consistently has the best performance, followed by the \textbf{4N} and \textbf{best} schemes.
The \textbf{same} scheme has the worst performance.
These results are consistent with empirical observations that over-parametrization is essential in training neural networks \citep{2017-Ge-learning-one-hidden-layer-neural-networks-with-landscape-design,2014-Livni-on-the-computational-efficiency,2018-Neyshabur-towards-understanding-the-role-of-over-parametrization}.
Perhaps surprisingly, the performance difference between \textbf{tune} and \textbf{best} also indicates that for optimal performance, the architecture of the fitting networks needs to be tuned to the particular \emph{instantiation} of the teacher network, not just to its architecture and number of samples.

\section{Code for running experiments}
\label{sec:appendix-code-url}
\url{https://github.com/fanzhuyifan/rl-sample-complexity}

\end{document}

%% file: math_commands.tex
\usepackage{amsmath,amsfonts,bm}

\def\eqref#1{equation~\ref{#1}}
\def\1{\bm{1}}

\DeclareMathAlphabet{\mathsfit}{\encodingdefault}{\sfdefault}{m}{sl}
\SetMathAlphabet{\mathsfit}{bold}{\encodingdefault}{\sfdefault}{bx}{n}

\newcommand{\E}{\mathbb{E}}

\newcommand{\R}{\mathbb{R}}

\newcommand{\KL}{D_{\mathrm{KL}}}

\DeclareMathOperator*{\argmin}{arg\,min}

\DeclareMathOperator{\sign}{sign}